\documentclass[12pt]{article}
\usepackage{mathrsfs}
\usepackage{amsmath}
\usepackage{bm}
\usepackage{dsfont}
\usepackage{amssymb}
\usepackage{amsthm}
\usepackage{graphicx}
\usepackage{enumerate}
\usepackage{natbib}
\usepackage{algorithm}
\usepackage{algorithmic}
\usepackage{url} % not crucial - just used below for the URL 
\usepackage{xcolor}
\usepackage{multirow} 
\usepackage{hyperref}
\usepackage{booktabs}
\usepackage[capitalize]{cleveref}
\newcommand{\blind}{1}

\def\T{\top}
\def\i{\infty}
\def\wh{\widehat}
\def\wt{\widetilde}
\def\EE{\mathbb{E}}
\def\RR{\mathbb{R}}
\def\cF{\mathcal{F}}

\def\cA{\mathcal{A}}
\def\d{\text{d}}
\def\cR{\mathcal{R}}
\def\cT{\mathcal{T}}
\def\cZ{\mathcal{Z}}
\def\cS{\mathcal{S}}
\def\cE{\mathcal{E}}
\def\cB{\mathcal{B}}
\def\cX{\mathcal{X}}

\def\cL{\mathcal{L}}
\def\cH{\mathcal{H}}
\def\cD{\mathcal{D}}
\def\cW{\mathcal{W}}
\def\emloRa{\widehat{\mathcal{R}}}

\def\l{\left}
\def\r{\right}
\def\b{\big}
\def\B{\Big}

\newtheorem{theorem}{Theorem}[section]
\newtheorem{lemma}[theorem]{Lemma}
\newtheorem{assumption}[theorem]{Assumption}
\newtheorem{remark}{Remark}
\newtheorem{definition}[theorem]{Definition}
\newtheorem{proposition}[theorem]{Proposition}

\begin{document}

\def\spacingset#1{\renewcommand{\baselinestretch}%
{#1}\small\normalsize} \spacingset{1}

%%%%%%%%%%%%%%%%%%%%%%%%%%%%%%%%%%%%%%%%%%%%%%%%%%%%%%%%%%%%%%%%%%%%%%%%%%%%%%

\if1\blind
{
  \title{\bf Distributional Off-Policy Evaluation with Deep  Quantile Process Regression}
  \author{Qi Kuang$^{1}\thanks{The first two authors contribute equally to this paper}$,  Chao Wang$^{2*}$,  Yuling Jiao$^{3\dag}$ and  Fan Zhou$^{2}\thanks{Joint corresponding authors}$ \\ %\thanks{The authors gratefully acknowledge \textit{please remember to list all relevant funding sources in the unblinded version}}\hspace{.2cm}\\
    \\[1ex]
    {\small $^1$ School of Statistics and Data Science, and Philosophy and Social Sciences Laboratory}\\
    {\small of Data Science in Finance and Economics at the Ministry of Education,}\\
    {\small Jiangxi University of Finance and Economics}\\[0.6ex]
    {\small $^2$ School of Statistics and Data Science, Shanghai University of Finance and Economics}\\[0.6ex]
    {\small $^3$ School of Artificial Intelligence, and Hubei Key Laboratory of Computational Science,}\\
    {\small Wuhan University}
   } 
   \date{}
  \maketitle
} \fi

\if0\blind
{
  \bigskip
  \bigskip
  \bigskip
  \begin{center}
    {\LARGE\bf Distributional Off-Policy Evaluation with Deep  Quantile Process Regression}
\end{center}
  \medskip
} \fi

\bigskip
\begin{abstract}

This paper investigates the off-policy evaluation (OPE) problem from a distributional perspective. Rather than focusing solely on the expectation of the total return, as in most existing OPE methods, we aim to estimate the entire return distribution. To this end, we introduce a quantile-based approach for OPE using deep quantile process regression, presenting a novel algorithm called Deep Quantile Process regression-based Off-Policy Evaluation (DQPOPE). We provide new theoretical insights into the deep quantile process regression technique, extending existing approaches that estimate discrete quantiles to estimate a continuous quantile function.
A key contribution of our work is the rigorous sample complexity analysis for distributional OPE with deep neural networks, bridging theoretical analysis with practical algorithmic implementations. We show that DQPOPE achieves statistical advantages by estimating the full return distribution using the same sample size required to estimate a single policy value using conventional methods. Empirical studies further show that DQPOPE provides significantly more precise and robust policy value estimates than standard methods, thereby enhancing the practical applicability and effectiveness of distributional reinforcement learning approaches.
\end{abstract}

\noindent%
\vfill

\newpage
\spacingset{1.9} % DON'T change the spacing!
\section{Introduction}\label{sec:intro}

Off-policy evaluation (OPE) is a fundamental problem in reinforcement learning (RL) that seeks to estimate the value of a target policy from data collected under a different behavior policy. Its importance is especially pronounced in applications where online experimentation is costly, risky, or ethically constrained. OPE arises in a broad range of settings, including contextual bandits and more general sequential decision-making problems \citep{liao2021off}. In fields like healthcare, OPE enables the assessment of dynamic treatment policies using historical electronic health records, where real-world experimentation may be impractical or ethically sensitive \citep{wang2012evaluation, zhu2019proper}. In such settings, applying a new treatment policy without offline validation may lead to ethical concerns. These considerations make OPE a particularly vital tool in offline RL.

In recent years, distributional reinforcement learning (DRL) has gained significant traction as an alternative to traditional RL methods. Rather than estimating just the expected value of future returns, DRL models the entire distribution of returns \citep{C51, QRDQN, IQN}, capturing the inherent randomness in dynamic environments. This approach has shown promise, particularly when the mean information is insufficient to represent the full complexity of decision-making. By accounting for the distribution,
DRL offers advantages in mean estimation.  \citet{rowland2023icml} empirically demonstrate that quantile-based distributional RL (QDRL), through quantile averaging, can achieve a lower mean squared error (MSE) in value estimation compared to standard RL approaches.

The potential of DRL methods has been emphasized in various real-world decision-making scenarios \citep{Bodnar2019QuantileQF,Bellemare2020AutonomousNO}. For instance, in neuroscience, recent work \citep{dabney2020distributional,muller2024distributional,lowet2025opponent} demonstrate that in these complex biological environments, the distribution learning perspective is biologically plausible. These studies reveal that animals may encode not only the mean of return but the entire distribution of possible outcomes, enabling more flexible, nuanced, and context-sensitive decision-making. Similarly, in healthcare, \citet{jin2023bayesian} introduced a Bayesian framework that incorporates distributional methods to optimize sequential combination antiretroviral therapy (cART) for HIV patients, accounting for uncertainties in patient outcomes over time.  In the order dispatching systems of ride-sharing platforms, where balancing driver workloads with maximizing customer satisfaction is a core challenge \citep{zhuhongtu_book}, distributional methods can effectively capture the distributional characteristics of these objectives, enabling better management of trade-offs and leading to more informed decision-making based on value functions \citep{zhouicdm}.

While DRL has been developed largely for online control and policy optimization, these approaches are not directly applicable to offline settings. Nevertheless, given the remarkable success of distributional reinforcement learning (DRL) in online scenarios, its potential in offline tasks, such as off-policy evaluation (OPE), is promising. In principle, a distributional approach to OPE can provide a richer characterization of policy performance and may also improve the accuracy of policy value estimation.

Additionally, this paper contributes to the theoretical understanding of distributional reinforcement learning. Despite its strong empirical performance, the theoretical foundations underlying the advantages of distributional methods over standard reinforcement learning remain limited. Existing analyses are largely confined either to maximum likelihood estimation (MLE) frameworks \citep{mle-drl} or to relatively simple tabular cases \citep{rowland2023analysis, zhang2023estimation}. As a result, there is still limited understanding of how distributional methods, especially when combined with deep neural network approximation, can yield statistical and practical benefits in complex, real-world applications. To address this gap, we propose \textbf{D}eep \textbf{Q}uantile \textbf{P}rocess regression-based \textbf{O}ff-\textbf{P}olicy \textbf{E}valuation (DQPOPE), a novel approach that applies distributional methods to off-policy evaluation (OPE) using deep quantile process regression. This paper aims to establish a theoretical foundation for distributional OPE, providing insights that are relevant to the broader DRL landscape. Our contributions can be summarized as follows:

\begin{itemize}
    \item 
   
    \textbf{Introduction of quantile process regression for OPE.}    Unlike prior QDRL approaches that focus on estimating discrete quantiles, our method employs the quantile process to model the entire return distribution, thus avoiding the representation error inherent in discrete quantile approximations. Essentially, the introduction of quantile process regression effectively transforms the inherently infinite-dimensional distribution learning task into a finite-dimensional regression task,  enhancing both theoretical tractability and practical implementation.
    \item 
    \textbf{Advancing the theory of deep distributional OPE.} We provide a theoretical framework for analyzing distributional OPE through the lens of the quantile process. Particularly, we employ deep neural network (DNN) approximations, which are essential both theoretically and practically due to the inherently nonlinear structure of the distributional Bellman operator \citep{rowland2023analysis}. 
    
    \item   \textbf{Advantages over value-based OPE methods.} 
    We theoretically demonstrate that DQPOPE can estimate the entire return distribution with the same sample size required by value-based OPE methods for estimating only the distribution mean. Our empirical results further indicate that DQPOPE consistently achieves more accurate mean value estimates compared to value-based OPE methods by better capturing the randomness.
\end{itemize}

\section{Problem Setup and Notations}\label{sec_2}
\noindent\textbf{Notations.} 
For any measurable space $\cX$, let $\Delta(\mathcal X)$ denote the set of all probability measures on $\mathcal X$. 
For a measurable function $f:\mathcal X\to \mathbb R$ and $\nu\in\Delta(\mathcal X)$, define $\|f\|_{p,\nu}= \b(\int_\cX |f(x) |^p~  d \nu(x)\b)^{\frac{1}{p}}$ for $p\ge1 $ if it exists. For  any $\nu, \mu\in \Delta(\RR)$, the $p$-Wasserstein distance between $\nu$ and $\mu$ is defined by $\mathcal W_p(\nu,\mu)
:=\b(\int_0^1 |F_\nu^{-1}(t)-F_\mu^{-1}(t)|^pdt\b)^{1/p}$,
where $F_\nu^{-1}$ and $F_\mu^{-1}$ denote the quantile functions of $\nu$ and $\mu$, respectively. For a measurable map $\eta:\mathcal X\to \Delta(\mathbb R)$ and a random variable $X\sim \nu\in\Delta(\mathcal X)$, we use boldface to denote the mixture distribution induced by $\eta(X)$, namely $\boldsymbol{\eta}: = \mathbb{E} [\eta(X)]$. For two positive sequences $a_n$ and $b_n$, we write $a_n \lesssim b_n$ if $a_n \le C b_n$ for some positive constant $C>0$ independent of $n$. For $v=(v_1,...,v_d)^\T\in \RR^d$, define  $\max\{v,0\}=(\max\{v_1,0\},...,\max\{v_d,0\})^\T$. Let $\mathbb{N}$ denote the natural numbers, $\mathbb{N}^+ = \mathbb{N} \setminus \{0\}$, and $\lfloor x \rfloor$ the floor function. The $\mathcal{O}$ notation omits constants and lower-order terms for clarity.

\medskip

%We use bold notation to represent the expectation of random distribution.  For instance,  given  $\eta(x) \in \Delta (\RR)$ with  $x\in \cX $ and  $X\sim \nu \in \Delta(\cX )$, write $\boldsymbol{\eta}: = \mathbb{E} [\eta(X)]$.  

\noindent\textbf{Markov Decision Processes.} Consider a Markov Decision Process $(\mathcal{S}, \mathcal{A}, P, \gamma, \mathcal{R})$. $\mathcal{S}$ is the state space,  $\mathcal{A}$ is a  finite set of actions, $P :\mathcal{S} \times \mathcal{A} \rightarrow \Delta(\mathcal{S})$ 
is the transition kernel,  $\gamma\in (0,1)$ is some 
pre-specified discounted factor, $\mathcal{R}:\mathcal{S} \times \mathcal{A} \rightarrow \Delta(\mathbb{R})$  is the distribution of reward.
A policy  $\pi:\mathcal S \to \Delta(\mathcal A)$ specifies the distribution $\pi(\cdot| s)$ of taking action given $s\in \cS$. 
% Consider the evaluation of a target policy $\pi:\mathcal S \to \Delta(\mathcal A)$.
Starting from an initial state $S_0\sim \rho\in\Delta(\mathcal S)$, the trajectory $\{S_t,A_t,R_t\}_{t\ge0}$ evolves according to $A_t\sim \pi(\cdot |\ S_t)$, $R_{t} \sim \mathcal{R}(\cdot| S_{t}, A_{t})$, $S_{t+1} \sim P(\cdot | S_{t}, A_{t})$.

% Consider the evaluation of a target policy $\pi:\mathcal S \to \Delta(\mathcal A)$.
Standard RL estimates the expected return 
$Q^{\pi}(s,a):=\mathbb{E}_{\pi}\b[\sum_{t=0}^{\infty} \gamma^t R_t~|~S_0=s,A_0=a\b]$, where $\mathbb{E}_{\pi}$  takes expectation over
$\{ R_t\}_{t\ge0}$ given $S_0=s, A_0=a$ under the policy $\pi$.  DRL instead studies the law of the discounted return $Z^{\pi}(s, a):= \sum_{t=0}^{\infty}\gamma^t R_t~|~S_0=s,A_0=a$, regarded as a random variable indexed by $(s,a)$.
Denote the collection of  maps from $\cS\times \cA$ to $\Delta(\RR)$ as $\Delta(\RR)^{\cS\times \cA}$. 
The return distribution 
$\eta^{\pi}\in \Delta(\RR)^{\cS\times \cA} $ 
is defined by $\eta^{\pi}(s,a)=\operatorname{law}(Z^{\pi}(s, a))$ for any $(s,a)\in \cS\times \cA$,
where $\operatorname{law}(\cdot)$ extracts the distribution of the input random variable. 
For any $(s,a)\in \cS\times \cA$, let $(s,a,R,S')$ be a random transition according to $R\sim \mathcal{R}(\cdot|s,a), S'\sim P(\cdot| s,a)$. The distributional Bellman operator $\mathcal{T}^{\pi}:\Delta(\mathbb{R})^{\mathcal{S}\times\mathcal{A}}\to\Delta(\mathbb{R})^{\mathcal{S}\times\mathcal{A}}$ is the mapping defined by
\begin{align}\label{equ: distributional bellman operator}
\left(\mathcal{T}^{\pi} \eta\right){(s, a)} := \mathbb{E}_{\pi}\b[(g_{\gamma, R})_{\#}\eta(S',A')\mid s, a\b],
\end{align} 
where $g_{\gamma, r}:\mathbb{R}\to\mathbb{R}$ is defined by $g_{\gamma, r}(s)=r+\gamma s,$ and  $(g_{\gamma, r})_{\#}\nu$ is the pushforward on distribution $\nu$ 
defined as $(g_{\gamma, r})_{\#}\nu = \text{law}(g_{\gamma, r}(Z))$,  with $Z\sim \nu$.  Equivalently, for $Z(s,a)\sim \eta(s,a)$, the $\mathcal{T}^{\pi}$ can also be defined in terms of random variables \citep{drl_book}\footnote{One can refer to Section 4 of this book for more details.}
\begin{align}\label{equ: distributional bellman operator random variable}
\big(\mathcal{T}^{\pi} \eta\big)(s, a) := \text{law}\big( R+\gamma Z(S',A')~ | ~ s,a\big).
\end{align} 
% Here, the $\text{law}(\cdot)$ captures the randomness from  $R\sim \mathcal{R}(\cdot|s,a), S'\sim P(\cdot| s,a)$ and $ A'\sim \pi(\cdot| S')$.  
The return distribution $\eta^{\pi}$ is the unique fixed point of $\mathcal{T}^{\pi}$, namely $\mathcal{T}^{\pi} \eta^{\pi}=\eta^{\pi}$ \citep{drl_book}. Finally, define the discounted occupancy distribution under policy $\pi$ and initial state distribution $\rho$ by $d^{\pi}(s, a):=(1-\gamma) ~ \sum_{t=0}^{\infty} \gamma^t~  \mathbb{P}\b[S_t=s, A_t=a~ |~ \pi,\rho\b]$.
\medskip

\noindent\textbf{Off-policy evaluation.} In standard off-policy evaluation (OPE), the goal is to estimate the value of a target policy $\pi$ under an initial distribution $\rho$, defined by $V^{\pi}:=\mathbb{E}_{\pi}\b[\sum_{t=0}^{\infty} \gamma^t R_t\b]=\mathbb{E}_{(S, A) \sim \rho\times\pi}\b[Q^{\pi}(S, A)\b]$, where $(\rho\times \pi)(s,a): = \rho(s) \pi(a|s)$. 
For any estimate $\wh Q$ of $Q^{\pi}$, 
its estimation accuracy is typically assessed by the absolute error  $|V^{\pi}-\widehat{V}|$, where $\widehat{V} = \mathbb{E}_{ (S, A) \sim \rho\times\pi }\b[\widehat{Q}(S, A)\b]$.
Distributional OPE instead targets the full return distribution.  
The target of interest is the performance of the target policy $\pi$, defined  by  {$\boldsymbol{\eta^{\pi}}:=\mathbb{E}_{ (S, A) \sim \rho\times\pi}
\b[\eta^{\pi}(S,A)\b]$}, which is the mixture distribution obtained by averaging $\eta^\pi(S,A)$ over $(S,A)\sim \rho\times\pi$. For any estimate $\widehat{\eta}\in \Delta(\mathbb{R})^{\cS\times \cA}$ of $\eta^\pi$, the estimation accuracy is assessed by a distributional discrepancy, specifically the $p$-Wasserstein distance $\mathcal{W}_{p} (\boldsymbol{\eta}^{\pi},\widehat{\boldsymbol{\eta}})$,  where {$\widehat{\boldsymbol{\eta}} = \mathbb{E}_{(S,A)\sim\rho\times\pi}\b[\widehat{\eta}(S, A)\b]$}. 
The offline dataset can be summarized as $\mathcal{D}=\{(s_i,a_i,r_i,s'_i)\}_{i=1}^{N}$,  which is collected under another unknown behavior policy $\pi^b$ with the data generating procedure that $a_i\sim \pi^b(\cdot| s_i)$, $r_i\sim \mathcal{R}(\cdot| s_i,a_i)$, and $s'_i\sim P(\cdot| s_i,a_i)$. We assume that $(s_i,a_i)$ are i.i.d. draws from the data distribution $\mu(s,a):= (1-\gamma)\sum_{t=0}^{\infty} \gamma^t \mathbb{P}\b[S_t=s, A_t=a~|~ \pi^b,\rho\b]$.

We aim to analyze the sample complexity of distributional OPE by bounding $ \mathcal{W}_{p}(\boldsymbol{\eta}^{\pi},\widehat{\boldsymbol{\eta}})$. 
% \redtext{While the error bounds are conventionally expressed within the form
% $\sup_{s, a} \mathcal{W}_{p}(\eta(s,a),\eta'(s,a))$, this metric is not well-aligned with the target $\boldsymbol{\eta^{\pi}}$.} 
We adopt an expected Wasserstein metric equipped with an $L_2$-norm structure, following recent literature  \citep{abdullah2019wasserstein,mle-drl}. Given a distribution $\nu$ over $(S, A)$,
for any $\eta, \eta'\in \Delta(\RR)^{\cS\times \cA}$,  define
\begin{align*}
\overline{\mathcal{W}}_{p,\nu}(\eta,\eta') : = \Big(\mathbb{E}_{(S, A) \sim \nu} \l[\mathcal{W}_p^{2p}(\eta(S,A), \eta'(S,A))\r]\Big)^{\frac{1}{2p}}.
\end{align*}
In particular, when $\nu=d^\pi$, \citet{mle-drl} shows that for any $\eta, \eta'\in \Delta(\RR)^{\cS\times \cA}$ and $p\ge1$,  the distributional Bellman operator is $\gamma^{1-\frac{1}{2p}}$-contractive under the metric $\overline{\mathcal{W}}_{p,d^{\pi}}(\cdot,\cdot)$,
\begin{align*}
\overline{\mathcal{W}}_{p,d^{\pi}}\big(\mathcal{T}^\pi \eta,\mathcal{T}^\pi \eta'\big)~\leq~ \gamma^{1-\frac{1}{2p}} ~ \overline{\mathcal{W}}_{p,d^{\pi}}\big( \eta, \eta'\big).
\end{align*}
This property is crucial for deriving the sub-optimality decomposition  in Lemma \ref{lem: Sub-optimality decomposition}.

%Compared to Lemma 4.9 in \cite{mle-drl}, we refine the formalization of this $L_2$-type Wasserstein metric by explicitly specifying the distribution over which variables are drawn and restricting $p\ge1$. 

% \begin{lemma}[Lemma 4.9 in \citet{mle-drl}]\label{lem: contractive}
%  For any $\eta, \eta'\in \Delta(\RR)^{\cS\times \cA}$ and $p\ge1$, the distributional Bellman operator is $\gamma^{1-\frac{1}{2p}}$-contractive under the metric $\overline{\mathcal{W}}_{p,d^{\pi}}(\cdot,\cdot)$,
% \begin{align*}
% \overline{\mathcal{W}}_{p,d^{\pi}}(\mathcal{T}^\pi \eta,\mathcal{T}^\pi \eta'\big)~\leq~ \gamma^{1-\frac{1}{2p}} ~ \overline{\mathcal{W}}_{p,d^{\pi}}( \eta, \eta'\big).
% \end{align*}
% \end{lemma}

\section{Deep Quantile Process Regression-based Off-Policy Evaluation (DQPOPE)}
\label{sec: methods}
In this section, we provide a comprehensive description of the proposed DQPOPE method and demonstrate its advantages from both analytical and practical perspectives. 

\subsection{Foundations and Implementation of DQPOPE}
The goal of distributional OPE is to estimate the target return distribution $\eta^\pi$. Leveraging 
the contraction property of $\mathcal{T}^{\pi}$,  a natural approach is to consider the iterative scheme $\eta_t=\mathcal T^\pi \eta_{t-1}$, initialized from some $\eta_0$. This sequence converges to the fixed-point $\eta^\pi$ under the metric $\overline{\mathcal{W}}_{p,d^{\pi}}(\cdot,\cdot)$. 
However, directly applying the operation $\mathcal{T}^{\pi}\eta_t$ is infeasible in implementation, 
necessitating an approximation of $\mathcal{T}^{\pi}$ at each step. To this end, we exploit the one-to-one correspondence between a distribution and its quantile function and reformulate the distributional iteration as a sequence of quantile function estimation problems.  For any fixed $(s,a)\in\cS\times\cA$, denote $f(s,a,\cdot):  (0,1)\rightarrow \mathbb{R}$
% denote $f(s,a,\cdot):\mathcal{S}\times\mathcal{A}\times(0,1)\rightarrow \mathbb{R}$ 
as the quantile function of its corresponding distribution $\eta(s,a)$. The learning task therefore gets transferred into recovering the quantile function $f^*: \mathcal{S}\times\mathcal{A}\times(0,1) \rightarrow \mathbb{R}$ of the target return distribution $\eta^{\pi}\in \Delta^{\cS\times \cA}$.
% We therefore aim to recover the quantile function of the target return distribution $\eta^{\pi}$, denoted by $f^{*}(s,a,\cdot)$.
We proceed to detail how to implement a one-step distributional Bellman update using quantile process regression.

Specifically,  starting from the quantile function $\wh f_0: \mathcal{S} \times\mathcal{A}\times(0,1)\rightarrow \mathbb{R}$, the algorithm recursively produces a sequence of quantile functions $\wh f_1, \wh f_2,\cdots,\wh f_T$, where 
each quantile function  $\wh f_t$ is searched from certain function space $\mathcal{F}$. 
For each $t\in [T]$,
the quantile function $\widehat{f}_{t}$ induces a distribution $\widehat{\eta}_{t}$, and hence provides the full distributional information at iteration $t$.
Recall that given any 
$(s,a)$, $R\sim \cR(\cdot|s,a)$, $S'\sim P(\cdot|s,a)$, and let $Z_{t}(s,a)\sim \widehat{\eta}_{t}(s,a)$.
For the transition $(s,a,R,S')$, by definition of  $\mathcal{T}^{\pi} $ in \eqref{equ: distributional bellman operator random variable},  $\mathcal{T}^{\pi} \widehat{\eta}_{t-1}$ is given by
$$
\left(\mathcal{T}^{\pi} \widehat{\eta}_{t-1}\right){(s, a)} := \text{law}\Big( R+\gamma Z_{t-1}(S',A')~|~ s,a\Big).
$$
This reformulates the problem into finding the quantile function of the random variable $Y_t = R+\gamma Z_{t-1}(S',A')$ conditional on $(s,a)$.
Writing $X=(S,A)$, the conditional quantile function of $Y_{t}$ given $X$ is the solution to the risk minimization problem $f^*_t = \arg\min_f \mathcal{L}_t(f)$:
\begin{align}\label{equ: iqn loss}
\mathcal{L}_t(f)= \mathbb{E}_{X, Y_{t}, \tau}\B(\rho_{\tau}(Y_{t}-f(X, \tau))\B),
\end{align}
where $\rho_{\tau}(u)=u(\tau-\textbf{1}_{u\leq 0})$ 
is the check loss, and $\tau\sim \mathrm{Unif}(0,1)$ is independent of $(X,Y_t)$. Here the quantile level $\tau$ is treated as a random input to the function $f$, a formulation referred to as a quantile process \citep{Volgushev}. Henceforth, we slightly abuse the notation of $\tau$ as both random variable and fixed quantile level when clear from context. 

\begin{figure}[!ht]
\centering
\includegraphics[width = 0.48\linewidth]{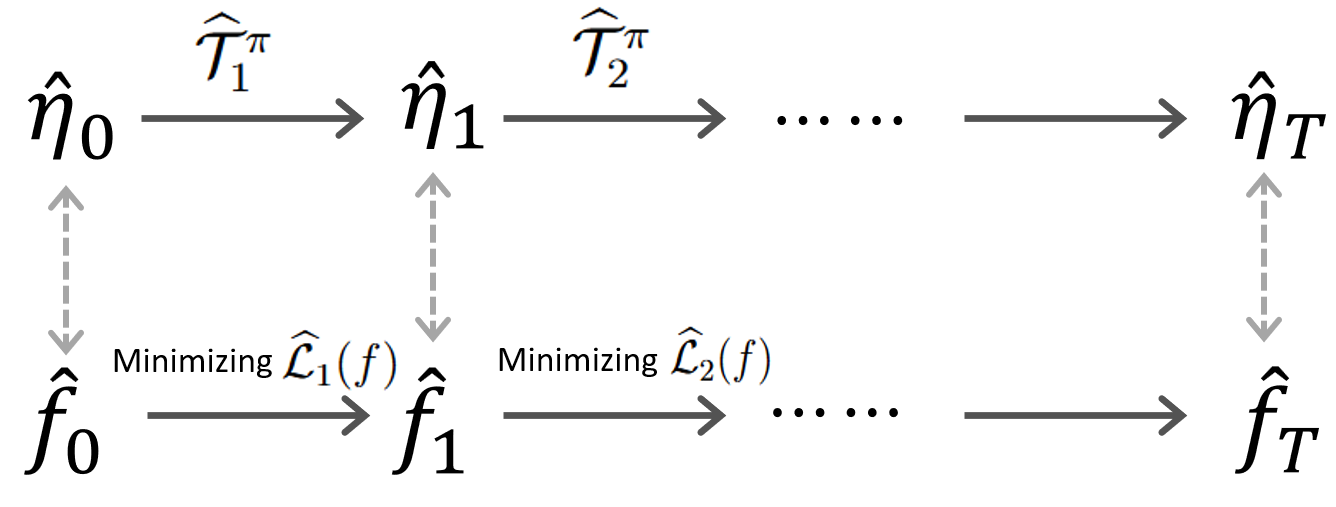}
\vskip -0.25in
\caption{Illustration of the equivalence between distribution iteration in DRL and quantile process training procedure in statistics. Herein,  $\widehat{\mathcal{T}}_t^{\pi}$ denotes the transition operator from $\widehat{\eta}_{t-1}$ to $\widehat{\eta}_{t}$, such that $\widehat{\eta}_t=\widehat{\mathcal{T}}_t^{\pi}\widehat{\eta}_{t-1}$. See Section C  of the Supplemental Material for details. }
% where $\widehat{\mathcal{T}}_t^{\pi}$ represents the empirical distributional Bellman operator based on $\mathcal{D}_t$, producing the distribution $\widehat{\eta}_t=\widehat{\mathcal{T}}_t^{\pi}\widehat{\eta}_{t-1}$ that corresponds to the quantile function $\widehat{f}_t$ at each iteration step (see Section C  for formal definition of  $\widehat{\mathcal{T}}_t^{\pi}$).} 
\label{fig: iteration illustration}
\vskip -0.01in
\end{figure}

%where $\widehat{\mathcal{T}}^{\pi}$ is the empirical distributional Bellman operator defined as the empirical counterpart to \eqref{equ: distributional bellman operator}, $\widehat{\mathcal{T}}^{\pi} \widehat{\eta}_t(s, a) := \sum_{s',r\in\mathcal{D}_t,a'\sim \pi(\cdot|s')} (g_{\gamma, r})_{\#}\widehat{\eta}_t(s',a')$.

In implementation, the dataset is split into $T$ equal subsets $\{\mathcal{D}_t\}_{t=1}^T$, with each of size $n = |\mathcal{D}_t|$, so that $N= nT$.  At iteration $t$, we work with the empirical counterpart of $\mathcal{L}_t(\cdot)$ constructed from the dataset $\mathcal{D}_t$. However, the random variable $Y_t = R+\gamma Z_{t-1}(S',A')$ cannot be directly observed since $Z_{t-1}$ depends on future rewards. To address this issue, we recover $Z_{t-1}(S',A')$ through its quantile function $\widehat{f}_{t-1}(S',A',U)$, where $U\sim \text{Unif}(0,1)$ is independent of $(s,a,R,S')$. Specifically, 
a useful property states that $\widehat{f}_{t-1}(S',A',U)$ has the same distribution as $Z_{t-1}(S',A')$ \footnote{This property enables us recover the sample from $Z_t(S,A)$ by sampling $u_i \sim \text{Unif}(0,1)$ and plugging into $\widehat{f}_t(S,A,u_i)$. Furthermore, selecting a uniform distribution naturally facilitates mean estimation, as $\mathbb{E}[Z(s,a)] = \int_{0}^1 f(s,a,u) d u = \mathbb{E}[f(s,a,U)]$.} (see Proposition D.1 of Supplemental Material).   
For each  transition sample $(s_i,a_i,r_i,s'_i)$, sampling $a'_i\sim \pi(\cdot| s'_i)$ and $u_i\sim\mathrm{Unif}(0,1)$, we generate $\widehat{f}_{t-1}(s_i',a_i',u_i)$ by plugging $(s_i',a_i',u_i)$ into $\widehat{f}_{t-1}$. This allows us to generate the exact samples $r_i + \gamma \widehat{f}_{t-1}(s_i',a_i',u_i)$ from $Y_t$. Consequently, we obtain a collection of i.i.d. samples $(x_i,y_i,\tau_i)$, which we use to define the empirical risk:
\begin{align}\label{equ: empirical loss}
\mathcal{\widehat{L}}_{t}(f) &=\frac{1}{|\mathcal{D}_t|} \sum_{(s_i,a_i,r_i,s'_i)\in\mathcal{D}_t}\rho_{\tau_i}\b(y_i-f(x_i, \tau_i)\b),
\end{align}
where $x_i=(s_i,a_i)$, $\tau_i\sim \mathrm{Unif}(0,1)$, and $y_i=r_i+\gamma \widehat{f}_{t-1}(s'_i,a'_i,u_i)$. The quantile function estimator $\widehat{{f}}_t$ is then obtained by minimizing the empirical risk over certain function space $\mathcal{F}$, such that $\widehat{{f}}_t=\arg\min_{f\in \mathcal{F}} \mathcal{\widehat{L}}_{t}(f)$. In the context of deep quantile regression, the function space $\mathcal{F}$ is typically represented by Deep Neural Networks (DNNs). 
The iterative nature of this procedure is summarized in Algorithm \ref{alg:DOPE}. Figure \ref{fig: iteration illustration} provides a schematic representation of this dynamic iterative process, demonstrating the equivalence between distributional Bellman updates in DRL and the quantile process regression training procedure.

\begin{algorithm}[!tb]
   \caption{\small{Deep Quantile Process regression-based OPE (DQPOPE)}}
   \label{alg:DOPE}
\begin{algorithmic}[1]\small
   \STATE {\bfseries Initialize:}  DNN class $\mathcal{F}$, $\widehat{f}_0\in\mathcal{F}$, datasets $\{\mathcal{D}_t\}_{t=1}^{T}$, target policy $\pi$.
   \FOR{$t =1$ to $T$}
   \STATE Collect sample $(s_i, a_i, r_i, s'_i)\in \mathcal{D}_t$, and sample quantile level $\tau_i\sim\mathrm{Unif}(0,1)$ for all $(s_i,a_i)$.
   %\STATE Impute target distribution,
    %\STATE $\frac{1}{m}\sum_{j=1}^m \delta_{y_j^i} \leftarrow \Psi(f_{t}(s'_i,a'_i,\tau_1^i),\dots,f_{t}(s'_i,a'_i,\tau_m^i))$, $\forall i$, where $a'_i\sim \pi(\cdot|s'_i)$.
    \STATE Generate target sample from $\widehat{\eta}_{t-1}(s',a')$: ~~
     Sample $u_i\sim\mathrm{Unif}(0,1)$ for each $(s'_i,a'_i)$ with $a'_i\sim\pi(\cdot|s'_i)$, and plug $(s'_i,a'_i,u_i)$ into $\widehat{f}_{t-1}(s',a',U)$
   \STATE Compute target sample:  ~~
     $y_i \leftarrow r_i + \gamma \widehat{f}_{t-1}(s'_i,a'_i,u_i)$.
    \STATE Update: ~~
     $\widehat{f}_{t} \leftarrow \arg \underset{f \in \mathcal{F}}{\min}~
   \frac{1}{|\mathcal{D}_t|} \sum_{(s_i,a_i,r_i,s'_i)\in\mathcal{D}_t}\rho_{\tau_i}\b(y_i-f(s_i,a_i, \tau_i)\b)$.
   \ENDFOR
    \STATE {\bfseries Output:}$\widehat{f}_{T}(s,a,\tau)$ {\color{gray}(i.e., $\widehat{\eta}_T(s,a)$)}.
\end{algorithmic}
\end{algorithm}

\subsection{Advantages of Quantile Process Regression over Discrete Quantile Estimation}

To more clearly illustrate the advantages of introducing quantile process regression, we compare our method with previous QDRL approaches, such as QR-DQN \citep{QRDQN} and related approaches \citep{rowland2023analysis}. QDRL methods aim to estimate the conditional $\tau$-th quantile of $Y_{t}$ given $X$
by minimizing the population risk
\begin{align}\label{equ: qrdqn loss}
 \mathcal{L}_{t,\tau}(f)=\mathbb{E}_{X, Y_t}\big(\rho_{\tau}(Y_t-f(X))\big),
\end{align} 
where we recall that $X = (S,A)$ and the target response is $Y_t = R + \gamma Z_{t-1}(S',A')$. To implement this, prior methods estimate multiple quantiles at a set of fixed levels $\{\tau_i\}_{i=1}^m$ by minimizing the aggregate loss $\sum_{i=1}^{m}\widehat{\mathcal{L}}_{t,\tau_i}(f)$, where $\widehat{\mathcal{L}}_{t,\tau_i}(\cdot)$ is some empirical approximation to \eqref{equ: qrdqn loss}. 
Importantly, in these methods, the quantile levels $\{\tau_i\}_{i=1}^m$ are pre-determined and remain fixed throughout the learning process.
In contrast, our quantile regression-based method introduces a fundamental innovation by embedding the quantile level $\tau$ directly as model input. This design enables the model to learn the continuous representation of the quantile function, producing estimates for any quantile level $\tau \in (0,1)$. 
\medskip

\noindent \textbf{Addressing pseudo sample issue}. In the empirical formulation $\widehat{\mathcal{L}}_{t,\tau_i}(\cdot)$ of \eqref{equ: qrdqn loss}, a major challenge arises due to the lack of direct access to the true distribution of $Z_{t-1}$. Since $f(x)$ cannot fully represent a quantile process that reconstructs the original distribution, it is impossible to generate exact samples for the target response $Y_t$. To address this, previous QDRL methods rely on a "pseudo-sample" construction, where $y^p_i = r_i+\gamma\widehat{f}_{t-1,\tau_i}(s_i',a_i')$. Here, $\widehat{f}_{t-1,\tau_i}$ is the estimated conditional $\tau_i$-th quantile at the previous step $t-1$. These pseudo-samples are then used to approximate $Y_t$ with a mixture of Dirac distributions $\frac{1}{m}\sum_{i=1}^{m}\delta_{y^p_i}$. While this approximation becomes exact as $m\to\infty$, finite $m$ introduces unavoidable representation errors, making it challenging to fully recover the target response $Y_t$. This limitation is reflected in the empirical performance of QDRL methods, where performance degradation is more pronounced for small $m$, while increasing $m$ significantly raises computational costs, especially in complex environments. Besides, this discretization issue also complicates the convergence analysis of these methods \citep{drl_book}.
%in quantile-based DRL methods \citep{drl_book}.

In contrast, our method resolves this issue through quantile process regression, which acts as a "generator." By embedding the quantile level $\tau$ as an input and treating it as a continuous random variable, our method directly learns the full quantile function, enabling yields of the exact samples for target response $Y_t$ without requiring additional imputation or discretization \footnote{Previous work use an imputation step, which imposes extra computational burden but fails to generate exact samples,  as discussed in Section A.4 of Supplementary Material.}. This approach eliminates the need for pseudo-samples and enables a seamless representation of the full return distribution, addressing the limitations of discretized quantile methods by prior QDRL methods.

%a "\emph{pseudo sample}" issue arises. Since the true distribution of $Z_{t-1}$ is unobserved and $f(x)$ can not fully represent a quantile process to recover the original distribution, generating exact samples from target response $Y_t$ is inaccessible. Quantile-based DRL methods thus construct a "pseudo sample" for $Y_t$ using $y^p_i = r_i+\gamma\widehat{f}_{t-1,\tau_i}(s_i',a_i')$, where $\widehat{f}_{t-1,\tau_i}$ denotes the conditional $\tau_i$-th quantile estimate obtained at $t-1$ step. These "pseudo sample" are used to construct a mixture of Dirac’s distribution 
%$\frac{1}{m}\sum_{i=1}^{m}\delta_{y^p_i}$ to approximate the target response $Y_t$, and $y^p_i$ can be viewed as drawn from $\frac{1}{m}\sum_{i=1}^{m}\delta_{y^p_i}$. 
%This approximation is  maintained during training constructed by $m$ simultaneous  quantile estimates, and only as $m\to\infty$,  the mixture Dirac’s distribution converges to the target response $Y_t$. For finite $m$, however, an unavoidable representation error persists, making the $m$-quantile representation challenging for convergence analysis of DRL \citep{drl_book}. Contrary, quantile process regression addresses this issue by acting as a "generator", taking random noise (quantile level) as input to yield real data from the original distribution without the extra computational burden \footnote{Previous work use an imputation step but fail to address this issue,  as discussed in Section A.4 of Supplementary Material.} and compromises in full distribution representation due to discretization.
\begin{figure}[!ht]
\vskip -0.1in
\centering
\includegraphics[width = 0.63\linewidth]{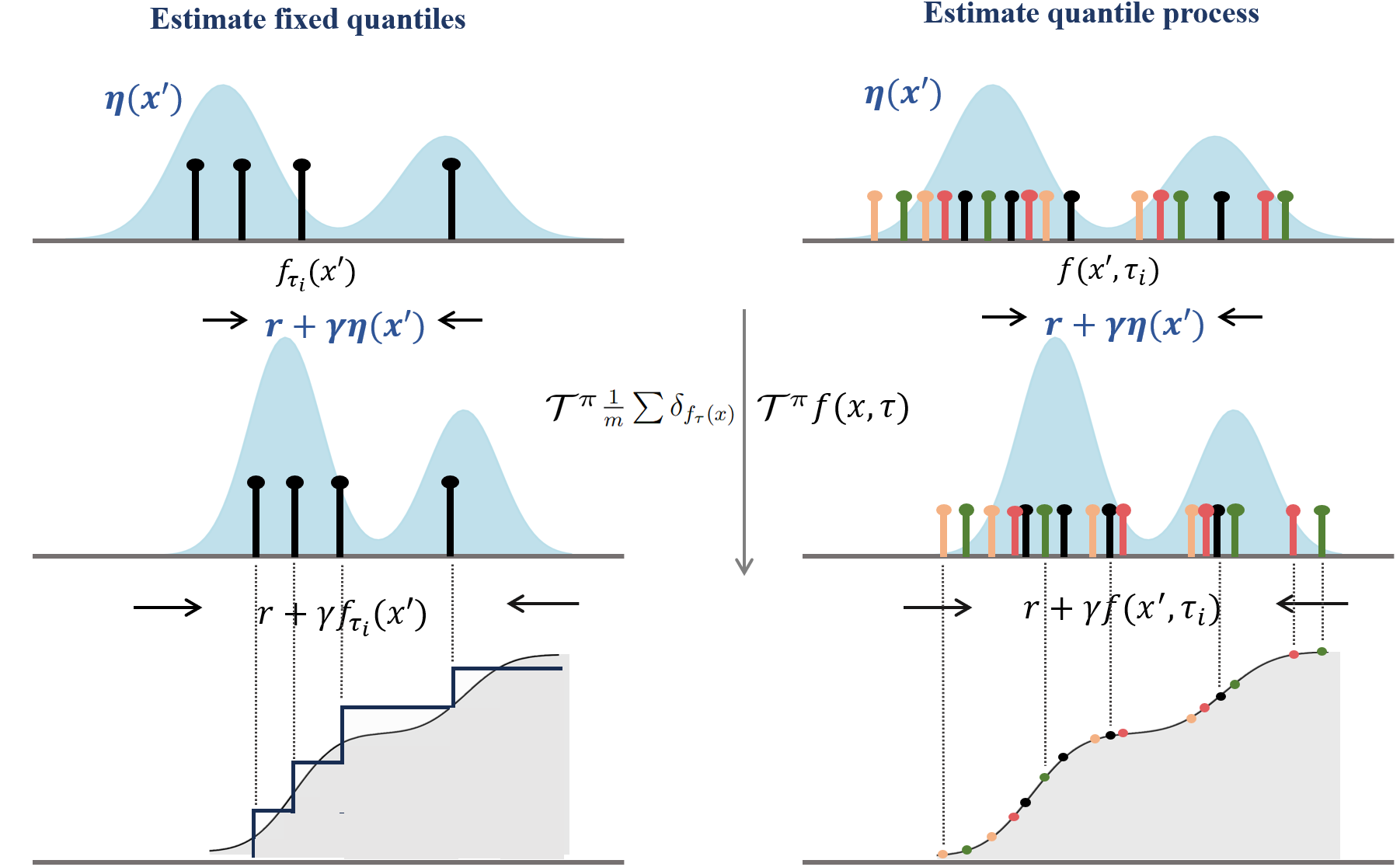}
\vskip -0.15in
\caption{Illustration of estimating quantile process effectively captures the mapping of the distributional Bellman operator. The blue areas symbolize the return distribution, and the markers denote the quantile estimations at certain quantile levels. The bottom line compares the CDF of the true return distribution and the approximated ones}
\label{fig: illustration}
%\vskip -0.3in
\end{figure}  \medskip

\noindent \textbf{Key advantages of quantile process regression}. %The fundamental benefit of learning a quantile process is that transferring an infinite-dimensional distribution learning task to a finite-dimensional regression task. This procedure effectively captures the behavior of the distributional Bellman operator by embedding the quantile level as an input into the function, which aligns the theoretical analysis with practical implementation. 
From a practical perspective, a key advantage of learning a quantile process is transforming an infinite-dimensional distribution-learning problem into a finite-dimensional regression task. By embedding the quantile level as a model input, our method effectively captures the behavior of the distributional Bellman operator. This innovation ensures that the theoretical framework of distributional RL aligns with its practical implementation.

To illustrate this, Figure \ref{fig: illustration} contrasts traditional QDRL methods with quantile process regression-based approaches. The left panel depicts one-step updates for fixed quantile levels, where the estimates are discrete and susceptible to representation gaps. In contrast, the right panel demonstrates the estimation of quantiles at continuously sampled levels, progressively recovering the entire continuous quantile function over $\tau$ during training. This enables quantile process regression-based methods to update the entire distribution seamlessly by operating directly on its quantiles, effectively bridging the gap between theoretical distributional Bellman operators and practical implementation.

\section{Theoretical Results}
In this section, we present non-asymptotic statistical guarantees for DQPOPE when implemented with DNNs. Building on previous studies in deep nonparametric regression \citep{SchmidtHieber2020, farrell2021deep}, our analysis focuses on neural networks with rectified linear unit (ReLU) activation. While our theoretical analysis is grounded in the ReLU setting, it can be extended to more general DNN architectures.

\begin{definition}[ReLU network]\label{def: ReLU Network}
The ReLU network $f : \mathbb{R}^{m_0}\to\mathbb{R}^{m_{L+1}}$ is defined by
\begin{align}\label{equ: ReLU}
f(x)=\phi_{L}(\sigma(\phi_{L-1}(\cdots\sigma(\phi_{0}(x)))),
\end{align}
where $\sigma(x)=\max \{x, 0\}$ is the ReLU activation function, $\phi_{\ell}(x)=A_\ell x+b_\ell$, and $A_{\ell} \in \mathbb{R}^{m_l \times m_{\ell-1}}, b_{\ell} \in \mathbb{R}^{m_\ell}$ are the weight matrix and bias vector in $\ell$-th layer, respectively. Particularly, we consider the first layer width $m_{0}=d$ and the last $m_{L+1}=1$, the maximum width $W=\max_{\ell\in[L]} m_{\ell}$, 
% the number of parameters of the neural network $S = \sum_{\ell=0}^{L}m_{\ell+1}(m_\ell+1)\leq \mathcal{O}(W^2L)$, 
and the sup-norm of the function $\|f\|_{\infty} \leq F$. 
We denote the class of such functions by $\mathcal{F}:=\mathcal{F}(W,L)$.
%In deep learning, the goal is to learn the network parameters $\{A_\ell,b_\ell\}_{\ell\in[L+1]}$.
\end{definition}

% \[
%  \mathcal{F} : =  \mathcal{F}(W,L) = \left\{ \min \{\max \{-F, f \}, F\}  : \text{$f$ defined in \eqref{equ: ReLU} } \right\}
% \]

 Recent studies \citep{fan2020theoretical, nguyen2022sample,ji2022sample} have investigated RL algorithms using ReLU neural networks, where the Bellman target is assumed to belong to the H\"{o}lder or Besov space. 
 To capture the smoothness of the Bellman operator $\mathcal{T}^{\pi}$, we define the H\"{o}lder class $\mathcal{G}:=\mathcal{G}([0,1]^d, \beta, H)$ as follows.

%Moreover, recent developments in the approximation theory of neural networks to smooth functions have been well-established \citep{yarotsky2018optimal,petersen2018optimal}. 

\begin{definition}[H\"{o}lder class] \label{def: holder class}
Let $\beta=r+s$, where $s\in\mathbb{N}, r \in (0,1]$. The class of H\"{o}lder smooth functions is defined by
\begin{small}
\begin{align}\label{equ: holder}
\mathcal{G}([0,1]^d, \beta, H)=\bigg\{ f:[0,1]^d \rightarrow \mathbb{R}: \max_{\boldsymbol{\alpha}:\|\boldsymbol{\alpha}\|_1\le s}\left\|\partial^{\boldsymbol{\alpha}}f\right\|_{\infty}+    \max_{\|\boldsymbol{\alpha}\|_1=s} \sup _{x \neq y} \frac{\left|\partial^{\boldsymbol{\alpha}} f(x)-\partial^{\boldsymbol{\alpha}}(y)\right|}{\|x-y\|_{2}^{r}} \leq H\bigg\},
\end{align}
\end{small}
where $H>0$, $\boldsymbol{\alpha}=\left(\alpha_1, \ldots, \alpha_d\right)^{\top} \in \mathbb{N}^d$,  and $\partial^{\boldsymbol{\alpha}} =\partial^{\alpha_1} \ldots \partial^{\alpha_d}$ is the multi-index notation.
\end{definition}
Throughout this section, write  $\wt \mu : = \mu\times \text{Unif}(0,1)$ 
as the product measure of $\mu $ and $\text{Unif}(0,1)$.
The following technical assumptions are introduced to support analysis.
%\subsection{Assumption}
\begin{assumption}[Coverage]\label{assum: coverage}
Given $d^{\pi}$ and $\mu$ in Section \ref{sec_2}, there exists a constant $C_{\mu}$ such that
\begin{align*}
\sup_{\eta,\eta' \in \Delta(\mathbb{R})^{\mathcal{S}\times\mathcal{A}}}\frac{ \overline{\mathcal{W}}_{p,d^{\pi}}\big( \eta, \mathcal{T}^{\pi}\eta'\big)} {\overline{\mathcal{W}}_{p,\mu}\big( \eta, \mathcal{T}^{\pi}\eta'\big)}\leq C_{\mu}.
\end{align*}
\end{assumption}

%Here we slightly abuse the Wasserstein distance and Bellman operator notation on $f$, and the same as below.

%Even two distributions $\d^{\pi}$ and $\mu$ that are sufficiently disparate might admit a reasonable transfer. Moreover, our assumption incorporates the function class into the definition, captures the crucial role of function approximation in generalizing across different states. providing a more refined measure of how well Bellman errors under $\pi$ transfer between the distributions $d^{\pi}$ and $\mu$. Here we slightly abuse the Wasserstein distance and Bellman operator notation on $f$, and the same as below.

\begin{assumption}[Bellman completeness]\label{assum: Bellman completeness}
We assume that for any $\eta \in \Delta(\mathbb{R})^{\mathcal{S}\times\mathcal{A}}$, if its corresponding quantile function $f$ belongs to $\mathcal{F}$, the quantile function  of $\mathcal{T}^{\pi}\eta $ belongs to $ \mathcal{G}$.
\end{assumption} 

The data coverage and completeness assumptions are standard and widely employed in RL theory literature \citep{munos2008finite,chen2019information}. Unlike the classical data coverage, which bounds the distribution ratio $\|\frac{d^\pi}{\mu}\|_{\infty}:=\sup_{s,a}\frac{d^\pi(s,a)}{\mu(s,a)}$ across state-action pairs $(s,a)$, we measure how well Bellman errors transfer between the distributions $d^{\pi}$ and $\mu$,
% by explicitly \redtext{incorporating the function class $\mathcal{F}$ into the definition \citep{xie2021bellman,mle-drl},} 
offering a tighter measure than $\|d^{\pi}/\mu\|_{\infty}$. Even though two distributions $d^{\pi}$ and $\mu$ are substantially disparate, this discrepancy can still be effectively quantified. 
The completeness assumption specifies that if quantile function of any $\eta\in\Delta(\mathbb{R})^{\mathcal{S}\times\mathcal{A}}$ belongs to $\mathcal{F}$, then the Bellman operator $\mathcal{T}^{\pi}$ applied on  $\eta$ results in the quantile function of $\mathcal{T}^{\pi}\eta$ sitting in $\mathcal{G}$. This assumption is mild and holds for most common smooth dynamics with concrete examples elucidated in \citet{fan2020theoretical}.  Please refer to Section G of the Supplementary Material for a detailed justification.

%Note that our assumption explicitly incorporates the function class $\mathcal{F}$ into the definition, which captures the power of function approximation in generalizing across state-action pairs. Even two distributions $d^{\pi}$ and $\mu$ that are substantially disparate, this discrepancy can still be effectively quantified. 

% The completeness assumption is commonly used in literature \citep{chen2019information,xie2021bellman}. 

\begin{assumption}[Strong convexity]\label{assum: convexity}
There exists a universal constant $c_0>0$ such that for any  $t\in [T]$ and any  function $f\in \mathcal{F}$,
% satisfying $\|f- f^*_t\|_{2,\mu\times\tau}^2\le b $ 
% with $b$ being arbitrarily small constant, remove, 
we have 
\begin{align*}
\mathcal{L}_t(f)-\mathcal{L}_t(f^*_t)~\geq~ c_0~ \|f- f^*_t\|_{2,\wt \mu }^2.
\end{align*}
\end{assumption}
In contrast to the previous two standard assumptions in  RL literature, Assumption \ref{assum: convexity} introduces a $c_0$-strong convexity condition of the population risk of quantile loss. Under a mild Assumption G.1, which requires the density of conditional distribution $Y_t$ given $(s,a)$ near  $f^*_t$ to be bounded away from zero, one can ensure that Assumption \ref{assum: convexity} always holds. It is worth mentioning that Assumption \ref{assum: convexity}  is crucial for establishing Lemma \ref{lem: Bounding Wasserstein}. This convexity property is widely used in non-parametric quantile regression literature \citep{belloni2011}. See Section G of the Supplementary Material for more details.

%\begin{remark}
%    In practice, QDRL adopts a modified quantile loss, known as the quantile huber loss \citep{huber1992robust, QRDQN}, that becomes quadratic smooth near zero errors and remains linear for larger errors. This property helps mitigate the effects of outliers and provides a differentiable loss function for optimization, preventing the gradient from becoming constant when the error goes to zero. Additionally, this modification introduces a local convexity near the target quantile function which matches the Assumption \ref{assum: convexity} regarding the local convexity of population risk. 
%\end{remark}

\subsection{Preliminary Results}\label{sec: Main Results}

This part introduces the analytical framework for distributional OPE and presents the preliminary results of DQPOPE.  We assume $ \sup_{t \ge 0 }~ | R_t |\le  R_{max}$, a common assumption in the RL literature for simplifying the analysis. This condition is not essential and can be relaxed to more general assumptions, such as sub-Gaussian tails. We provide a detailed discussion in Section G.1 of the Supplementary Material.
Let  $C_{F,R} = F+R_{max}$ and 
define $\widehat{\varepsilon}_{p, t}:=\overline{\mathcal{W}}_{p,\mu}(\widehat{\eta}_{t},\mathcal{T}^{\pi}\widehat{\eta}_{t-1})$ for each $t\in [T]$, where $\{\widehat{\eta}_{t}\}_{t\in [T]}$ are % the return distributions 
obtained by Algorithm \ref{alg:DOPE}.  

To the aim of bounding $\mathcal{W}_{p}(\boldsymbol{\eta}^{\pi},\widehat{\boldsymbol{\eta}})$, we start with the following 
decomposition. 

\begin{lemma}[Sub-optimality decomposition]\label{lem: Sub-optimality decomposition}
Suppose that Assumption \ref{assum: coverage} is satisfied. %and   there exists a finite constant $R_{\text{max}}$ such that $|R_t| \leq R_{\text{max}}$ for all $t\in[T]$. 
Then,  the sub-optimality of  $\boldsymbol{\widehat{\eta}}_{T}$  satisfies 
\begin{align*}
\mathcal{W}_p(\boldsymbol{\eta}^{\pi},\boldsymbol{\widehat{\eta}}_{T})~ \leq~  \frac{2 C_{\mu}^{\frac{1}{2p}}}{(1-\gamma)^{\frac{3}{2}}}~ \max_{0<t\leq T}\widehat{\varepsilon}_{p, t} + \frac{\gamma^{\frac{T}{2}}}{(1-\gamma)^{\frac{3}{2}}}~ C_{F,R} %R_{max}. %+ {\red \frac{2}{m} R_{max}}.
\end{align*}
\end{lemma}

Lemma \ref{lem: Sub-optimality decomposition} is motivated by the error propagation analysis from the RL literature \citep{antos2008learning,mle-drl}. %and the distributional RL setting \citep{mle-drl}. 
This error propagation illustrates that the total error of distributional OPE can be interpreted as a sum of statistical error and algorithmic error. The statistical error arises from the one-step Bellman error,  $\widehat{\varepsilon}_{p, t}$,  %, a weighted-norm error derived from the quantile process regression procedure.
 which can be explicitly associated with the excess risk of the quantile loss by applying Lemma \ref{lem: Bounding Wasserstein}. %The constant $C_{\mu}$ appears in the bound due to the change-of-measure argument used to propagate the error bounds across iterations, and the multiplier $(1-\gamma)^{-\frac{3}{2}}$ reflects the impact of the metric $\overline{\mathcal{W}}_{p,\mu}(\cdot,\cdot)$ on the performance of the distributional policy iteration. 
The algorithmic error is specific to the iterative nature of the RL setting and has no counterpart in regression settings. It reflects the error that remains after executing a finite number of iterations $T$. 
% \redtext{It's worth noting that the second term is controlled by $C_{F,R}$, a constant depending on $F$ and the bounded reward $R_{max}$.} 
%and the case of the bounded reward, 
%The bounded reward is a common assumption in the RL literature. %Even in scenarios with unbounded reward distributions, the algorithmic error can still be controlled if the tail distribution of the reward is well characterized, such as sub-Gaussian random variables with an exponential tail decay. 
In conclusion, the key component in the analysis hinges on the one-step Bellman error, $\widehat{\varepsilon}_{p,  t}$, as it ultimately controls the error propagation. 

For $p=1$, the one-step Bellman error can be bounded by excess risk, stated as follows. 
\begin{lemma}
\label{lem: Bounding Wasserstein}
Suppose Assumption \ref{assum: convexity} is satisfied. Then for each $t\in [T]$,  we have
    \begin{align*}
   \widehat{\varepsilon}_{1, t} =      \overline{\mathcal{W}}_{1,\mu}(\widehat{\eta}_{t},\mathcal{T}^{\pi}\widehat{\eta}_{t-1})
        ~ \leq~  c_0^{-\frac{1}{2}} ~ \big(\mathcal{L}_t(\widehat{f}_t)-\mathcal{L}_t(f^*_t)\big)^{\frac{1}{2}},
    \end{align*}
    where $c_0>0 $ is a constant introduced in Assumption \ref{assum: convexity}.
\end{lemma}
Lemma \ref{lem: Bounding Wasserstein} provides a key step in linking the bound for the Wasserstein metric to excess risk, enabling the analysis of distributional OPE in a non-parametric regression manner. Combined with Lemma \ref{lem: Sub-optimality decomposition}, it reduces the problem to bound excess risk for each step.

\begin{theorem}[Excess risk bound, slow rate]\label{lem: excess risk bound}
Suppose Assumption \ref{assum: Bellman completeness} is satisfied.
With probability at least $1-2n^{-1}$, the excess risk satisfies 
\begin{align}\label{equ: excess risk bound preliminary}
\mathcal{L}_t(\widehat{f}_t)-\mathcal{L}_t(f^*_t)~ \leq~  C \sqrt{\frac{W^2 L^2\log(W^2L) \log n}{n}} + \inf_{f\in\mathcal{F}}\left(\mathcal{L}_t(f)-\mathcal{L}_t(f^*_t)\right), 
\end{align}
 where $C$ is a constant independent of $W,L,n $. Furthermore, for sufficiently large $U, V\in \mathbb{N}^+$, 
 setting width and length to be $W=\mathcal{O}\left((s+1)^2 d^{s+1} U \log U\right)$ and  $L=\mathcal{O}\left((s+1)^2 V \log V\right)$, if we choose $UV=\lfloor n^{\frac{d}{4\beta+2d}} \rfloor$,  when $n$ is sufficiently large, with probability at least $1-2n^{-1}$, the excess risk has upper bound that
\begin{align}\label{equ: excess risk bound preliminary optimal bound}
\mathcal{L}_t(\widehat{f}_t)-\mathcal{L}_t(f^{*}_t)~ \leq~  C(s+1)^4 d^{s+(\frac{\beta }{2}\vee 1)} (\log n)^3 n^{-\frac{\beta}{2\beta+d}},
\end{align}
where $C$ is a constant independent of $s,\beta, d, n$. 
\end{theorem}

%The bound in \eqref{equ: excess risk bound preliminary} admits the usual decomposition into a stochastic term and an approximation term. The first term reflects estimation error in nonparametric quantile-process regression based on ReLU networks, whereas the second term captures the approximation bias induced by restricting the estimator to the class $\mathcal F$. For fixed $W$ and $L$, the stochastic term is of order $n^{-1/2}$ up to logarithmic factors. Combined with Lemma \ref{lem: Bounding Wasserstein}, which translates excess risk into one-step Bellman error, this yields a slower rate of order $n^{-1/4}$, in contrast to the $n^{-1/2}$ rate typically obtained in standard OPE \citep{chen2019information}.

%The upper bound in \eqref{equ: excess risk bound preliminary} typically captures the error source in implementing non-parametric quantile process regression using ReLU networks, which comprises the stochastic error of the first term and the approximation error of the second term. Ignoring the approximation error, the  stochastic error scales as $n^{-\frac{1}{2}}$ for fixed width $W$ and length $L$. 

The bound in \eqref{equ: excess risk bound preliminary} tracks the usual decomposition into a stochastic term and an approximation term. The first term reflects estimation error in nonparametric quantile process regression-based on ReLU networks, whereas the second term captures the approximation bias induced by restricting the estimator to the class $\mathcal F$. Combined with Lemma \ref{lem: Bounding Wasserstein}, which translates excess risk into one-step Bellman error, this yields a slower rate of order $n^{-1/4}$, in contrast to the $n^{-1/2}$ rate typically obtained in standard OPE \citep{chen2019information}.
%By carefully choosing a configuration of the width and length according to $UV$, we strike a balance between these two types of errors. 

The approximation term quantifies the error incurred by approximating the H\"older class $\mathcal G$ with the ReLU class $\mathcal F$. it is closely related to the \emph{inherent Bellman error} \citep{munos2008finite}, $\sup_{g \in \mathcal{G}} \inf_{f \in \mathcal{F}}\|g-f\|_{1,\wt \mu }$, which reflects how well the function class $\mathcal{F}$ is aligned with the Bellman image of the target class. By utilizing approximation theory (Lemma H.3), $\inf_{f\in\mathcal{F}}(\mathcal{L}_t(f)-\mathcal{L}_t(f^*_t))\leq \inf_{f\in\mathcal{F}}\|f-f^*_t\|_{1,\wt \mu }\leq C(U V)^{-\frac{2 \beta}{d}}$. 
It is also worth mentioning that the control of approximation error can be further refined with additional smoothness conditions.  We will provide a tighter bound in Theorem \ref{thm: excess risk bound faster rate}.

The stochastic error characterizes the variance in estimating the quantile functions, depending on both the richness of network  and sample size.   A smaller network reduces variance but increases approximation bias, whereas a larger network improves expressivity at the price of a larger stochastic error. The scaling of $W$ and $L$ in Theorem \ref{lem: excess risk bound} balances these two effects and leads to the convergence rate $n^{-\frac{\beta}{2\beta+d}}$.

\begin{theorem}\label{thm: main results}
Suppose Assumptions \ref{assum: Bellman completeness}, and \ref{assum: convexity} are satisfied. For   each $t\in[T]$, using the same choice of $L$ and $W$ as in Theorem \ref{lem: excess risk bound},   when $n$ is sufficiently large, with probability at least $1-2n^{-1}$,  the one-step Bellman error has upper bound that
\begin{align*}
  \overline{\mathcal{W}}_{1,\mu}(\widehat{\eta}_{t},\mathcal{T}^{\pi}\widehat{\eta}_{t-1})~ \leq~ C(s+1)^2 d^{s/2+(\frac{\beta }{4}\vee \frac{1 }{2})}(\log n)^{\frac{3}{2}} n^{-\frac{\beta}{4\beta+2d}},
\end{align*}
where $C$ is a constant independent of $s,\beta, d, n$. If Assumption \ref{assum: coverage} further holds and $T=\mathcal{O}(\xi\log N)$ for some constant $\xi>0$, when $N$ is sufficiently large, with probability at least $1-c N^{-1}(\log N)^{2}$,   the sub-optimality of $\boldsymbol{\widehat{\eta}}_{T}$ has upper bound that
\begin{align} \label{equ: sub-optimality bound for qpope slow}
\mathcal{W}_1(\boldsymbol{\eta}^{\pi},\boldsymbol{\widehat{\eta}}_{T})~ \leq~ \frac{C C_{\mu}^{\frac{1}{2}}}{(1-\gamma)^{\frac{3}{2}}}(\log N)^{\frac{5}{2}} N^{-\frac{\beta}{4\beta+2d}} +  \frac{N^{\frac{\xi\log \gamma}{2}}}{(1-\gamma)^{\frac{3}{2}}}~C_{F,R} ,
\end{align}
where $0<\gamma<1$, and $c, C$ are  constants  independent of $C_\mu, N,\gamma$.  
\end{theorem}

%Building on previous results, we can immediately derive bounds for one-step Bellman error and sub-optimality of DQPOPE. The sub-optimality bound \eqref{equ: sub-optimality bound for qpope slow} follows directly by accumulating one-step Bellman error with substitution $n = N/T$, and we compare it against standard OPE algorithms using ReLU networks under comparable assumptions.

Theorem \ref{thm: main results} follows by combining the excess risk bound with the control of the one-step Bellman error and then propagating the error across $T$ iterations, with the substitution $n=N/T$.
 The constant $C_{\mu}$ arises from the change-of-measure argument used to propagate the error across iterations, and the multiplier $(1-\gamma)^{-3/2}$ reflects the impact of the metric $\overline{\mathcal{W}}_{p,\mu}(\cdot,\cdot)$ on the performance of the distribution iteration. For clarity,  we suppress in $C$ its dependence on the H\"older smoothness parameters  $s, \beta$ and the dimension $d$.

 %we omit the factor hidden in $C$, such as smooth constant $s, \beta$ defined in Definition \ref{def: holder class}, and the dimension $d$.  
 
 The non-asymptotic statistical error bound attains  $N^{-\frac{\beta}{4\beta+2d}}$,  which is slightly slower than $N^{-\frac{\beta}{2\beta+2d}}$ \citep{nguyen2022sample} with respect to $\beta$, and slower than  $N^{-\frac{\beta}{2\beta+d}}$ \citep{fan2020theoretical,ji2022sample} with
respect to both $\beta$ and $d$. Here,  $\beta$ represents the smooth parameter of the Bellman operator, and $d$ is the input dimension. Consequently, this error bound does not achieve the best possible rate within the OPE setting. The slower rate  can be intuitively attributed to the increased complexity of learning a distribution rather than a scalar. Technically,  it stems from a coarse error decomposition of excess risk and the reliance on Rademacher complexity to control the stochastic error. This motivates a sharper localized analysis. Furthermore, we extend our results beyond the i.i.d. setting in Section H of the Supplementary Material.

\begin{remark}
   Note that $N = nT$ is the total sample size across $T$ steps, where the choice of $T=\mathcal{O}(\xi\log N)$ results from data splitting strategy. Since the second term in \eqref{equ: sub-optimality bound for qpope slow} is polynomial dependent on $N$ as $ \gamma^{\frac{T}{2}}=\gamma^{\frac{\xi \log N}{2}} = N^{\frac{\xi \log \gamma}{2}}$,  we require $\xi\ge \frac{2\beta}{(4\beta+d)\log (1/\gamma)}$ to ensure that $\gamma^{\frac{\xi\log N}{2}}\lesssim N^{-\frac{\beta}{4\beta+2d}}$ in \eqref{equ: sub-optimality bound for qpope slow}. Data splitting removes the correlation between $\widehat{f}_t$ and data that may arise from reusing the same sample batch.  Although this simplifies the theoretical analysis, we do not apply data splitting in our experiments. While a more refined analysis could potentially eliminate the need for this technique, we leave this for future work. 
\end{remark}

\subsection{Fast Rate of Excess Risk Bound}\label{sec: Faster rate for excess risk bound}

In this subsection, we derive a sharper excess risk bound. Recent studies \citep{rowland2024near, zhang2023estimation} reveal that estimating the return distribution requires no more samples than estimating the mean value.  However, these works come with trade-offs in analysis. In particular, they focus on categorical distributional temporal difference (TD) learning, which approximates the return distribution with a discrete distribution. Their analysis relies on access to a generative model and is confined to the tabular setting, which is impractical for modern RL applications involving large state spaces and limited offline data.   This highlights the importance of showing the sample complexity results of distributional OPE that are comparable to those available for  standard OPE \citep{nguyen2022sample,ji2022sample} under more realistic settings. Our goal is to show that the excess risk bound can attain an optimal rate of $N^{-\frac{2\beta}{2\beta+d}}$, consistent with theoretical expectations.

%However, unlike the squared loss setting, where a faster rate may be expected, the quantile regression requires additional assumptions to achieve comparable results. 
Achieving such a fast rate is substantially more delicate for quantile regression than for squared-loss regression. Recent work on non-parametric quantile regression with ReLU networks, such as \cite{padilla2022quantile}, show that the prediction error of conditional quantile estimator at the quantile level $\tau = 0.5$ can achieve the minimax rate as shown in \cite{SchmidtHieber2020}, given Gaussian errors and uniformly distributed covariates in $[0, 1]^d$; \cite{shen2024} establish a minimax lower bound under a strong convexity condition (Assumption \ref{assum: convexity}), whereas the corresponding upper bound remains slower than the minimax rate. To derive a faster rate, one may expect that the excess risk of quantile loss could exhibit some nice local quadratic structure, analogous to that of squared loss.

%show that the lower bound for prediction error can match the minimax rate in the sense of \cite{stone1982optimal} under the strong convexity condition (Assumption \ref{assum: convexity}),  

%we introduce an alternative condition

%However, unlike the squared loss setting, where one may expect a faster rate scale as $n^{-1}$, the stochastic error rate of excess risk of quantile loss naturally scales as $n^{-\frac{1}{2}}$, and achieving a faster rate requires appropriate assumptions. 

\begin{assumption}[Local strong convexity and smoothness]\label{assum: convexity-smoothness}
There exist two universal constants  $c_0'\ge c_0>0$ such that for any $t\in [T]$, and for any  $f\in \mathcal{F}$ satisfying $\|f- f^*_t\|_{2,\wt \mu }^2\le b_n$, where  $b_n=  C_b (\log n)^6 n^{-\frac{2\beta}{2\beta+d}}$ with $C_b$ being  some constant independent of $n$,
we have
% For $f_t\in \mathcal{F}$ and any $t\in [T]$, and $\tau$ is independent of $(s,a)\sim \mu$, there exists universal constant  $c_0'\ge c_0>0$, such that for the target quantile function $f^*_t:=\mathcal{T}^\pi f_{t-1} \in \mathcal{G}$ {satisfying $\mathbb{E}_\tau
% \|f_t(\cdot,\tau)- f^*_t(\cdot,\tau)\|_{2,\mu}^2\le b $ with $b$ being arbitrarily small constant}, we have
%\begin{align*}
%\big|\mathcal{L}(f)-\mathcal{L}(f')\big|\geq c_0\mathbb{E}_\tau
%\|f(\cdot,\tau)- f'(\cdot,\tau)\|_{2,\mu}^2,
%\end{align*}
\begin{align}\label{equ: radius of local property}
c_0
\|f- f^*_t\|_{2,\wt \mu }^2 ~ \le~ \mathcal{L}_t(f)-\mathcal{L}_t(f^*_t)~ \le~  c'_0\|f- f^*_t\|_{2,\wt \mu }^2.
\end{align}
\end{assumption}

%Compared to Assumption \ref{assum: convexity}, 
Assumption \ref{assum: convexity-smoothness} further imposes a smoothness condition by the RHS inequality, implying that the population risk of quantile loss maintains curvature characteristics similar to the squared loss nearby around the target quantile function $f^*_t$.  For squared loss, it naturally holds with $c_0=c'_0=1$.  Intuitively, this assumption establishes a pivotal link between quantile loss and squared loss in the neighborhood around the target function $f^*_t$, allowing a faster convergence rate for excess risk. Technically, a similar assumption is also employed by \cite{farrell2021deep}, where condition  \eqref{equ: radius of local property} is required for any $f\in \mathcal{F}$. 
However,  our analysis relaxes this condition by requiring \eqref{equ: radius of local property} only within a neighborhood around the target function $f_t^*$, where the radius of the neighborhood is controlled by a shrinking sequence $b_n$. For additional details,
% and the explicit dependence of $C_b$ on $F,s,\beta,d,c_0,c'_0$, 
refer to Remark E.1 in the Supplementary Material.
%We thus refine the previous results in Theorem \ref{lem: excess risk bound} by using  Assumption \ref{assum: convexity-smoothness}.
% Meanwhile,  attaining the minimax optimal rate of the excess risk bound. 
% More notably, the locality parameter $b_n$ shrinking to zero as $n$ grows to infinity.

% Previously, Assumption \ref{assum: convexity} has been used only in Lemma \ref{lem: Bounding Wasserstein}, 
% %establish a connection between the excess risk and the 1-Wasserstein distance, 
% without directly contributing to bound excess risk.  

%but we relax it by allowing the this local quadratic structure to hold within a neighborhood of radius $b$ around the target function $f_t^*$.

\begin{remark}
 Similar to Assumption \ref{assum: convexity}, Assumption \ref{assum: convexity-smoothness} holds under a mild Assumption G.1 in the Supplementary Material. Specifically, the smoothness condition requires that the density of conditional distribution $Y_t$ given $(s,a)$ near  $f^*_t$ is bounded. %, which is minor. 
 % Assumption \ref{assum: convexity-smoothness} enables a different error decomposition
 %    %in    Theorem \ref{thm: excess risk bound faster rate}
 %    (see Remark ... in the Supplementary Material for details), 
 %    which is pivotal in deriving a minimax optimal rate. 
\end{remark}

We assume the existence of $f_{t,\mathcal{F}}=\text{argmin}_{f\in\mathcal{F} }\|f-f^{*}_t\|^2_{2,\wt \mu } $ for each $t\in [T]$ in the rest of main text, which is a common requirement in the relevant literature \citep{farrell2021deep}.  We establish the fast rate of excess risk bound in Theorem \ref{thm: excess risk bound faster rate}. 

%the smoothness condition allows the approximation error quantified by $\inf_{f\in\mathcal{F}}\|f-f^*_t\|^2_{2,\mu}$. In contrast, in Lemma \ref{lem: excess risk bound}, the excess risk $\mathcal{L}(f_t)-\mathcal{L}(f^*_t)$ can be decomposed into stochastic error $\sup_{f\in\mathcal{F}} \big|\mathcal{L}(f)-\widehat{\mathcal{L}}(f) \big|$ and approximation error $\inf_{f\in\mathcal{F}}(\mathcal{L}(f)-\mathcal{L}(f^*_t))$. The approximation error can be further quantified by $\inf_{f\in\mathcal{F}}\|f-f^*_t\|_{1,\mu}$ utilizing the fact that 1-Lipschitz continuity of $\mathcal{L}(f)-\mathcal{L}(f^*_t)\leq \|f-f^*_t\|_{1,\mu}$. 

\begin{theorem}[Excess risk bound, fast rate]\label{thm: excess risk bound faster rate}
Suppose Assumption \ref{assum: Bellman completeness} and \ref{assum: convexity-smoothness} are satisfied. With probability at least $1-c\exp(-W^2 L^2 \log (W^2 L)\log n)$, the excess risk satisfies
\begin{align}\label{equ: upper_bound_2}
\mathcal{L}_t(\widehat{f}_t)-\mathcal{L}_t(f^*_t) ~\le~ C~  \frac{{W^2 L^2 \log (W^2 L)\log n}}{{n}}+ 2 ~  \mathcal{A}_t,
\end{align}
where $c,C$ are  constants independent of $W,L,n $, and $\mathcal{A}_t := \| f_{t,\mathcal{F}}-f^*_t\|^2_{2, \wt \mu }$. 
Furthermore, using the same choice of $L$ and $W$ as in Theorem \ref{lem: excess risk bound}, when $n$ is sufficiently large,    with probability at least $1-c\exp(-n^{\frac{2d}{2d+4\beta}}\log n)$,  the excess risk has upper bound that
$$
\mathcal{L}_t(\widehat{f}_t)-\mathcal{L}_t(f^*_t) ~\le~ C~  (s+1)^8 d^{2s+(\beta \vee 2)}~  (\log n)^6n^{-\frac{2\beta}{2\beta+d}},
$$
where $c,C$ are  constants independent of $s,\beta, d, n$. 
\end{theorem}

%We establish the fast rate of excess risk bound, setting the foundation for deriving the ultimate sub-optimality bound.
In \eqref{equ: upper_bound_2}, the first term of the RHS is derived to control the stochastic error, and the second term measures the approximation error. 
For fixed width $W$ and length $L$, the stochastic error scales as $n^{-1}$, an improvement over $n^{-1/2}$ in Theorem \ref{lem: excess risk bound}. By selecting the appropriate  width $W$ and length $L$, both errors scale as $n^{-\frac{2\beta}{2\beta+d}}$. Theorem \ref{thm: excess risk bound faster rate} thus yields a faster rate, which attains the minimax rate $n^{-\frac{2\beta}{2\beta+d}}$ as established in \cite{stone1982optimal} for the $d$-dimensional non-parametric regression function with smoothness index $\beta$.

\begin{theorem}\label{thm: main results with faster rate}
Suppose Assumptions \ref{assum: Bellman completeness}, and \ref{assum: convexity-smoothness} are satisfied. For  each $t\in[T]$, using the same choice of $L$ and $W$ as in Theorem \ref{lem: excess risk bound}, 
when $n$ is sufficiently large, with probability at least $1-c\exp(-n^{\frac{2d}{2d+4\beta}}\log n)$,  the one-step Bellman error has upper bound that
\begin{align*}
\overline{\mathcal{W}}_{1,\mu}(\widehat{\eta}_{t},\mathcal{T}^{\pi}\widehat{\eta}_{t-1})~\le~ C~  (s+1)^4 d^{s+(\frac{\beta}{2} \vee 1)} (\log n)^3n^{-\frac{\beta}{2\beta+d}},
\end{align*}
where $c,C$ are constants independent of $s,\beta, d, n$. If Assumption  \ref{assum: coverage} further holds and $T= \mathcal{O}(\xi \log N)$ for some constant $\xi>0$, when $N$ is sufficiently large, with probability at least $1-c\log N\exp(-(N/\log N)^{\frac{2d}{2d+4\beta}})$, 
the sub-optimality of $\boldsymbol{\widehat{\eta}}_{T}$ has an upper bound that
\begin{align} \label{equ: final faster rate bound}
\mathcal{W}_1(\boldsymbol{\eta}^{\pi},\boldsymbol{\widehat{\eta}}_{T})~\le~ \frac{C C_{\mu}^{\frac{1}{2}}}{(1-\gamma)^{\frac{3}{2}}}~ (\log N)^4 N^{-\frac{\beta}{2\beta+d}} +  \frac{N^{\frac{\xi \log \gamma}{2}}~C_{F,R}}{(1-\gamma)^{\frac{3}{2}}}, 
\end{align}
where $0<\gamma<1$, and $c, C$ are constants  independent of $C_\mu, N,\gamma$.  
\end{theorem}

To ensure that the error rate matches  $N^{-\frac{\beta}{2\beta+d}}$, the second term in \eqref{equ: final faster rate bound} must be dominated by the first. This requires $N^{\frac{\xi \log \gamma}{2}}\lesssim N^{-\frac{\beta}{2\beta+d}}$, which holds whenever $\xi\ge \frac{2\beta}{(2\beta+d)\log (1/\gamma)}$.  The non-asymptotic statistical error bound of sub-optimality attains  $N^{-\frac{\beta}{2\beta+d}}$.
Relative to existing results for standard OPE, this rate is faster than  $N^{-\frac{\beta}{2\beta+2d}}$ \citep{nguyen2022sample} in its dependence on  $d$, and is comparable to $N^{-\frac{\beta}{2\beta+d}}$ \citep{ji2022sample} in its dependence on both $\beta$ and $d$. For a pre-specified error $\epsilon$, the DQPOPE requires a sample complexity of $\mathcal{O}\big((1-\gamma)^{-(3+\frac{3d}{2\beta})}C_{\mu}^{1+\frac{d}{2\beta}}\epsilon^{-(2+\frac{d}{\beta})}\big)$. Compared to $\mathcal{O}\big((1-\gamma)^{-(4+2\frac{d}{\beta})}C_{\mu}^{1+\frac{d}{2\beta}}\epsilon^{-(2+\frac{d}{\beta})}\big)$ in \cite{ji2022sample}, 
the sample complexity of DQPOPE has the same dependence on the distribution shift constant $C_{\mu}$ and the pre-specified error $\epsilon$, while exhibiting a milder dependence on the horizon due to data splitting. 
Compared to  $\mathcal{O}\big((1-\gamma)^{-(2+2\frac{d}{\beta})}\kappa^{1+\frac{d}{\beta}}\epsilon^{-(2+2\frac{d}{\beta})}\big)$ in \cite{nguyen2022sample} where $\kappa$ is the upper bound on distribution ratio $\|d^{\pi}/\mu\|_{\infty}$,  our result has  a weaker dependence on $\epsilon$ and $\kappa$, as $C_{\mu}$ is often substantially smaller than $\kappa$. Our result could achieve comparable sample efficiency as long as  $(\epsilon/(1-\gamma)^{1/2})^{d/\beta}\le 1$. Additionally, our result has a stronger dependence on the horizon when $\beta/d >1/2$. This difference is attributable to the fact that the distributional Bellman operator contracts at a rate $\gamma^{1-\frac{1}{2p}}$ under $\overline{\mathcal{W}}_{p,d^{\pi}}(\cdot,\cdot)$ metric, whereas the standard Bellman operator contracts at a rate $\gamma$ under $\|\cdot\|_{\infty}$ metric.

Estimating the full return distribution is inherently more challenging than its mean, as demonstrated by the fact that $|\mathbb{E}Z_1-\mathbb{E}Z_2|\leq \mathcal{W}_p(\nu_1,\nu_2)$ for $p\ge 1$, where $Z_1\sim\nu_1$ and $Z_2\sim\nu_2$. Despite increased complexity, DQPOPE learns the entire quantile curve, capturing the full distributional information without sacrificing the convergence rate. Notably, our results are the first to show that distributional OPE with ReLU neural network approximation achieves sample efficiency comparable to standard OPE. Compared to model-based DRL in the tabular case \cite{zhang2023estimation, rowland2024near}, our analysis aligns closely with the most practical model-free QDRL algorithms \cite{QRDQN, IQN}.

\subsection{Estimating Policy Value through Quantile Process}\label{sec4.3}

%While DQPOPE is highly effective in recovering comprehensive distributional information, the value function (expectation of the return) remains a critical focus in practice for guiding decision-making, though risk-sensitive policies based on alternative risk measures derived from the distribution can also be explored in certain scenarios.

The preceding analysis characterizes the statistical error of the estimated return distribution. We next show that the learned quantile process also yields a natural estimator of the policy value, %namely the expectation of the return, 
which remains a central target in practice for guiding decision-making.  Recall in Algorithm \ref{alg:DOPE}, the quantile process $\widehat{f}_t(x,\tau)$ is used as a generator
 by sampling $\tau_i\sim\mathrm{Unif}(0,1)$ and plugging it into $\widehat{f}_t(x,\tau)$ to recover samples from the estimated return distribution $\widehat{\eta}_t(x)$ for a given $x$.   Given quantile levels  $\{\tau_k\}_{k=1}^K$, the sample average $\frac{1}{K}\sum_{k=1}^{K} \widehat{f}_t(x,\tau_k)$ provides a natural estimator of the value function at $x$.

\begin{proposition}\label{prop: value estimation bound}
For the policy value estimator $\widehat{V}_K = \frac{1}{K}\sum_{k=1}^{K} \widehat{f}_T(s_{0,k},a_{0,k},\tau_k)$, where $s_{0,k}\sim \rho,a_{0,k}\sim\pi(\cdot|s_{0,k})$ and $\tau_k\sim\mathrm{Unif}(0,1)$, under the same assumptions and network size in Theorem \ref{thm: main results with faster rate}, if $ T=\mathcal{O}(\xi\log N)$ with some constant $\xi>0$ and sufficiently large $N$, the following bound holds with probability at least $1-c\log N\exp(-(N/\log N)^{\frac{2d}{2d+4\beta}})-K^{-1}$,
\begin{align}\label{equ: mean value bound 1}
\big|\widehat{V}_K-V^\pi\big| ~ \le~ \frac{C C_{\mu}^{\frac{1}{2}}}{(1-\gamma)^{\frac{3}{2}}}(\log N)^4 N^{-\frac{\beta}{2\beta+d}} + C F\sqrt{\frac{\log K}{K}}+  \frac{ N^{\frac{\xi\log \gamma}{2} } C_{F,R}}{(1-\gamma)^{\frac{3}{2}}},
\end{align}
where $c, C$ is a constant independent of $C_{\mu},N,K,\gamma$.
% Choosing $K\ge C  N^{\frac{2\beta}{2\beta+d}} (\log N)^{-6}$ yields
% \begin{align*}
% \big|\widehat{V}_K-V^\pi\big| ~ \le~    \frac{C C_{\mu}^{\frac{1}{2}}}{(1-\gamma)^{\frac{3}{2}}}(\log N)^4 N^{-\frac{\beta}{2\beta+d}}  + \frac{N^{\frac{\xi\log \gamma}{2}}~C_{F,R}}{(1-\gamma)^{\frac{3}{2}}}.
% \end{align*}
\end{proposition}

Proposition \ref{prop: value estimation bound} follows from the decomposition
$|\widehat{V}_K-V^\pi|\leq |\widehat{V}_K-\widehat{V}| +|\widehat{V}-V^\pi|$, where $\widehat{V}= \mathbb{E}_{Z\sim \boldsymbol{\widehat{\eta}_T}}[Z]$  is expectation from the output of Algorithm \ref{alg:DOPE} and ${V}^\pi= \mathbb{E}_{Z\sim \boldsymbol{\eta^\pi} }[Z]$ is the target value. The first term 
on the RHS of \eqref{equ: mean value bound 1} represents the finite sample error that can be directly related to the result in Theorem \ref{thm: main results with faster rate}, and the second term reflects the computational error arising from the empirical average over $K$ draws, which scales as $\sqrt{1/K}$.  
Moreover, choosing $K\ge C  N^{\frac{2\beta}{2\beta+d}} (\log N)^{-6}$ and  $\xi\ge \frac{2\beta}{(2\beta+d) \log(1/\gamma)}$ ensures that the last two terms in \eqref{equ: mean value bound 1} are dominated by the first term. Thus, the resulting bound for policy value estimation ultimately attains $N^{-\frac{\beta}{2\beta+d}}$ up to $\log$ factors, matching the optimal rate in the standard OPE setting \citep{fan2020theoretical,ji2022sample}.

A key benefit of estimating policy value through a quantile process is the robustness gained from averaging multiple quantiles compared to directly estimating a single distribution mean.  In particular, quantile loss is less sensitive to
outliers, as it does not disproportionately penalize large deviations, and its gradient scales as $\mathcal{O}(1)$.  By contrast, squared loss places disproportionate weight on extreme observations, which can make optimization more sensitive to heavy tails and atypical
rewards.

\section{Experiments}

This section validates the theoretical analysis of DQPOPE across various scenarios. Section \ref{sec: exp_5.2} presents experiments that validate the sample complexity results in Theorem \ref{thm: main results with faster rate}. Sections \ref{sec: exp_5.1} and \ref{sec: exp_5.4} demonstrate the advantage of DQPOPE in estimating policy values via quantile averaging, leveraging both a simple one-step toy example and a real-world dataset. 
Throughout these experiments, we compare DQPOPE with value-based OPE \footnote{Value-based OPE estimates the value function $Q^{\pi}$ by minimizing squared loss (see Section I for details).} implemented using deep ReLU networks (referred to as DOPE). Additional experimental details are provided in Section J of the Supplementary Material.

\subsection{Sample Complexity Analysis in CartPole}\label{sec: exp_5.2}
To provide a clearer understanding of the sample complexity results in Theorem \ref{thm: main results with faster rate}, we conducted simulation studies in the  CartPole environment. The target policy was obtained by training a DQN agent\citep{DQN} for 10,000 update steps. The target return from the initial state was estimated using 1,000 Monte Carlo (MC) rollouts with discounted cumulative rewards ($\gamma =0.99$) for each rollout, which is visualized in Figure \ref{fig: cartpole} (b).

\begin{figure}[!ht]
    \centering
    \includegraphics[width = 0.67\linewidth]{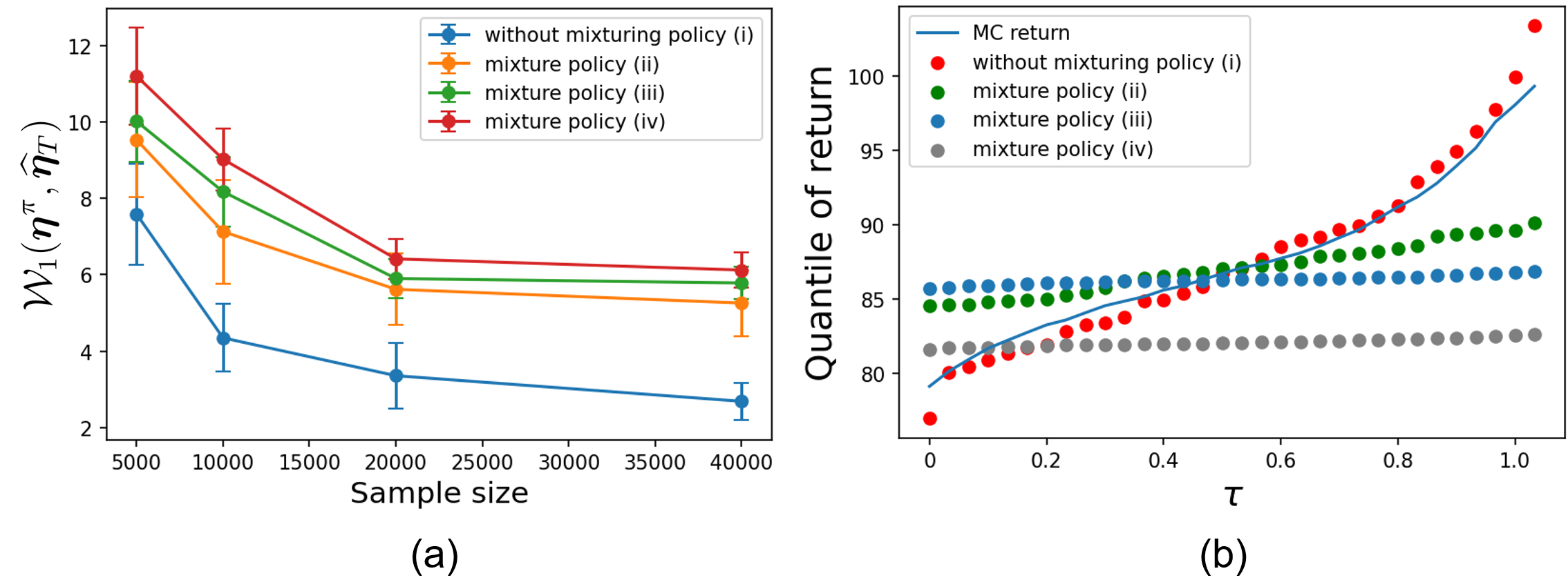}
    \vskip -0.25in
    \caption{ (a) The performance metric versus different sample size. (b) Quantile estimation performance, where the red dots represent the estimated quantiles without mixture policy, other colored dots from the mixture policy (ii)-(iv), and the blue line represents the ground truth quantile function calculated by MC rollouts.}
    \label{fig: cartpole}
\end{figure}

We consider four data-generation schemes:
(i) data collected entirely from the target policy; (ii) data collected from a mixture policy with 80\% target policy actions and 20\% random; (iii) data collected from a mixture policy with 60\% target policy actions; (iv) data collected from a mixture policy with 40\% target policy actions. We use a three-layer fully connected network with 64 units per layer and ReLU activation. The learning rate was set to 0.0005, the batch size to 64, and the target network was updated every 15 steps.

Figure \ref{fig: cartpole} (a) shows the performance metric $\mathcal{W}_1(\boldsymbol{\eta}^{\pi},\boldsymbol{\widehat{\eta}}_{T})$ for DQPOPE across varying sample sizes, under the four data generation schemes. In all cases, the error is consistent with a polynomial decay as the sample size increases. 
As expected, estimation accuracy improves with a larger proportion of data generated from the target policy, reflecting the reduced distribution shift, as quantified by $C_{\mu}$. Figure \ref{fig: cartpole} (b) highlights the quantile estimation performance under the different distribution shifts.  When the data are generated entirely from the target policy, DQPOPE effectively captures the true return distribution.

\subsection{Simulation: A One-step Toy Example}\label{sec: exp_5.1}
In this subsection, we compare DQPOPE and DOPE for policy value estimation within a toy environment (Figure \ref{fig: mse}(a)).  To examine robustness under heavy-tailed rewards, the rewards are generated from Student's t-distributions with different degrees of freedom, thereby
controlling the tail heaviness. For DQPOPE, as described in Section \ref{sec4.3}, the policy values are estimated by averaging over $K$ quantile levels. We evaluate the accuracy of both methods using the Mean Squared Error (MSE).

%In this subsection, we compare the performance of DQPOPE and DOPE in estimating the policy value by a toy environment (Figure \ref{fig: mse} (a)) under varying degrees of heavy-tailed reward distributions. The heavy tails were modeled using Student's t-distributions with different degrees of freedom to represent varying levels of tail heaviness. For DQPOPE, as illucatrated in Section \ref{sec4.3}, $K$ quantile levels are sampled to obtain the sample average estimator for the policy value. Mean Squared Error (MSE) was computed for both algorithms across these settings. 

To ensure a fair comparison, both algorithms use the same network architecture, namely a two-layer fully connected neural network with 12 hidden units per layer and ReLU activation.  In DQPOPE, the quantile level is concatenated with the state and used as the network input. The training parameters are identical across both methods, with a learning rate of 0.002, a batch size of 32, and 100 update iterations (corresponding to a total sample size of 3200). Numerical results are reported in Table \ref{table: MSE}.

\begin{table}[!ht]
\centering
\footnotesize
\caption{Comparison of MSE ($\times 10^{-3}$) of policy value estimation between DOPE and DQPOPE, based on 100 times of replicates with standard deviations.}
\label{table: MSE}
\vskip 0.15in
\begin{tabular}{cccccc}
\hline  & DOPE &\multicolumn{4}{c}{DQPOPE} \\
\hline
   heavy-tailness &       &  \multicolumn{1}{c}{$K=4$} &  \multicolumn{1}{c}{$K=8$}  &  \multicolumn{1}{c}{$K=16$}  &  \multicolumn{1}{c}{$K=32$} \\ 
    \hline
\multirow{1}{*}{t(2)}  & 11.3(17.9)   &     7.21(11.0)     & 3.53(7.42)   &  1.94(3.68) &    0.93(1.13)             \\\hline
\multirow{1}{*}{t(4)}  & 6.51(9.94)    &  4.60(8.97) &  3.71(5.66)    & 2.20(4.64) & 0.59(1.07)  \\ \hline
\multirow{1}{*}{t(6)}   & 6.73(8.14)    &   4.90(5.25)       &    3.19(4.47)       &  2.71(3.55) & 0.78(1.26)   \\ \hline
\multirow{1}{*}{t(8)}  & 4.37(9.15)    &  3.98(7.95)        &      2.91(4.53)         &1.01(1.52)  & 0.82(1.41)   \\ \hline
\multirow{1}{*}{t(10)}   & 8.05(10.4)    &   4.50(9.52)      &      3.82(5.14)         & 1.39(2.25) &  0.61(0.95) \\ \hline
\multirow{1}{*}{$\mathcal{N}(0,1)$}   & 3.75(10.0)    &   4.73(9.71)      &     3.87(6.24)         & 1.41(3.26) &  0.62(1.26) \\ \hline
\end{tabular}
\end{table}

\begin{figure}[!ht]
    \centering
    \includegraphics[width = 0.75\linewidth]{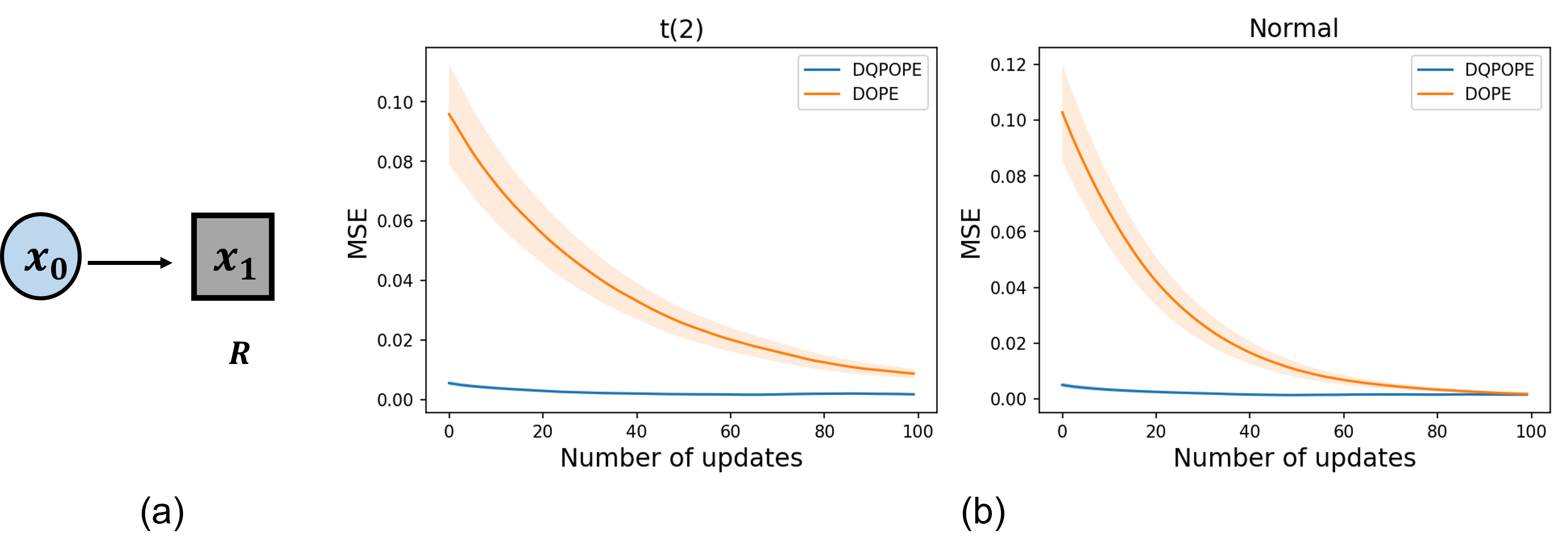}
    \vskip -0.25in
    \caption{ (a) Two state environment with the reward receiving at the terminal state $x_1$. (b) MSE of the policy value estimation under $t(2)$ and $\mathcal{N}(0,1)$ distribution, where each curve is computed based on 100 times replicates and shaded by their confidence intervals.}
    \label{fig: mse}
\end{figure}

Our findings reveal that DQPOPE consistently achieves lower MSE than DOPE, with the advantage becoming more pronounced as the reward distribution becomes heavier tailed. Moreover, increasing the number of quantiles $K$ enhances the accuracy of policy value estimation, aligning with our theoretical analysis. To further illustrate the robustness of quantile averaging estimators,   Figure \ref{fig: mse}(b) visualizes the  MSE against the number of updates under both  the $t(2)$ and $\mathcal{N}(0,1)$ distributions. In both cases, DQPOPE exhibits   significantly faster and more stable convergence.

\subsection{Real Data Analysis: MIMIC-III Dataset}\label{sec: exp_5.4}
We next examine the performance of DQPOPE on the MIMIC-III v1.4 \footnote{\href{https://physionet.org/content/mimiciii/1.4/}{https://physionet.org/content/mimiciii/1.4/}} dataset, which contains critical care records for over 40{,}000 patients admitted to the Beth Israel Deaconess Medical Center between 2001 and 2012 \citep{johnson2016mimic}. Our analysis focuses on \textbf{sepsis}, a high-stakes medical condition that requires sequential treatment decisions under substantial uncertainty. In particular, septic patients often require repeated administration of intravenous fluids and vasopressors to maintain hemodynamic stability. The substantial heterogeneity in patient responses and the complexity of longitudinal clinical trajectories make this setting a natural testbed for distributional OPE.

%Recent studies \citep{raghu2017deep, komorowski2018artificial} have shown that reinforcement learning (RL) can provide more effective and personalized treatment recommendations for sepsis patients compared to human clinicians. This underscores the relevance of distributional RL methods, which are well-suited for capturing the inherent uncertainty in complex, sequential decision-making systems like sepsis management.

Recent studies \citep{raghu2017deep,komorowski2018artificial} suggest that reinforcement learning can provide useful treatment policies for sepsis management, which further motivates the use of distributional RL methods for quantifying uncertainty in such sequential decision problems. In our analysis, patient trajectories are modeled as Markov decision processes (MDPs). The detailed MDPs specification, together with additional implementation details, is provided in Section I.2 of the Supplementary Material.

\begin{figure}[!htb]
\vskip -0.05in
\centering
\includegraphics[width = 0.85\linewidth]{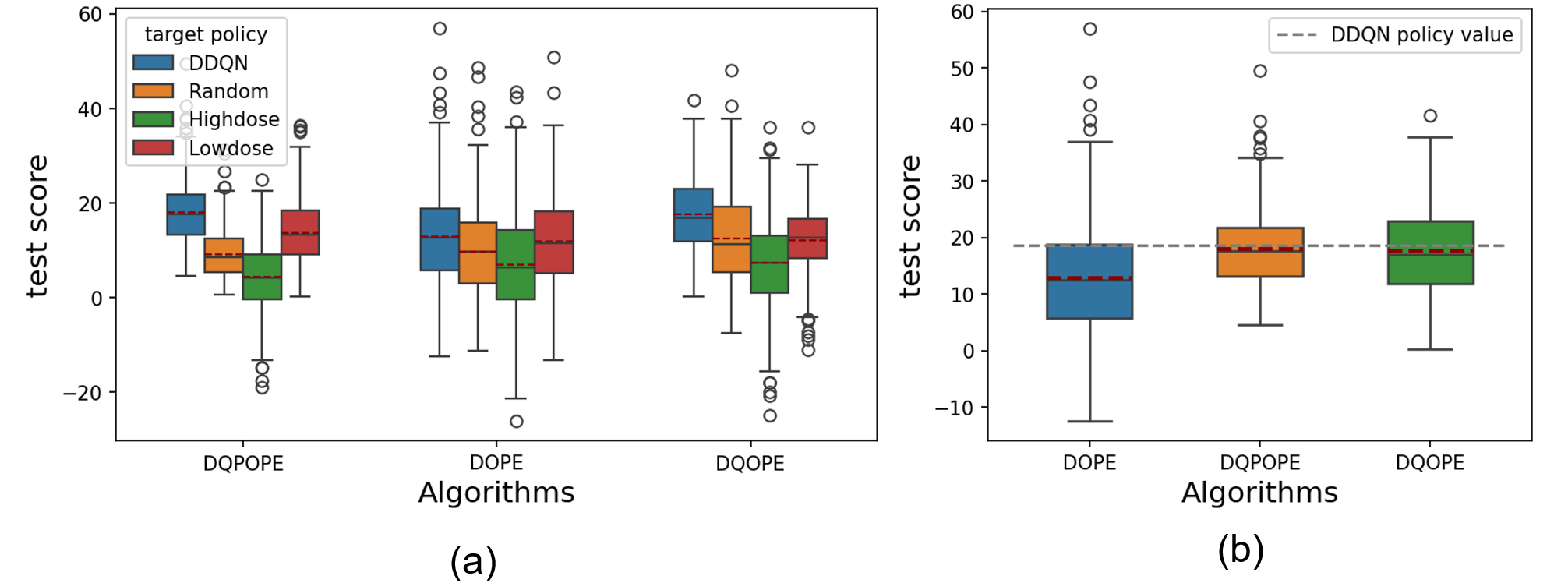}
\vskip -0.25in
\caption{Boxplot of the estimated policy value for different algorithms. In each box, the red dashed line represents the mean, the central solid line indicates the median, and the bottom and top edges of the box correspond to the 25-th and 75-th percentiles, respectively. The bottom and top line outside box correspond to the lower and upper extremes. (a) Comparison of DQPOPE, DOPE and DQOPE across 4 target policies. (b) Comparison of DQPOPE, DOPE and DQOPE under DDQN target policy.}
\label{fig: boxplot}
%\vskip -0.05in
\end{figure}

%Following the analysis of \cite{raghu2017deep,komorowski2018artificial}, we model the patient environment and trajectories as an MDP. For action space, we define a $5\times5$ action (medical intervention) space by discretizing both intravenous (IV) fluid and maximum vasopressor (VP) dosage into 5 levels, and the combination of the two treatments produce 25 possible discrete actions.  The reward function is designed to be clinically guided, reflecting changes in the patient's health condition after a certain treatment has been conducted. Specifically, positive rewards are assigned for improvements in the patient's wellbeing, while negative rewards are given for deteriorations. The commonly used indicators of a patient’s overall health including the SOFA (Sequential Organ Failure Assessment) score, which measures the extent of organ dysfunction with higher scores indicating greater organ dysfunction, and the patient’s lactate levels, a measure of cell-hypoxia that is higher in septic patients. The reward is designed using weighted increments of SOFA score and lactate levels, as outlined in Equation \eqref{equ: mimic_reward}. The state space is composed of 44 physiological features, including the demographics, lab values, vital signs,  fluid balance and intake and output events. The full list of features are provided in Table \ref{table:mimic_state}, and the detailed configuration of the MDP are provided in Supplementary \ref{sec: mimic_appendix_1}.

To implement the OPE experiments, we split the dataset into training (75\%), validation (5\%), and test (15\%) sets, ensuring that each subset maintained the same proportion of surviving and non-surviving patients. We consider four target policies: (i) DDQN: A policy trained using Dueling Double Deep Q-Networks (DDQN), which serves as the benchmark policy\footnote{DDQN is a state-of-the-art reinforcement learning (RL) method that has demonstrated significant success in scaling RL to clinical decision-making problems \citep{lu2020deep}}; (ii) Random: A policy with random dose selection; (iii) High-dose: A policy that always administers high doses of treatment; (iv) Low-dose: A policy that always administers low doses of treatment.

We compare DQPOPE ($K=32$), Deep Quantile-based OPE (DQOPE, $K=32$), and DOPE under these four target policies. Here DQOPE estimates a finite set of quantiles using the QR-DQN approach \citep{QRDQN}, as described in Section \ref{sec: methods}.  All methods are trained on the training set, and model selection is performed using the validation loss. Weighted importance sampling (WIS) is then applied on the test set to estimate the policy value for each target policy.

Figure \ref{fig: boxplot} presents the estimated policy values for four target policies over 500 bootstrapped replications with 90\% resampling on the test set. In Figure \ref{fig: boxplot}(a), the two quantile-based methods, especially DQPOPE, more clearly separate the DDQN policy from the remaining policies, with DDQN consistently receiving the highest test scores. Figure \ref{fig: boxplot}(b) further indicates that DQPOPE yields both more accurate and more stable value estimates than DQOPE, emphasizing the advantages of estimating the entire quantile process rather than discrete quantiles.

To further examine the policies induced by the different methods, Figure \ref{fig: mimic_policy} displays their action-selection distributions. The policy estimated by DQPOPE aligns more closely with the target DDQN policy, with the most frequent treatment selections being (0,0) and (2,1). This pattern is clinically plausible, as many sepsis patients are not critically ill and thus do not require aggressive doses of intravenous fluids or vasopressors. In contrast, the policies estimated by DOPE and DQOPE exhibit larger deviations from the DDQN policy, leading to invalid or suboptimal recommendations.

Overall, these results illustrate the practical advantage of DQPOPE in capturing the uncertainties inherent in complex clinical decision-making scenarios. The results demonstrate that integrating distributional methods into off-policy evaluation (OPE) can significantly improve the performance of policy value estimation. In particular, these results emphasize the strength of estimating a quantile process, which allows for more robust and accurate policy value estimation by fully capturing the behavior of the distributional Bellman operator and better accounting for the inherent randomness in MDPs.

\begin{figure}[!ht]
\vskip -0.1in
\centering
\includegraphics[width = 1.0\linewidth]{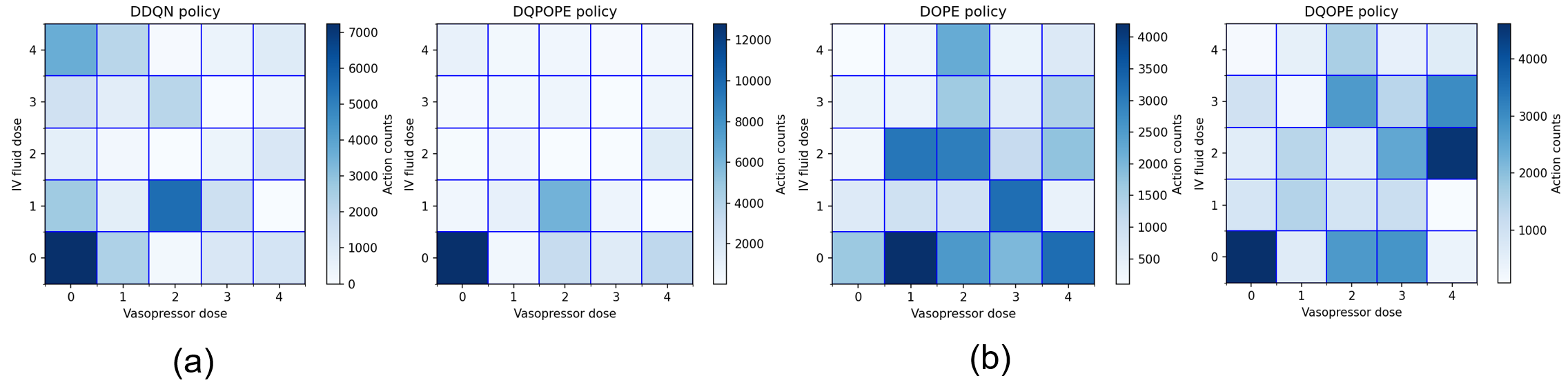}
\vskip -0.25in
\caption{A 2D histogram visualizing the action selection frequency across 4 policies,  with each grid representing the count of actions selected over all time steps in the test set. (a) Target policy of DDQN. (b) Estimated policies by DQPOPE, DOPE, and DQOPE.}
\label{fig: mimic_policy}
%\vskip -0.05in
\end{figure}

\section{Conclusion}
\label{sec:conc}

This paper introduces a deep quantile process regression method for OPE, which incorporates the full return distribution rather than focusing solely on the expected return of a target policy. It provides novel insights into both theoretical and practical aspects. Theoretically, we develop a rigorous statistical framework for distributional OPE and establish sample complexity guarantees for DQPOPE with deep neural network approximation. We further show that estimating the full return distribution can be as sample-efficient as estimating only the mean policy value, thereby clarifying the advantage of DQPOPE over standard value-based OPE methods.
From a practical perspective, we demonstrate how quantile process regression effectively implements the distributional Bellman update and how DQPOPE addresses the pseudo-sample issue encountered in existing quantile-based DRL methods. Extensive experiments further demonstrate the empirical advantage of DQPOPE in OPE tasks.

\section{Acknowledgements}

The work of Qi Kuang is supported by the National Natural Science Foundation of China (Grant 12571286), the Jiangxi Provincial Natural Science Foundation (Grant 20242BAB26002) and the Early-Career Young Scientists and Technologists Project of Jiangxi Province (Grant 20252BEJ730126). The work of Fan Zhou is supported by the Shanghai Research Center for Data Science and Decision Technology, CCF-DiDi
GAIA Collaborative Research Funds, and the "Chenguang Program" supported by Shanghai Education Development and Shanghai Municipal Education Commission. The work of Yuling Jiao is supported in part by the National Key Research and Development Program of China (Grant 2024YFA1014202), the National Natural Science Foundation of China (Grants 12371441 and 12526216),
and the Fundamental Research Funds for the Central Universities.
%%%%%%%%%%%%%%%%%%%%%%%%%%%%%%%%%%%%%%%%%%%%%%%%%%%%%%%%%%%%%%%%%%%%%%%%%%%%%%%%%%%%%%%%%%%%%%%%%%%%%%%%%%%%%%%%%%%%%%%%%%%%%%%%%%%%%%
    %%% SUPPLEMENTARY  %%%
%%%%%%%%%%%%%%%%%%%%%%%%%%%%%%%%%%%%%%%%%%%%%%%%%%%%%%%%%%%%%%%%%%%%%%%%%%%%%%%%%%%%%%%%%%%%%%%%%%%%%%%%%%%%%%%%%%%%%%%%%%%%%%%%%%%%%%    

%\section{BibTeX}

 \spacingset{1.5}

\bibliography{ref}

@article{jin2023bayesian,
  title={A Bayesian decision framework for optimizing sequential combination antiretroviral therapy in people with HIV},
  author={Jin, Wei and Ni, Yang and O’halloran, Jane and Spence, Amanda B and Rubin, Leah H and Xu, Yanxun},
  journal={The Annals of Applied Statistics},
  volume={17},
  number={4},
  pages={3035--3055},
  year={2023},
  publisher={Institute of Mathematical Statistics}
}

@article{boucheron2003concentration,
  title={Concentration inequalities using the entropy method},
  author={Boucheron, St{\'e}phane and Lugosi, G{\'a}bor and Massart, Pascal},
  journal={The Annals of Probability},
  volume={31},
  number={3},
  pages={1583--1614},
  year={2003},
  publisher={Institute of Mathematical Statistics}
}

@article{bartlett2005local,
  title={LOCAL RADEMACHER COMPLEXITIES},
  author={Bartlett, Peter L and Bousquet, Olivier and Mendelson, Shahar},
  journal={The Annals of Statistics},
  volume={33},
  number={4},
  pages={1497--1537},
  year={2005}
}

@book{wainwright2019high,
  title={High-dimensional statistics: A non-asymptotic viewpoint},
  author={Wainwright, Martin J},
  volume={48},
  year={2019},
  publisher={Cambridge University Press}
}

@inproceedings{EDRL,
  title     ={Statistics and samples in distributional reinforcement learning},
  author    ={Rowland, Mark and Dadashi, Robert and Kumar, Saurabh and Munos, R{\'e}mi and Bellemare, Marc G  and Dabney, Will},
  booktitle={International Conference on Machine Learning},
  pages     ={5528--5536},
  year      ={2019},
  organization={PMLR}
}

@inproceedings{MMDRL,
  title={Distributional Reinforcement Learning via Moment Matching},  
  author={Nguyen-Tang, Thanh and Gupta, Sunil and Venkatesh, Svetha}, 
  booktitle={Proceedings of the AAAI Conference on Artificial Intelligence}, 
  volume={35}, 
  number={10}, 
  pages={9144--9152},
 year={2021}
}

@inproceedings{QRDQN,
  title={Distributional Reinforcement Learning With Quantile Regression},
  author={Dabney, Will and Rowland, Mark and Bellemare, Marc G and Munos, R{\'e}mi},
  booktitle={Proceedings of the AAAI Conference on Artificial Intelligence},
  year={2018}
}

@inproceedings{IQN,
  title={Implicit Quantile Networks for Distributional Reinforcement Learning},
  author={Dabney, Will and Ostrovski, Georg and Silver, David and Munos, Remi},
  booktitle={International Conference on Machine Learning},
  pages={1096--1105},
  year={2018}
}

@article{DQN,
  title={Human-level control through deep reinforcement learning},
  author={Mnih, Volodymyr and Kavukcuoglu, Koray and Silver, David and Rusu, Andrei A and Veness, Joel and Bellemare, Marc G and Graves, Alex and Riedmiller, Martin and Fidjeland, Andreas K and Ostrovski, Georg and others},
  journal={Nature},
  volume={518},
  number={7540},
  pages={529--533},
  year={2015},
  publisher={Nature Publishing Group}
}

@inproceedings{C51,
  title={A Distributional Perspective on Reinforcement Learning},
  author={Bellemare, Marc G and Dabney, Will and Munos, R{\'e}mi},
  booktitle={International Conference on Machine Learning},
  pages={449--458},
  year={2017},
  organization={PMLR}
}

@inproceedings{rowland2018analysis,
  title={An analysis of categorical distributional reinforcement learning},
  author={Rowland, Mark and Bellemare, Marc and Dabney, Will and Munos, R{\'e}mi and Teh, Yee Whye},
  booktitle={International Conference on Artificial Intelligence and Statistics},
  pages={29--37},
  year={2018},
  organization={PMLR}
}

@inproceedings{sac,
  title={Soft actor-critic: Off-policy maximum entropy deep reinforcement learning with a stochastic actor},
  author={Haarnoja, Tuomas and Zhou, Aurick and Abbeel, Pieter and Levine, Sergey},
  booktitle={International Conference on Machine Learning},
  pages={1861--1870},
  year={2018},
  organization={PMLR}
}

@inproceedings{TD3,
  title={Addressing function approximation error in actor-critic methods},
  author={Fujimoto, Scott and Hoof, Herke and Meger, David},
  booktitle={International Conference on Machine Learning},
  pages={1587--1596},
  year={2018},
  organization={PMLR}
}

@book{drl_book,
    title={Distributional Reinforcement Learning},
    author={Marc G. Bellemare and Will Dabney and Mark Rowland},
    publisher={MIT Press},
    note={\url{http://www.distributional-rl.org}},
    year={2023}
}

@article{rowland2023analysis,
  title={An Analysis of Quantile Temporal-Difference Learning},
  author={Rowland, Mark and Munos, R{\'e}mi and Azar, Mohammad Gheshlaghi and Tang, Yunhao and Ostrovski, Georg and Harutyunyan, Anna and Tuyls, Karl and Bellemare, Marc G and Dabney, Will},
  journal={Journal of Machine Learning Research},
   year={2024}
}

@inproceedings{mle-drl,
  title={Distributional Offline Policy Evaluation with Predictive Error Guarantees},
  author={Wu,Runzhe and Uehara,Masatoshi and Sun,Wen },
  booktitle={International Conference on Machine Learning},
  year={2023},
  organization={PMLR}
}

@article{antos2008learning,
  title={Learning near-optimal policies with Bellman-residual minimization based fitted policy iteration and a single sample path},
  author={Antos, Andr{\'a}s and Szepesv{\'a}ri, Csaba and Munos, R{\'e}mi},
  journal={Machine Learning},
  volume={71},
  pages={89--129},
  year={2008},
  publisher={Springer}
}

@article{antos2007fitted,
  title={Fitted Q-iteration in continuous action-space MDPs},
  author={Antos, Andr{\'a}s and Szepesv{\'a}ri, Csaba and Munos, R{\'e}mi},
  journal={Advances in neural information processing systems},
  volume={20},
  year={2007}
}

@article{munos2008finite,
  title={Finite-Time Bounds for Fitted Value Iteration.},
  author={Munos, R{\'e}mi and Szepesv{\'a}ri, Csaba},
  journal={Journal of Machine Learning Research},
  volume={9},
  number={5},
  year={2008}
}

@inproceedings{chen2019information,
  title={Information-theoretic considerations in batch reinforcement learning},
  author={Chen, Jinglin and Jiang, Nan},
  booktitle={International Conference on Machine Learning},
  pages={1042--1051},
  year={2019},
  organization={PMLR}
}

@article{shen2024,
  title={Nonparametric Estimation of Non-Crossing Quantile Regression Process with Deep ReQU Neural Networks},
  author={Shen, Guohao and Jiao, Yuling and Lin, Yuanyuan and Horowitz, Joel L and Huang, Jian},
  journal={Journal of Machine Learning Research},
  volume={25},
  number={88},
  pages={1--75},
  year={2024}
}

@article{belloni2011,
  title={$l_1$-penalized quantile regression in high dimensional sparse models.},
  author={Belloni, Alexandre and Chernozhukov, Victor},
  journal={The Annals of Statistics},
  pages={82--130},
  year={2011},
  publisher={JSTOR}
}

@article{madrid2022risk,
  title={Risk bounds for quantile trend filtering},
  author={Madrid Padilla, Oscar Hernan and Chatterjee, Sabyasachi},
  journal={Biometrika},
  volume={109},
  number={3},
  pages={751--768},
  year={2022},
  publisher={Oxford University Press}
}

@article{farrell2021deep,
  title={Deep neural networks for estimation and inference},
  author={Farrell, Max H and Liang, Tengyuan and Misra, Sanjog},
  journal={Econometrica},
  volume={89},
  number={1},
  pages={181--213},
  year={2021},
  publisher={Wiley Online Library}
}

@book{anthony1999neural,
  title={Neural Network Learning: Theoretical Foundations},
  author={Anthony, Martin and Bartlett, Peter L},
  year={2009},
  publisher={Cambridge University Press}
}

@article{Bartlett2019,
  author  = {Peter L. Bartlett and Nick Harvey and Christopher Liaw and Abbas Mehrabian},
  title   = {Nearly-tight VC-dimension and Pseudodimension Bounds for Piecewise Linear Neural Networks},
  journal = {Journal of Machine Learning Research},
  year    = {2019},
  volume  = {20},
  number  = {63},
  pages   = {1--17},
}

@article{nguyen2022sample,
  title={On Sample Complexity of Offline Reinforcement Learning with Deep ReLU Networks in Besov Spaces},
  author={Nguyen-Tang, Thanh and Gupta, Sunil and Venkatesh, Svetha and others},
  journal={Transactions on Machine Learning Research},
  year={2022}
}

@inproceedings{fan2020theoretical,
  title={A theoretical analysis of deep Q-learning},
  author={Fan, Jianqing and Wang, Zhaoran and Xie, Yuchen and Yang, Zhuoran},
  booktitle={Learning for dynamics and control},
  pages={486--489},
  year={2020},
  organization={PMLR}
}

@article{padilla2022quantile,
  title={Quantile regression with ReLU Networks: Estimators and minimax rates},
  author={Padilla, Oscar Hernan Madrid and Tansey, Wesley and Chen, Yanzhen},
  journal={Journal of Machine Learning Research},
  volume={23},
  number={1},
  pages={11251--11292},
  year={2022},
  publisher={JMLRORG}
}

@article{jiao2023deep,
  title={Deep nonparametric regression on approximately low-dimensional manifolds},
  author={Jiao, Yuling and Shen, Guohao and Lin, Yuanyuan and Huang, Jian},
  journal={The Annals of Statistics},
  volume={51},
  number={2},
  pages={691–716},
  year={2023},
}

@article{Volgushev,
author = {Stanislav Volgushev and Shih-Kang Chao and Guang Cheng},
title = {{Distributed inference for quantile regression processes}},
volume = {47},
journal = {The Annals of Statistics},
number = {3},
publisher = {Institute of Mathematical Statistics},
pages = {1634 -- 1662},
year = {2019},
}

@INPROCEEDINGS{GMM-DRL,
  author={Choi, Yunho and Lee, Kyungjae and Oh, Songhwai},
  booktitle={International Conference on Robotics and Automation}, 
  title={Distributional Deep Reinforcement Learning with a Mixture of Gaussians}, 
  year={2019},
  volume={},
  number={},
  pages={9791-9797},
 }

@article{ji2022sample,
  title={Sample complexity of nonparametric off-policy evaluation on low-dimensional manifolds using deep networks},
  author={Ji, Xiang and Chen, Minshuo and Wang, Mengdi and Zhao, Tuo},
  journal={International Conference on Learning Representations},
  year={2023}
}

@article{zhang2023estimation,
  author = {Liangyu Zhang and Yang Peng and Jiadong Liang and Wenhao Yang and Zhihua Zhang},
title = {{Estimation and inference in distributional reinforcement learning}},
volume = {53},
journal = {The Annals of Statistics},
number = {5},
publisher = {Institute of Mathematical Statistics},
pages = {1987 -- 2011},
year = {2025},
}

@article{rowland2023icml,
  title={The Statistical Benefits of Quantile Temporal-Difference Learning for Value Estimation},
  author={Rowland, Mark and Tang, Yunhao and Lyle, Clare and Munos, R{\'e}mi and Bellemare, Marc G and Dabney, Will},
  journal={International Conference on Machine Learning},
  year={2023},
  organization={PMLR}
}

@article{wang2023benefits,
  title={The Benefits of Being Distributional: Small-Loss Bounds for Reinforcement Learning},
  author={Wang, Kaiwen and Zhou, Kevin and Wu, Runzhe and Kallus, Nathan and Sun, Wen},
  journal={International Conference on Machine Learning},
  year={2023}
}

@article{Bodnar2019QuantileQF,
  title={Quantile QT-Opt for Risk-Aware Vision-Based Robotic Grasping},
  author={Cristian Bodnar and Adrian Li and Karol Hausman and Peter Pastor and Mrinal Kalakrishnan},
  journal={ In Robotics: Science and Systems},
  year={2020},
}

@article{Bellemare2020AutonomousNO,
  title={Autonomous navigation of stratospheric balloons using reinforcement learning},
  author={Marc G. Bellemare and Salvatore Candido and Pablo Samuel Castro and Jun Gong and Marlos C. Machado and Subhodeep Moitra and Sameera S. Ponda and Ziyun Wang},
  journal={Nature},
  year={2020},
  volume={588},
  pages={77 - 82},
}

@article{stone1982optimal,
  title={Optimal global rates of convergence for nonparametric regression},
  author={Stone, Charles J},
  journal={The Annals of Statistics},
  pages={1040--1053},
  year={1982},
  publisher={JSTOR}
}

@article{SchmidtHieber2020,
  title={Nonparametric regression using deep neural networks with ReLU activation function},
  author={Johannes Schmidt-Hieber},
  journal={The Annals of Statistics},
  volume={4},
  number={48},
  pages={1875–-1897},
  year={2020},
}

@article{ddpg,
  title={Continuous control with deep reinforcement learning},
  author={Lillicrap, Timothy P and Hunt, Jonathan J and Pritzel, Alexander and Heess, Nicolas and Erez, Tom and Tassa, Yuval and Silver, David and Wierstra, Daan},
  journal={arXiv preprint arXiv:1509.02971},
  year={2015}
}

@InProceedings{schaul15,
  title = 	 {Universal Value Function Approximators},
  author = 	 {Schaul, Tom and Horgan, Daniel and Gregor, Karol and Silver, David},
  booktitle = 	 {Proceedings of the 32nd International Conference on Machine Learning},
  pages = 	 {1312--1320},
  year = 	 {2015},
  volume = 	 {37}
}

@article{yu1994rates,
  title={Rates of convergence for empirical processes of stationary mixing sequences},
  author={Yu, Bin},
  journal={The Annals of Probability},
  pages={94--116},
  year={1994},
  publisher={JSTOR}
}

@article{liang2009strong,
  title={Strong convergence in nonparametric regression with truncated dependent data},
  author={Liang, Han-Ying and Li, Deli and Qi, Yongcheng},
  journal={Journal of multivariate analysis},
  volume={100},
  number={1},
  pages={162--174},
  year={2009},
  publisher={Elsevier}
}

@article{hang2017bernstein,
  title={A BERNSTEIN-TYPE INEQUALITY FOR SOME MIXING PROCESSES AND DYNAMICAL SYSTEMS WITH AN APPLICATION TO LEARNING},
  author={Hang,Hanyuan  and Steinwart, Ingo},
  journal={The Annals of Statistics},
  volume={45},
  number={2},
  pages={708--743},
  year={2017}
}

@article{lazaric2012finite,
  title={Finite-sample analysis of least-squares policy iteration},
  author={Lazaric, Alessandro and Ghavamzadeh, Mohammad and Munos, R{\'e}mi},
  journal={Journal of Machine Learning Research},
  volume={13},
  pages={3041--3074},
  year={2012}
}

@article{chen2006estimation,
  title={Estimation of copula-based semiparametric time series models},
  author={Chen, Xiaohong and Fan, Yanqin},
  journal={Journal of Econometrics},
  volume={130},
  number={2},
  pages={307--335},
  year={2006},
  publisher={Elsevier}
}

@article{wong2020lasso,
  title={Lasso guarantees for $\beta$-mixing heavy-tailed time series},
  author={Wong, Kam Chung and Li, Zifan and Tewari, Ambuj},
  journal={The Annals of Statistics},
  volume={48},
  number={2},
  pages={1124--1142},
  year={2020},
  publisher={JSTOR}
}

@article{rowland2024near,
  title={Near-Minimax-Optimal Distributional Reinforcement Learning with a Generative Model},
  author={Rowland, Mark and Wenliang, Li Kevin and Munos, R{\'e}mi and Lyle, Clare and Tang, Yunhao and Dabney, Will},
  journal={Advances in Neural Information Processing Systems},
  year={2024}
}

@article{peng2024statistical,
  title={Statistical Efficiency of Distributional Temporal Difference},
  author={Peng, Yang and Zhang, Liangyu and Zhang, Zhihua},
  journal={Advances in Neural Information Processing Systems},
  year={2024}
}

@article{johnson2016mimic,
  title={MIMIC-III, a freely accessible critical care database},
  author={Johnson, Alistair EW and Pollard, Tom J and Shen, Lu and Lehman, Li-wei H and Feng, Mengling and Ghassemi, Mohammad and Moody, Benjamin and Szolovits, Peter and Anthony Celi, Leo and Mark, Roger G},
  journal={Scientific data},
  volume={3},
  number={1},
  pages={1--9},
  year={2016},
  publisher={Nature Publishing Group}
}

@article{komorowski2018artificial,
  title={The artificial intelligence clinician learns optimal treatment strategies for sepsis in intensive care},
  author={Komorowski, Matthieu and Celi, Leo A and Badawi, Omar and Gordon, Anthony C and Faisal, A Aldo},
  journal={Nature medicine},
  volume={24},
  number={11},
  pages={1716--1720},
  year={2018},
  publisher={Nature Publishing Group US New York}
}

@article{raghu2017deep,
  title={Deep reinforcement learning for sepsis treatment},
  author={Raghu, Aniruddh and Komorowski, Matthieu and Ahmed, Imran and Celi, Leo and Szolovits, Peter and Ghassemi, Marzyeh},
  journal={Workshop on Machine Learning For Health at the conference on Neural Information Processing Systems},
  year={2017}
}

@article{AIS2019,
  author  = {Jayakumar Subramanian and Amit Sinha and Raihan Seraj and Aditya Mahajan},
  title   = {Approximate Information State for Approximate Planning and Reinforcement Learning in Partially Observed Systems},
  journal = {Journal of Machine Learning Research},
  year    = {2022},
  volume  = {23},
  number  = {12},
  pages   = {1--83},
  url     = {http://jmlr.org/papers/v23/20-1165.html}
}

@InProceedings{jiang16_ope,
  title = 	 {Doubly Robust Off-policy Value Evaluation for Reinforcement Learning},
  author = 	 {Jiang, Nan and Li, Lihong},
  booktitle = 	 {Proceedings of The 33rd International Conference on Machine Learning},
  pages = 	 {652--661},
  year = 	 {2016},
  volume = 	 {48},
  month = 	 {20--22 Jun},
  publisher =    {PMLR},
}

@article{precup2000eligibility,
  title={Eligibility traces for off-policy policy evaluation},
  author={Precup, Doina},
  journal={Computer Science Department Faculty Publication Series},
  pages={80},
  year={2000}
}

@inproceedings{lu2020deep,
  title={Is deep reinforcement learning ready for practical applications in healthcare? A sensitivity analysis of duel-DDQN for hemodynamic management in sepsis patients},
  author={Lu, MingYu and Shahn, Zachary and Sow, Daby and Doshi-Velez, Finale and Li-wei, H Lehman},
  booktitle={AMIA Annual Symposium Proceedings},
  volume={2020},
  pages={773},
  year={2020},
  organization={American Medical Informatics Association}
}

@book{panaretos2020invitation,
  title={An invitation to statistics in Wasserstein space},
  author={Panaretos, Victor M and Zemel, Yoav},
  year={2020},
  publisher={Springer Nature}
}

@article{wang2012evaluation,
  title={Evaluation of viable dynamic treatment regimes in a sequentially randomized trial of advanced prostate cancer},
  author={Wang, Lu and Rotnitzky, Andrea and Lin, Xihong and Millikan, Randall E and Thall, Peter F},
  journal={Journal of the American Statistical Association},
  volume={107},
  number={498},
  pages={493--508},
  year={2012},
  publisher={Taylor \& Francis}
}

@article{zhu2019proper,
  title={Proper inference for value function in high-dimensional Q-learning for dynamic treatment regimes},
  author={Zhu, Wensheng and Zeng, Donglin and Song, Rui},
  journal={Journal of the American Statistical Association},
  volume={114},
  number={527},
  pages={1404--1417},
  year={2019},
  publisher={Taylor \& Francis}
}

@article{dabney2020distributional,
  title={A distributional code for value in dopamine-based reinforcement learning},
  author={Dabney, Will and Kurth-Nelson, Zeb and Uchida, Naoshige and Starkweather, Clara Kwon and Hassabis, Demis and Munos, R{\'e}mi and Botvinick, Matthew},
  journal={Nature},
  volume={577},
  number={7792},
  pages={671--675},
  year={2020},
  publisher={Nature Publishing Group}
}

@INPROCEEDINGS{zhouicdm,
  author={Zhou, Fan and Lu, Chenfan and Tang, Xiaocheng and Zhang, Fan and Qin, Zhiwei and Ye, Jieping and Zhu, Hongtu},
  booktitle={2021 IEEE International Conference on Data Mining (ICDM)}, 
  title={Multi-Objective Distributional Reinforcement Learning for Large-Scale Order Dispatching}, 
  year={2021},
  volume={},
  number={},
  pages={1541-1546},
 }

@book{zhuhongtu_book,
    title={Reinforcement Learning in the Ridesharing Marketplace},
    author={Qin, Zhiwei (Tony) and Tang, Xiaocheng and Li, Qingyang and Zhu, Hongtu and Jieping Ye},
    publisher={Springer Nature},
    year={2025}
}

@article{liao2021off,
  title={Off-policy estimation of long-term average outcomes with applications to mobile health},
  author={Liao, Peng and Klasnja, Predrag and Murphy, Susan},
  journal={Journal of the American Statistical Association},
  volume={116},
  number={533},
  pages={382--391},
  year={2021},
  publisher={Taylor \& Francis}
}

@article{Dualdice,
  title={Dualdice: Behavior-agnostic estimation of discounted stationary distribution corrections},
  author={Nachum, Ofir and Chow, Yinlam and Dai, Bo and Li, Lihong},
  journal={Advances in neural information processing systems},
  volume={32},
  year={2019}
}

@article{muller2024distributional,
  title={Distributional reinforcement learning in prefrontal cortex},
  author={Muller, Timothy H and Butler, James L and Veselic, Sebastijan and Miranda, Bruno and Wallis, Joni D and Dayan, Peter and Behrens, Timothy EJ and Kurth-Nelson, Zeb and Kennerley, Steven W},
  journal={Nature Neuroscience},
  volume={27},
  number={3},
  pages={403--408},
  year={2024},
  publisher={Nature Publishing Group US New York}
}

@article{lowet2025opponent,
  title={An opponent striatal circuit for distributional reinforcement learning},
  author={Lowet, Adam S and Zheng, Qiao and Meng, Melissa and Matias, Sara and Drugowitsch, Jan and Uchida, Naoshige},
  journal={Nature},
  pages={1--10},
  year={2025},
  publisher={Nature Publishing Group UK London}
}

@article{abdullah2019wasserstein,
  title={Wasserstein robust reinforcement learning},
  author={Abdullah, Mohammed Amin and Ren, Hang and Ammar, Haitham Bou and Milenkovic, Vladimir and Luo, Rui and Zhang, Mingtian and Wang, Jun},
  journal={arXiv preprint arXiv:1907.13196},
  year={2019}
}

@article{gerstenberg2024policy,
  title={On Policy Evaluation Algorithms in Distributional Reinforcement Learning},
  author={Gerstenberg, Julian and Neininger, Ralph and Spiegel, Denis},
  journal={arXiv preprint arXiv:2407.14175},
  year={2024}
}

@article{stable-baselines3,
  author  = {Antonin Raffin and Ashley Hill and Adam Gleave and Anssi Kanervisto and Maximilian Ernestus and Noah Dormann},
  title   = {Stable-Baselines3: Reliable Reinforcement Learning Implementations},
  journal = {Journal of Machine Learning Research},
  year    = {2021},
  volume  = {22},
  number  = {268},
  pages   = {1-8},
  url     = {http://jmlr.org/papers/v22/20-1364.html}
}

@article{ho2020denoising,
  title={Denoising diffusion probabilistic models},
  author={Ho, Jonathan and Jain, Ajay and Abbeel, Pieter},
  journal={Advances in neural information processing systems},
  volume={33},
  pages={6840--6851},
  year={2020}
}

@inproceedings{thomas2016data,
  title={Data-efficient off-policy policy evaluation for reinforcement learning},
  author={Thomas, Philip and Brunskill, Emma},
  booktitle={International conference on machine learning},
  pages={2139--2148},
  year={2016},
  organization={PMLR}
}

\bibliographystyle{apalike}
%\bibliographystyle{plainnat}

%%%%%%%%%%%%%%%%%%%%%%%%%%%%%%%%%%%%%%%%%%%%%%%%%%%%%%%%%%%%%%%%%%%%%%%%%%%%%%%%%%%%%%%%%%%%%%%%%%%%%%%%%%%%%%%%%%%%%%%%%%%%%%%%%%%%%%
    %%% SUPPLEMENTARY  %%%
%%%%%%%%%%%%%%%%%%%%%%%%%%%%%%%%%%%%%%%%%%%%%%%%%%%%%%%%%%%%%%%%%%%%%%%%%%%%%%%%%%%%%%%%%%%%%%%%%%%%%%%%%%%%%%%%%%%%%%%%%%%%%%%%%%%%%%    

\begin{center}
{\large\bf SUPPLEMENTARY MATERIAL}
\end{center}
\renewcommand{\thesection}{\Alph{section}}
\setcounter{section}{0}

\begin{description}

\item[Title:] This is the supplementary material of the manuscript, titled as “Distributional Off-Policy Evaluation with
Deep Quantile Process Regression”.

%\item[R-package for  MYNEW routine:] R-package ÒMYNEWÓ containing code to perform the diagnostic methods described in the article. The package also contains all datasets used as examples in the article. (GNU zipped tar file)

%\item[HIV data set:] Data set used in the illustration of MYNEW method in Section~ 3.2. (.txt file)

\end{description}

%\newpage
\spacingset{1.9} % DON'T change the spacing!
We begin by introducing the key notations and techniques used throughout the analysis. The supplemental material is organized as follows:
\begin{itemize}
    %\item \textbf{Key Notations}: 
    \item \textbf{Section \ref{sec: related work}}: Additional related work is discussed to provide further context and background for this study.
    \item \textbf{Sections \ref{sec: proof of  contraction} to \ref{sec: proof of proposition 4.13}}: Detailed proofs and discussions of the theoretical results presented in the main text are provided.
    
      \item Section \ref{sec: further discussions on assumptions}: Further discussion on the sub-Gaussian reward condition and Assumptions 4.4 and 4.5.
     \item  Section \ref{sec: beta-mixing for dependent sequence}: Extend the discussion beyond the i.i.d. setting. 
     \item \textbf{Section \ref{sec: Supporting lemmas}}: Technical lemmas  essential for the proofs are established.
    \item Section \ref{sec: expriments}: Additional experimental results to further substantiate the findings reported in the main text.
\end{itemize}

\paragraph{Some additional key notations.} Throughout the supplemental material, we sometimes omit the reference distribution  $\nu$ in the notation for the $L_p(\nu)$-norm whenever the context is clear, denoting $\|\cdot\|_{p,\nu}$ as $\|\cdot\|_{p}$ for simplicity. We use $\|f\|_n^2:=\frac{1}{n}\sum_{i=1}^nf^2(x_i,\tau_i)$ to denote the squared empirical ${L}_2(P_n)$-norm associated with the samples $\{(x_i,\tau_i)\}_{i=1}^n$. We use $\EE$ to represent the expectation over all random variables involved, unless explicitly indicated otherwise by subscripts. 
For any set $\cH$, equipped with some metric $\|\cdot\|$, 
we denote the covering number of $\cH$ 
as $N( \delta,  \cH, \|\cdot\|)$. Unless otherwise stated, we use $C$ to denote some universal constant whose values may vary line by line.

\section{Related Work}\label{sec: related work}

%\subsection{Off-policy Evaluation}
%Off-policy evaluation, a fundamental problem in RL, aims to accurately estimate the value function of a target policy, utilizing observational data generated under a different behavior policy. OPE encompasses both the non-dynamic setting (also referred to as contextual bandits) \citep{dudik2011doubly} and the sequential decision-making setting \citep{sutton1998reinforcement,uehara2022review,levine2020offline}.   OPE is crucial in applications where experimentation is expensive, risky, or unethical. Particularly in the area of healthcare, the objective of OPE requires to precisely evaluate the expected outcomes (value functions) of (dynamic) treatment policies using historical data from observational electronic health records.  There has been a growing interest in evaluating the mean outcome of patients under a fixed decision policy \citep{wang2012evaluation,luedtke2016statistical,zhu2019proper}, with recent studies \citep{wang2018quantile, kallus2024localized} also considering the evaluation of quantile outcomes. However, aforementioned works and most current works focus on the non-dynamic setting, while our work investigates the more complex sequential decision-making setting. Moreover, there has been no exploration of distributional OPE methods in the context of sequential decision-making for treatment regimes. 

\subsection{Types of Distributional RL Algorithms}
Existing distributional RL algorithms can be broadly categorized into category-based methods \citep{C51}, quantile-based methods \citep{QRDQN, IQN}, and particle-based methods \citep{MMDRL}. Among these, quantile-based methods, such as QR-DQN \citep{QRDQN} and IQN \citep{IQN}, have gained significant popularity due to their strong empirical performance.

QR-DQN estimates discrete quantiles at fixed quantile levels, while IQN extends this approach by learning the continuous quantile function, embedding the quantile level directly into the neural network input. However, IQN's primary focus is on enhancing the representational capacity of the neural network for modeling continuous quantile functions, inspired by implicit representation techniques \citep{schaul15}, rather than explicitly addressing a distribution learning problem through quantile process regression.

In contrast, our DQPOPE framework rigorously formulates distribution estimation using quantile process regression, providing a deeper theoretical foundation for distributional OPE (and RL) while advancing the understanding of distributional learning methods.

%Existing distributional RL algorithms include category-based methods \citep{C51}, quantile-based methods \citep{QRDQN, IQN}, and particle-based methods \citep{MMDRL}. Among these, quantile-based methods such as QR-DQN \cite{QRDQN} and IQN \cite{IQN}, are particularly popular due to their empirical achievements.  QR-DQN estimates discrete quantiles at fixed quantile levels, while IQN extends it to learn the continuous quantile function by embedding the quantile level into the neural network input. However, the primary goal of IQN is to enhance the representational capacity of the neural network for modeling continuous quantile functions from a practical perspective, inspired by implicit representation \citep{schaul15}, rather than addressing a distribution learning problem through quantile process regression. In contrast, our DQPOPE provides a rigorous formulation for distribution estimation by using quantile process regression, offering a deeper theoretical understanding for distributional OPE(RL).  

%Besides, IQN dynamically samples quantile levels from $\mathrm{Unif}(0,1)$ per step, thus QPOPE is conceptually close to IQN. However, the primary goal of IQN is to enhance the representational capacity of the neural network for modeling continuous quantile functions from a practical perspective, rather than addressing a distribution learning problem through quantile process regression. In contrast, our QPOPE provides a rigorous formulation for distribution estimation using quantile process regression thus making a deeper theoretical understanding of distributional OPE(RL).  

\subsection{Theoretical Considerations  of Distributional RL} \label{sec: theory of distributional RL}
Recent theoretical advancements in distributional RL have sought to uncover its advantages over standard RL approaches, focusing on convergence properties and sample complexity.

\citet{rowland2018analysis,rowland2023analysis} provides convergence guarantees for categorical temporal-difference (CTD) learning and quantile temporal-difference (QTD) learning in the tabular setting. However, their analyses do not extend to finite-sample performance in non-tabular settings. Building on this foundation, \citet{rowland2023icml} reveal a surprising result: QTD learning can estimate the value function more accurately in stochastic environments, provided the number of quantiles is sufficiently large.

Another line of research explores the sample complexity of distributional RL. \citet{zhang2023estimation,rowland2024near} demonstrate that category-based distributional RL achieves comparable sample complexity to standard RL in model-based settings, while \citet{peng2024statistical} report similar findings for CTD learning. However, these works are limited to the tabular case and rely on a generative model to approximate the transition probability, leaving the sample complexity of widely-used model-free algorithms, such as QR-DQN, largely unexplored.

For a comprehensive overview of the theory underlying distributional RL, we refer readers to the recent book by \citet{drl_book}.

%Recent theoretical investigations in distributional RL aim to explore its advantages over standard approaches, focusing on analyzing convergence properties and convergence rates. \citet{rowland2018analysis,rowland2023analysis} establish convergence guarantees for categorical temporal-difference (CTD) learning and quantile temporal-difference (QTD) learning in tabular case but do not delve into the finite sample performance beyond the tabular case. \citet{rowland2023icml} unveils a surprising conclusion that QTD learning can more accurately estimate the value function in stochastic  environments, provided the number of quantiles is sufficiently large. 

%Another line of recent studies focus on the sample complexity of distributional RL. \citet{zhang2023estimation,rowland2024near} demonstrate category-based distributional RL achieves similar sample complexity results compared to standard RL in a model-based setting, and \citet{peng2024statistical} shows a similar conclusion for CTD learning. However, their work focus on tabular case and assume a generative model to substitute the transition probability,  without delving into the understanding of popular model-free RL algorithms like QR-DQN. For a comprehensive understanding of the theory of distributional RL, one could refer to the recent book \citep{drl_book}. 

%\subsection{Advantages of non-parametric quantile regression over MLE}
\subsection{MLE-based Distributional RL }
%Recent studies analyze distributional RL from a parametric perspective using MLE, which assumes a specific parametric form for the return distribution. Algorithms proposed by  \citet{mle-drl,wang2023benefits} for distributional RL using MLE offer theoretical insights but lack empirical guidance for specific parametric representations, resulting in a theory-practice gap. Moreover, their results rely on the assumption that MLE can achieve good generalization. The category-based method \citep{C51} parameterizes the return distribution as a categorical distribution over a fixed set of supports,  but this approach requires domain knowledge about the return support in implementation. \citet{GMM-DRL} assume a Gaussian mixture representation for return distribution, but the representation power and computational efficiency would be constrained when dealing with the multi-modal distributions in complex scenarios. In contrast, non-parametric quantile regression,  particularly with quantile process regression, addresses these challenges by providing a flexible, easy-to-implement framework for modeling return distributions, resulting in efficient training and better generalization.
Recent studies have analyzed distributional RL from a parametric perspective using Maximum Likelihood Estimation (MLE), which assumes a specific parametric form for the return distribution. Algorithms proposed by \citet{mle-drl, wang2023benefits} for MLE-based distributional RL provide valuable theoretical insights but lack practical guidance on selecting appropriate parametric representations, resulting in a gap between theory and practice. Furthermore, these approaches rely on the assumption that MLE achieves good generalization, which may not hold universally.

For example, the category-based method \citep{C51} models the return distribution as a categorical distribution over a fixed set of supports, requiring domain knowledge to specify the support range. Similarly, \citet{GMM-DRL} adopts a Gaussian mixture representation for the return distribution, but its representational power and computational efficiency are limited when handling multimodal distributions in complex scenarios.

In contrast, non-parametric approaches, particularly quantile process regression, overcome these challenges by providing a flexible and straightforward framework for modeling return distributions. This approach enables efficient training, avoids reliance on specific parametric assumptions, and facilitates better generalization in practice.

\subsection{Pseudo Sample Issue}

We are the first to highlight the pseudo-sample issue in quantile-based distributional RL (DRL). Previous studies, such as \citet{EDRL}, did not address this issue, focusing instead on resolving the "statistics and samples confusion" through an imputation step designed to generate target samples $\{v_i\}_{i=1}^{m'}$ based on the affine of quantile estimates $r_1 + \gamma\widehat{f}_{t,\tau_1}(s_i',a_i'),\dots,r_m + \gamma\widehat{f}_{t,\tau_m}(s_i',a_i')$. 

Formally, the imputation operator $\Psi: \mathbb{R}^{m} \rightarrow \Delta(\mathbb{R})$ is defined as $\Psi(r_1 + \gamma\widehat{f}_{t,\tau_1}(s_i',a_i'),\dots,r_m + \gamma\widehat{f}_{t,\tau_m}(s_i',a_i')) = \frac{1}{m'}\sum_{i=1}^{m'} \delta_{y_i}$ mapping a vector of quantile values to a distribution characterized by those quantiles.\footnote{Refer to Example 3.5 in \citet{EDRL} for more details.} he resulting imputed distribution is represented as $\frac{1}{m'}\sum_{i=1}^{m'} \delta_{v_i}$, with Dirac masses located at $\{v_i\}^{m'}_i$. However, this imputation step introduces significant computational overhead to Bellman updates, as it involves solving for $v_i$ by minimizing quantile loss. Furthermore, it fails to adequately address the pseudo-sample issue.

In contrast, quantile process regression resolves this issue seamlessly without incurring additional computational costs, offering a more robust and efficient solution for distributional RL.

%However,  this imputation step introduces significant computational overhead to Bellman updates that involve minimizing quantile loss to solve $v_i$,  and fails to solve the "pseudo sample issue". In contrast, quantile process regression addresses this issue without extra computational cost.

%This representation error leads to a need for an additional projection operator $\Pi_{\mathcal{W}_1}: \Delta(\mathbb{R}) \to \Delta(\mathbb{R})$ defined as $\Pi_{\mathcal{W}_1}\mu = \frac{1}{m} \sum_{i=1}^{m} \delta_{f_{\mu}(\tau_{i})}$ for any $\mu \in \Delta(\mathbb{R})$, where $\tau_{i}=\frac{2 i-1}{2 m}$ \citep{drl_book}. The combined operator $\Pi_{\mathcal{W}_1}\mathcal{T}^{\pi}$ imposes challenges for theoretical analysis, as it is not a contraction under $1$-Wasserstein distance \citep{drl_book}

%deals with this problem by modeling an approximate distribution via an imputation step at each step.  aiming to generate target samples $\{y_i^p\}_{i=1}^{m'}$ based on the affine of quantile estimations $r_1 + \gamma\widehat{f}_{t,\tau_1}(s_i',a_i'),\dots,r_m + \gamma\widehat{f}_{t,\tau_m}(s_i',a_i')$. 

\section{Proof of Contraction Property of $\mathcal{T}^\pi$}\label{sec: proof of  contraction}

The following lemma states that 
 $\mathcal{T}^{\pi}$ is a $\gamma^{1-\frac{1}{2p}}$-contraction under the metric $\overline{\mathcal{W}}_{p,d^{\pi}}(\cdot,\cdot)$.  This result has been proved by  \citet{mle-drl}, and we include the proof here for completeness.

\begin{lemma}[Lemma 4.9 in \citet{mle-drl}]\label{lem: contractive}
 For any $\eta, \eta'\in \Delta(\RR)^{\cS\times \cA}$ and $p\ge1$, the distributional Bellman operator is $\gamma^{1-\frac{1}{2p}}$-contractive under the metric $\overline{\mathcal{W}}_{p,d^{\pi}}(\cdot,\cdot)$,
\begin{align*}
\overline{\mathcal{W}}_{p,d^{\pi}}(\mathcal{T}^\pi \eta,\mathcal{T}^\pi \eta'\big)~\leq~ \gamma^{1-\frac{1}{2p}} ~ \overline{\mathcal{W}}_{p,d^{\pi}}( \eta, \eta'\big).
\end{align*}
\end{lemma}

\begin{proof}
In the proof, let $(S,A)\sim d^\pi$,  and 
for any $(s,a)\in \cS\times \cA$, let 
$S' \sim P ( \cdot  |  s,a)$ and $A' \sim \pi(\cdot| S' )$.
We proceed by   introducing the equation of Bellman flow constraint \citep{Dualdice} that for any $(s,a)\in \cS\times \cA$, 
\begin{align}\label{eq_Bell_flow}
    d^{\pi}(s, a)= (1-\gamma)~ \rho(s)~ \pi(a | s) + \gamma~  \pi(a|s)~ \mathbb{E}_{(S, A) \sim d^{\pi}}\b[  P( s ~|~  S,  A ) \b], 
\end{align}
which  implies
   
\begin{align}\label{key_ineq}
    \pi(a|s)~ \mathbb{E}_{(S, A)  \sim d^{\pi}}\b[  P( s~ |~   S,  A ) \b] ~\leq~ \gamma^{-1}~  d^{\pi}(s, a)  . 
\end{align}
    As a result, we observe that
    \begin{equation} \label{bd_W_2p}
    \begin{aligned}
  & {\mathbb{E}} \b[ \mathcal{W}_{p}^{2p}\left(\eta(S', A'), \eta'(S', A')\right)  \b] \\
    & = \EE_{ (S, A) \sim d^{\pi} }  \l [  \sum_{ s',  a' }  ~  P( s'| S, A)~  \pi(a'|s') ~  \mathcal{W}_{p}^{2p}(\eta(s', a'), \eta'(s', a')) \r] \\
      & =    \sum_{ s',  a' }  ~   \EE_{ (S, A) \sim d^{\pi} }  \l [ P( s'| S, A) \r]  ~  \pi(a'|s') ~  \mathcal{W}_{p}^{2p}(\eta(s', a'), \eta'(s', a')) \\
         & \le   \gamma^{-1} ~   \sum_{ s',  a' }  ~  d^ \pi(s',  a')~  \mathcal{W}_{p}^{2p}(\eta(s', a'), \eta'(s', a')) && \text{by \eqref{key_ineq}}\\
    &=  \gamma^{-1} ~ \mathbb{E}_{(S', A') \sim d^{\pi}} \b[\mathcal{W}_{p}^{2p}(\eta(S', A'), \eta'(S', A'))\b]=  \gamma^{-1}~  \overline{\mathcal{W}}^{2p}_{p, d^{\pi}}(\eta, \eta').
    \end{aligned}      
    \end{equation}  
By applying Lemma \ref{ew},  one can deduce
    \begin{align*}
 \overline{\mathcal{W}}^{2p}_{p, d^{\pi}}(\cT^\pi\eta, \cT^\pi \eta')& =   \mathbb{E}_{(S, A) \sim d^{\pi} } \B[\mathcal{W}_{p}^{2p}\left(\mathcal{T}^\pi \eta(S, A),\mathcal{T}^\pi \eta'(S, A)\right)  \B]
   \\ &  \le \gamma^{2p} ~  \EE\B[\mathcal{W}_{p}^{2p}  (\eta (S',A'),\eta'(S', A') )\B],  \\
 &  \le \gamma^{2p-1} ~ \overline{\mathcal{W}}^{2p}_{p, d^{\pi}}(\eta, \eta')   &&\text{by \eqref{bd_W_2p}, }
    \end{align*}
which concludes the proof of Lemma \ref{lem: contractive}. 
\end{proof}

\begin{lemma}\label{ew}
Fix  any $\eta, \eta' \in \Delta(\mathbb{R})^{\mathcal{S}\times\mathcal{A}}$ and $p\ge1$. Given any $(s,a)\in \cS\times \cA$, we have
$$
\mathcal{W}_{p}^{p}\big(\mathcal{T}^\pi \eta(s, a),\mathcal{T}^\pi \eta'(s, a)\big) ~\leq  ~ \gamma^p ~  \EE\B[\mathcal{W}_{p}^p (\eta (S',A'),\eta'(S', A') )\B], 
$$
where $S' \sim P ( \cdot  |  s,a)$ and $A' \sim \pi(\cdot| S' )$. %$P^{\pi}(S',A'|s, a) : = P(S'|s, a)\pi(A'|S')$.
\end{lemma}
\begin{proof} 
% Fix any $p\ge 1$ and $\eta,\eta'\in \Delta(\mathbb{R})^{\mathcal{S}\times\mathcal{A}}$. 
Fix   any $(s,a)\in \cS\times \cA$. 
First define the space $\Gamma$ of function pairs as 
\begin{align} \label{def_Gamma}
  \Gamma: = \b\{(\psi, \phi): ~ \psi(x)-\phi(y) \leq
  |x-y|^p  \quad\text{for all $x,y\in \RR$}\b\}. 
  \end{align}
By the Kantorovich dual form of Wasserstein distance (see, e.g., \citet[Theorem 1.4.2]{panaretos2020invitation}), we have
\begin{align}\label{kan_dual_eq_1}
\mathcal{W}_{p}^{p}\big(\mathcal{T}^\pi \eta(s, a), \mathcal{T}^\pi \eta'(s, a)\big) &= \sup_{(\psi, \phi) \in \Gamma} \Big\{\EE_{Z\sim (\mathcal{T}^\pi\eta)(s, a)}[\psi(Z)] - \EE_{Z\sim (\mathcal{T}^\pi\eta')(s, a)} [\phi(Z)]\Big\}. 
\end{align}
Let $(R, S', A')$  follow the procedure that
$R\sim \cR(\cdot|s,a)$,  $S' \sim P ( \cdot  |  s,a)$ and $A' \sim \pi(\cdot| S' )$.
For any $r\in \RR$, 
define the functions $\wt \psi_r$ and $\wt \phi_r$ of $z$ as
\[
\wt \psi_r (z):= \psi(r+\gamma z) \quad\text{and} \quad \wt \phi_r(z):= \phi(r+\gamma z).
\]
For any $(\psi, \phi)  \in \Gamma$,  define the function $\xi_{\psi, \phi}$ of $(r,s',a')\in \RR\times \cS\times \cA$ as 
\[
\xi_{\psi, \phi}(r,s',a'):= \EE_{Z\sim \eta(s', a')}[\wt \psi_r (Z)]
-
\EE_{Z\sim \eta'(s', a')}[\wt \phi_r (Z)].
\]
Observe  that 
\begin{equation}\label{xi_eq}
  \begin{aligned}
% \mathcal{W}_{p}^{p}\big(\mathcal{T}^\pi \eta(S, A),\mathcal{T}^\pi \eta'(S, A)\big) &=
&  \sup_{(\psi, \phi) \in \Gamma} \Big\{\EE_{Z\sim (\mathcal{T}^\pi\eta)(s, a)}[\psi(Z)] - \EE_{Z\sim (\mathcal{T}^\pi\eta')(s, a)} [\phi(Z)]\Big\} 
\\
&= 
\sup_{(\psi, \phi) \in \Gamma} ~  \EE \b[\xi_{\psi, \phi} (R, S',A')\b] &&\text{by the definition of $\cT^\pi$}
\\
& \le  \EE\B[
\sup_{(\psi, \phi) \in \Gamma}~ \xi_{\psi, \phi} (R, S',A')\B]  &&\text{by Jensen's inequality. }
\end{aligned}  
\end{equation}
% where   (i) follows from the definition of distributional Bellman operator, and  (ii) follows from Jensen's inequality.
% For any $(\psi, \phi) \in \Gamma$, we can define the function pair $(\widetilde{\psi}, \widetilde{\phi} )$ as 
% $\widetilde{\psi}(y)=\psi(r+\gamma y) / \gamma^p$ and $\widetilde{\phi}(y)=\phi(r+\gamma y) / \gamma^p$. 

We proceed to consider $\sup_{(\psi, \phi) \in \Gamma}~ \xi_{\psi, \phi}(r,s',a')$ for any fixed  $(r,s',a')\in \RR\times \cS\times \cA$. To this end, we observe that for any $x,y\in \RR$, 
\begin{align*}
  \gamma^{-p}~  (\widetilde{\psi}_r (x)-\widetilde{\phi}_r(y)) &= \gamma^{-p}~\b(\psi(r+\gamma x)-\phi(r+\gamma y)\b)\\
  & \leq \gamma^{-p}~ \b| (r+\gamma x)-(r+\gamma y)\b|^p=
  |x-y|^p . 
\end{align*}
This result implies that $(\gamma^{- p}   \widetilde{\psi}_r, ~ \gamma^{- p}   \widetilde{\phi}_r)$ also belongs to $\Gamma$ so that by the Kantorovich dual form of Wasserstein distance, 
\begin{align*}
\sup_{(\psi, \phi) \in \Gamma}~ \xi_{\psi, \phi}(r,s',a')
& =\gamma^{p} ~  \sup_{(\psi, \phi) \in \Gamma}~ \l\{\EE_{Z\sim \eta(s', a')}[ \gamma^{- p} \wt \psi_r (Z)]
-
\EE_{Z\sim \eta'(s', a')}[\gamma^{- p}  \wt \phi_r (Z)]\r\}
\\ 
&\le 
\gamma^{p} ~  \sup_{(\psi, \phi) \in \Gamma}~ \l\{\EE_{Z\sim \eta(s', a')}[\psi  (Z)]- \EE_{Z\sim \eta'(s', a')}[\phi (Z)]\r\} \\
&= \gamma^p ~  \mathcal{W}_{p}^p (\eta (s',a'),\eta'(s',a') ).
\end{align*}
In view of    the above inequality,  taking expectation over $(R,S',A')$ yields
\begin{align}\label{kan_dual_eq_2}
   \EE\B[
\sup_{(\psi, \phi) \in \Gamma}~ \xi_{\psi, \phi} (R, S',A')\B] ~ \le~\gamma^p ~  \EE\B[\mathcal{W}_{p}^p (\eta (S',A'),\eta'(S', A') )\B]. 
\end{align}

Finally, summarizing the results of \eqref{kan_dual_eq_1}, \eqref{xi_eq} and \eqref{kan_dual_eq_2} concludes the proof. 

\end{proof}

\section{Proof of Lemma 4.6 and Lemma 4.7}\label{sec: proof of error propagation}

Before proving Lemma 4.6 and Lemma 4.7, we formally define the empirical distributional Bellman $\widehat{\mathcal{T}}_t^{\pi}$  for each $t\in [T]$.

 Fix   any $\eta \in \Delta(\RR)^{\mathcal{S}\times \mathcal{A}}$. and let $f_\eta (s,a,\cdot )$ 
denote the quantile function  of $\eta(s,a)$
   for any $(s,a)\in \mathcal{S}\times \mathcal{A}$.
 Let $x_i=(s_i,a_i)$ and $y_i=r_i+\gamma f_\eta (s'_i,a'_i,u_i)$, where  $\tau_i$  and $u_i$
 are independently drawn from $\text{Unif}(0,1)$. For each step $t\in [T]$, define the empirical risk 
 \begin{align*}
\wt \cL (f) &= \frac{1}{|\mathcal{D}_t|} \sum_{(s_i,a_i,r_i,s'_i)\in\mathcal{D}_t}\rho_{\tau_i}\b(y_i-f(x_i, \tau_i)\b),
\end{align*}
Let  $\wt f_\eta $ be the minimizer of  $\wt \cL (f) $  over the prescribed function class on    $\cS\times \cA \times (0,1)$. This induces an element  $\wt \eta$
in $\Delta(\RR)^{\mathcal{S}\times \mathcal{A}}$  through the distribution associated with  $\wt f_\eta(s,a, \cdot)$ for any $(s,a)\in \mathcal{S}\times \mathcal{A}$.  We then define  the operator 
 \begin{align}\label{equ: empirical distributional Bellman operator}
 \widehat{\mathcal{T}}_t^{\pi}  : \Delta(\RR)^{\mathcal{S}\times \mathcal{A}} \to \Delta(\RR)^{\mathcal{S}\times \mathcal{A}}, \qquad \eta \mapsto  \wt \eta.   
 \end{align}

 \paragraph{Iterative estimation procedure:} 

Given the initialization $\widehat{\eta}_0$, Algorithm 1 in the main text constructs the sequence $\{\widehat{\eta}_t\}_{t\in[T]}$ recursively by
$$
\widehat{\eta}_t=\widehat{\mathcal{T}}_t^\pi \widehat{\eta}_{t-1},
\qquad t\in[T].
$$ 
%Recall that initializing the estimate $\wh \eta_0$,  $\{\wh \eta_t\}_{t\in [T]}$ are obtained from Algorithm 1 of the main text.  Specifically,  at each step $t\in [T]$,  given  $\wh \eta_{t-1}$, the $t$-step estimate $\wh \eta_t$ is obtained as $\wh \eta_t =\wh  \cT^{\pi}_t \wh \eta_{t-1}$,  where $\wh  \cT^{\pi}_t$ is defined in \eqref{equ: empirical distributional Bellman operator}.  
Here  $\wh  \cT^{\pi}_t$ serves as an empirical approximation of $\cT^\pi$  at step $t \in [T]$, as $\cT^\pi$ is not directly accessible.  %By hypothetically  replacing $\wh  \cT^{\pi}_t$ with $\cT^{\pi}$ for each $t\in [T]$,  one can deduce another sequence  $\eta_0^*, \eta_1^*,\ldots,\eta_T^*$, given by $\eta_0^*= \wh \eta_0$ and $\eta_t^* = {\mathcal{T}}^\pi \wh  \eta_{t-1}$ for $t\in [T]$. 
To isolate the error induced by replacing $\mathcal{T}^\pi$ with $\widehat{\mathcal{T}}_t^\pi$, consider the auxiliary sequence $\{\eta_t^*\}_{t=0}^T$ defined by
$$
\eta_0^*=\widehat{\eta}_0,
\qquad
\eta_t^*=\mathcal{T}^\pi \widehat{\eta}_{t-1},
\qquad t\in[T].
$$
The deviation between $ \eta_t^*$ and $\wh  \eta_t $, quantified by  $\wh \varepsilon_{p,t} = \overline{\mathcal{W}}_{p,\mu}(\widehat{\eta}_{t},\eta_{t}^*)$,  measures the one-step estimation error due to using $\widehat{\mathcal{T}}_t^\pi$ in place of $\mathcal{T}^\pi$.  This observation underlies the decomposition of the sub-optimality  $\mathcal{W}_p(\boldsymbol{\eta}^{\pi},\boldsymbol{\widehat{\eta}}_{T})$, used in the proof of  Lemma 4.6. 

It  is worth mentioning  that   the dataset $\cD$ is  partitioned into $T$ disjoint subsets $\mathcal{D}_1,\dots,\mathcal{D}_{T}$, and that step $t\in [T]$ uses only $\mathcal{D}_t$. Hence,  conditional on $\mathcal{D}_1,\dots,\mathcal{D}_{t-1}$, the target $\eta_t^*$ is fixed, while $\wh \eta_t$  is estimated from $\mathcal{D}_{t}$.
Recall the definitions of  $\{\wh f_t\}_{t\in [T]}$  and $\{f^*_t\}_{t\in [T]}$   from Section 3.1 of the main text. 
 For each $t\in [T]$ and any  $(s,a)\in \cS\times \cA$, 
 $\wh f_t(s,a,\cdot)$ and   $f^*_t(s,a,\cdot)$ are quantile functions corresponding to  
 $\wh \eta_t (s,a )$ and  $ \eta^*_t (s,a )$, respectively.

\begin{proof}[Proof of Lemma 4.6]
Recall $\Gamma$ from \eqref{def_Gamma}.  
Define 
$\eta_0: =\wh \eta_0$ and $\eta_t: = \cT^\pi \eta_{t-1}$ for each $t\in [T]$. 
% For short, define  the joint distribution $\nu$ over $(S, A)$ as
% $\nu (s,a) : = \rho(s) \pi(a|s), \forall~  (s,a)\in \cS\times \cA$.  
Recall  
$\eta^{\pi}$ and $\boldsymbol{\eta}^{\pi}$ from Section 2 of the main text. 
For any $\eta\in \Delta(\RR) ^{ \cS\times \cA}$, write 
$\boldsymbol{\eta}= \EE_{(S,A)\sim \rho\times \pi } [\eta(S,A)] $. 
Applying the Kantorovich dual form of Wasserstein distance gives
\[
\mathcal{W}_p^p(\boldsymbol{\eta}^{\pi},\boldsymbol{\eta}) = \sup_{(\psi, \phi) \in \Gamma}~  \Big\{  \EE_{Z\sim \boldsymbol{\eta}^{\pi}} [\psi(Z)]  -  \EE_{Z\sim \boldsymbol{\eta}} [\phi(Z) ]\Big\}. 
\]
For any $(\psi, \phi)  \in \Gamma$,
define the function $\zeta_{\psi,\phi}$ of $(s,a)\in \cS\times \cA$  as 
\[
\zeta_{\psi,\phi}(s,a) :=  \EE_ {Z\sim \eta^{\pi}(s,a)}[\psi(Z) ] -   \EE_ {Z \sim \eta (s,a)}[\phi(Z) ] . 
\]
It follows that 
\begin{equation}\label{bd_Wp_eta_pi}
  \begin{aligned}
\mathcal{W}_p^p(\boldsymbol{\eta}^{\pi},\boldsymbol{\eta}) =&  \sup_{(\psi, \phi) \in \Gamma}~  \EE_{(S,A)\sim \rho\times \pi} \l[ \zeta_{\psi,\phi}(S,A)  \r]\\
\le & ~  \EE_{(S,A)\sim \rho\times \pi} \B[ \sup_{(\psi, \phi) \in \Gamma}~  \zeta_{\psi,\phi}(S,A)  \B]  &&\text{by Jensen's inequality}
\\
= & ~  \EE_{(S,A)\sim\rho\times \pi} \B[ \mathcal{W}_{p}^{p}\big(\eta(S, A), \eta^{\pi}(S, A) \B] &&\text{by Kantorovich dual} \\
= & ~ \l(  \EE_{(S,A)\sim \rho\times \pi} \B[ \mathcal{W}_{p}^{2p}\big(\eta(S, A), \eta^{\pi}(S, A) \B]\r)^{\frac{1}{2}} &&\text{by Jensen's inequality.}
\end{aligned}  
\end{equation}
Applying 
\eqref{eq_Bell_flow} gives
\[
\rho(s)\pi(a| s)~\leq~  (1-\gamma)^{-1}~  d^{\pi}(s, a)\quad  \forall ~ (s,a) \in  \cS\times \cA,
\]
which, together with \eqref{bd_Wp_eta_pi},  implies
\begin{align*}
\mathcal{W}_p^p(\boldsymbol{\eta}^{\pi},\boldsymbol{\eta})
&~\leq~  (1-\gamma)^{-\frac{1}{2}}~ \Big(\EE_{(S,A)\sim d^{\pi}}\b[\mathcal{W}_{p}^{2p}(\eta(S, A), \eta^{\pi}(S, A))\b]\Big)^{\frac{1}{2}} \\
 & ~=~ (1-\gamma)^{-\frac{1}{2}}~  \overline{\mathcal{W}}^p_{p,d^{\pi}}({\eta}, \eta^{\pi}).
\end{align*}
Applying the above inequality with $\boldsymbol{\eta} =\boldsymbol{\widehat{\eta}}_T$   yields
\begin{align}\label{kanto_ineq} 
\mathcal{W}_p(\boldsymbol{\eta}^{\pi},\boldsymbol{\widehat{\eta}}_T)
~\leq~  (1-\gamma)^{-\frac{1}{2p}}~  \overline{\mathcal{W}}_{p,d^{\pi}}({\wh \eta_T}, \eta^{\pi}).
\end{align}
It follows from  the triangle inequality that
%\begin{align*}
%& \Big(\underset{s, a \sim d^{\pi}}{\mathbb{E}} \mathcal{W}_{p}^{2p}\left(\widehat{\eta}_{t}(s,a), \eta^{\pi}(s, a)\right)\Big)^{\frac{1}{2p}} \\
%\leq & \underbrace{\Big(\underset{s, a \sim d^{\pi}}{\mathbb{E}} \mathcal{W}_{p}^{2p}\left(\widehat{\eta}_{t}(s,a), \eta_t(s, a)\right)\Big)^{\frac{1}{2p}}}_{(a)}+   \underbrace{\Big(\underset{s, a \sim d^{\pi}}{\mathbb{E}} \mathcal{W}_{p}^{2p}\left(\eta_t(s, a), \eta^{\pi}(s, a)\right)\Big)^{\frac{1}{2p}}}_{(b)}.
%\end{align*}
\begin{align}\label{decom_eta_t}
& \overline{\mathcal{W}}_{p,d^{\pi}}(\widehat{\eta}_{T},\eta^{\pi})
~\leq~  \underbrace{\overline{\mathcal{W}}_{p,d^{\pi}}(\widehat{\eta}_{T},\eta_T)}_{(a)}+   \underbrace{\overline{\mathcal{W}}_{p,d^{\pi}}(\eta_T,\eta^{\pi})}_{(b)}.
\end{align}
We proceed to bound the terms (a) and (b) separately. 

For  (a), we can expand it recursively. Specifically, 
write
$a_t: =  \overline{\mathcal{W}}_{p,d^{\pi}}(\widehat{\eta}_{t},\eta_{t})$ 
for each $t\in [T] \cup\{0 \}$.  
Further write $b_t:=\overline{\mathcal{W}}_{p,d^{\pi}}(\widehat{\mathcal{T}}_t^{\pi}\widehat{\eta}_{t-1},\mathcal{T}^{\pi}\widehat{\eta}_{t-1})$ for each $t\in [T] $ . 
By Assumption 4.3, it can be seen that for each   $t\in [T] $, 
\begin{align}\label{use_assu_4_3}
  b_t~  = ~  \overline{\mathcal{W}}_{p,d^{\pi}}(\widehat{\eta}_{t },\mathcal{T}^{\pi}\widehat{\eta}_{t-1})   ~  \le ~ C_\mu^{\frac{1}{2p}} ~  \overline{\mathcal{W}}_{p,\mu}(\widehat{\eta}_{t },\mathcal{T}^{\pi}\widehat{\eta}_{t-1}) ~= ~   C_\mu^{\frac{1}{2p}}~ \wh \varepsilon_{p,t}  . 
\end{align}
%Before illustrating this recursive relation, we first
For each $t\in [T]$, 
 observe that 
\begin{align*}
a_t  &=  \overline{\mathcal{W}}_{p,d^{\pi}}(\widehat{\mathcal{T}}_t^{\pi}\widehat{\eta}_{t-1},\mathcal{T}^{\pi}\eta_{t-1}) \\
&\leq   \overline{\mathcal{W}}_{p,d^{\pi}}(\mathcal{T}^{\pi}\widehat{\eta}_{t-1},\mathcal{T}^{\pi}\eta_{t-1}) + \overline{\mathcal{W}}_{p,d^{\pi}}(\widehat{\mathcal{T}}_t^{\pi}\widehat{\eta}_{t-1},\mathcal{T}^{\pi}\widehat{\eta}_{t-1})  \\
    &\le  \gamma^{1-\frac{1}{2p}}~ \overline{\mathcal{W}}_{p,d^{\pi}}(\widehat{\eta}_{t-1},\eta_{t-1}) +   \overline{\mathcal{W}}_{p,d^{\pi}}(\widehat{\mathcal{T}}_t^{\pi}\widehat{\eta}_{t-1},\mathcal{T}^{\pi}\widehat{\eta}_{t-1}) &&\text{by Lemma \ref{lem: contractive}}\\
   % & \le   C_{\mu}^{\frac{1}{2p}}\gamma^{1-\frac{1}{2p}}~ \overline{\mathcal{W}}_{p,\mu}(\widehat{\eta}_{t-1},\eta_{t-1}) +  C_{\mu}^{\frac{1}{2p}}~ \overline{\mathcal{W}}_{p,\mu}(\widehat{\mathcal{T}}_t^{\pi}\widehat{\eta}_{t-1},\mathcal{T}^{\pi}\widehat{\eta}_{t-1})&&\text{by Assumption 4.3 } \\
    %&=   \gamma^{1-\frac{1}{2p}}~\overline{\mathcal{W}}_{p,d^{\pi}}(\widehat{\eta}_{t-1},\eta_{t-1}) +  \widehat{\varepsilon}_t \\
    & = \gamma^{1-\frac{1}{2p}} ~ a_{t-1} +  b _t . 
\end{align*}
By carrying out the recursion and noting that $a_0= \overline{\mathcal{W}}_{p,d^{\pi}}(\widehat{\eta}_{0},\eta_0)=0$,  one can deduce 
\begin{equation}\label{equ: first term of suboptimal decomp}
    \begin{aligned}
    (a) ~ &\leq ~  \sum_{k=0}^{T-1}~ \gamma^{k (1-\frac{1}{2p})}~  b _{T-k}\\ 
 ~  & \leq ~  C_{\mu}^{\frac{1}{2p}}~\sum_{k=0}^{T-1}\gamma^{k(1-\frac{1}{2p})}~ \widehat{\varepsilon}_{p, T-k}  &&\text{~~~~~~~~~~~~~~~ by \eqref{use_assu_4_3}}\\ 
 ~ & \leq ~  \frac{C_{\mu}^{\frac{1}{2p}}}{1-\gamma^{1-\frac{1}{2p}}}~\max_{0<k\leq T} \widehat{\varepsilon}_{p,   k}.  
\end{aligned}
\end{equation}
% where the first inequality uses the fact that $\overline{\mathcal{W}}_{p,d^{\pi}}(\widehat{\eta}_{0},\eta_0)=0$, and the second inequality follows $\overline{\mathcal{W}}_{p,d^{\pi}}(\widehat{\mathcal{T}}_t^{\pi}\widehat{\eta}_{t-1},\mathcal{T}^{\pi}\widehat{\eta}_{t-1})\leq C_{\mu}^{\frac{1}{2p}}~ \overline{\mathcal{W}}_{p,\mu}(\widehat{\mathcal{T}}_t^{\pi}\widehat{\eta}_{t-1},\mathcal{T}^{\pi}\widehat{\eta}_{t-1})$ by Assumption 4.3.

For (b), it follows from the triangle inequality that % we can perform the same recursion.
\begin{align*}
(b) 
~\leq ~& \overline{\mathcal{W}}_{p,d^{\pi}}(\eta_T,\mathcal{T}^{\pi}\eta_{T-1}) +\overline{\mathcal{W}}_{p,d^{\pi}}(\mathcal{T}^{\pi}\eta_{T-1},\eta^{\pi})  \\
~=  ~& \overline{\mathcal{W}}_{p,d^{\pi}}(\eta_T,\mathcal{T}^{\pi}\eta_{T-1}) + \overline{\mathcal{W}}_{p,d^{\pi}}(\mathcal{T}^{\pi}\eta_{T-1},\mathcal{T}^{\pi}\eta^{\pi}) &&\text{by $\eta^\pi =  \mathcal{T}^{\pi}\eta^\pi $}  \\
~\leq ~  &  0 + \gamma^{1-\frac{1}{2p}} ~ \overline{\mathcal{W}}_{p,d^{\pi}}(\eta_{T-1},\eta^{\pi}) &&\text{by $\eta_T =  \mathcal{T}^{\pi}\eta_{T-1}$ and Lemma \ref{lem: contractive}.}
\end{align*}
By repeating the above arguments to deduce 
\[
\overline{\mathcal{W}}_{p,d^{\pi}}(\eta_{t},\eta^{\pi})~  \le~   \gamma^{1-\frac{1}{2p}} ~ \overline{\mathcal{W}}_{p,d^{\pi}}(\eta_{t-1},\eta^{\pi})  \quad   \forall ~ t\in [T] , 
\]
 we obtain
\begin{align}\label{bd_b_1}
     (b)~\leq ~\gamma^{T(1-\frac{1}{2p})} ~ \overline{\mathcal{W}}_{p,d^{\pi}}(\eta_0,\eta^{\pi}). 
\end{align}
 For bounding $\overline{\mathcal{W}}_{p,d^{\pi}}(\eta_0,\eta^{\pi})$, 
 we observe that for any $(s,a)\in \cS\times \cA$, 
 \begin{align*}
     \cW_p( \eta_0 (s,a), \eta^{\pi}  (s,a) )&  = \l( \int_0^1 |\wh f _0 (s,a,t )-f^*(s,a,t )| ^p ~ d t \r)^{\frac{1}{p}} \\
     & \le \l(F+ \frac{R_{max}}{1-\gamma}\r)\le \frac{1}{1-\gamma} \l( F + {R_{max}}\r),
 \end{align*}
 implying that
 \[
 \overline{\mathcal{W}}_{p,d^{\pi}}(\eta_0,\eta^{\pi}) = \l(\EE _{(S,A) \sim d^\pi}\l[ \cW_p^{2p}( \eta_0 (S,A ), \eta^{\pi}  (S,A ) )  \r]\r)^{\frac{1}{2p}} \le \frac{1}{1-\gamma} \l( F + {R_{max}}\r). 
 \]
 Together with \eqref{bd_b_1} gives
\begin{align}\label{second term}
  (b)~ \leq ~  \frac{\gamma^{T(1-\frac{1}{2p})}}{1-\gamma}~ \l( F + {R_{max}}\r).
\end{align} 
Combining (\ref{kanto_ineq}), \eqref{decom_eta_t},  (\ref{equ: first term of suboptimal decomp}) and (\ref{second term}) yields
\begin{align}\label{bd_W_p_final}\nonumber
\mathcal{W}_p(\boldsymbol{\eta}^{\pi},\boldsymbol{\widehat{\eta}}_T)~ \leq~ & (1-\gamma)^{-\frac{1}{2p}}~  \overline{\mathcal{W}}_{p,d^{\pi}}(\eta_T,\eta^{\pi})\nonumber\\
~\leq~ & \frac{C_{\mu}^{\frac{1}{2p}}}{(1-\gamma)^{\frac{1}{2p}}(1-\gamma^{1-\frac{1}{2p}})} \max_{0<t\leq T}\widehat{\varepsilon}_{p, t} + \frac{\gamma^{T(1-\frac{1}{2p})}}{(1-\gamma)^{1+\frac{1}{2p}}}\l( F + {R_{max}}\r) \nonumber \\
~\leq~ & \frac{C_{\mu}^{\frac{1}{2p}}}{(1-\gamma)^{\frac{1}{2p}}(1-\gamma^{1-\frac{1}{2p}})} \max_{0<t\leq T}\widehat{\varepsilon}_{p, t} + \frac{\gamma^{\frac{T}{2}}}{(1-\gamma)^{\frac{3}{2}}}\l( F + {R_{max}}\r) &&\text{ by $p\ge1$}.
\end{align} 
By  noting
\[
(1-\gamma)^{\frac{1}{2p}}(1-\gamma^{1-\frac{1}{2p}}) \ge   (1-\gamma)^{\frac{1}{2}}(1-\gamma^{\frac{1}{2}})
\]
and 
\[
(1-\gamma^{\frac{1}{2}})>
\frac{(1-\gamma)}{2}, 
\]
the inequality\eqref{bd_W_p_final} implies that
\begin{align}\nonumber
\mathcal{W}_p(\boldsymbol{\eta}^{\pi},\boldsymbol{\widehat{\eta}}_T)~
\le ~\frac{2 C_{\mu}^{\frac{1}{2p}}}{(1-\gamma)^{\frac{3}{2}}}~\max_{0<t\leq T} \widehat{\varepsilon}_{p, t}
+\frac{\gamma^{\frac{T}{2}}}{(1-\gamma)^{\frac{3}{2}}}~ \l( F + {R_{max}}\r),
\end{align}
hence we finish the proof. 
% where (i) holds since $1/((1-\gamma)^{\frac{1}{2p}}(1-\gamma^{1-\frac{1}{2p}}))$ achieves maximum at $p=1$, and (ii) holds because $(1-\gamma^{\frac{1}{2}})<(1-\gamma)/2$.
\end{proof}

\begin{proof}[Proof of Lemma 4.7]
Consider any  $\eta, \eta'\in \Delta (\RR)^{\cS\times \cA}$. For any $(s,a)\in \cS\times \cA$, let   $f(s,a,\cdot )$ and $f'(s,a,\cdot )$ be the quantile functions corresponding to $\eta(s,a)$ and $\eta'(s,a)$, respectively.
Observe that for any $p\ge1$, 
\begin{equation}\label{link_w_L_2}
\begin{aligned} 
\overline{\mathcal{W}}_{p,\mu}(\eta,\eta') ~=~ & \Big(\EE_{(S,A)\sim \mu} \b[ \mathcal{W}_{p}^{2p}\left(\eta(S, A),  \eta'(S,A)\right)\b] \Big)^{\frac{1}{2p}}\\
~  \stackrel{(i)}{=}~  &\l(\EE_{(S,A)\sim \mu}  \B(
 \mathbb{E}_{\tau} \b[~ \b|f(S, A, \tau)-f'(S, A, \tau) \b|^{p}~  \b] \B)^{2} \r)^{\frac{1}{2p}}\\ 
~ \stackrel{(ii)}{\leq} ~ & \Big(\EE_{(S,A)\sim \mu} 
 \mathbb{E}_{\tau} \b[~ \b|f(S, A, \tau)-f'(S, A, \tau) \b|^{2p}~  \b]  \Big)^{\frac{1}{2p}}\\
 % = & \Big(\EE_{(S,A)\sim \mu} {\mathbb{E}}\left|f(s,a,\tau)-f'(s,a,\tau)\right|^{2p} \Big)^{\frac{1}{2p}}\\
 ~ = ~~  & \b \|f-f'\b\|_{2p,\wt \mu },
    %\stackrel{(iii)}{\leq}& \left(\mathbb{E}_{\tau}  \underset{s,a\sim \mu}{\mathbb{E}} \left\{\min\{|f(s,a,\tau)-f'(s,a,\tau)|,|f(s,a,\tau)-f'(s,a,\tau)|^2\}\right\}^{p}\right)^{\frac{1}{2p}}
   % = & \max\{1,2F\}^{\frac{1}{2}} \cdot \left(\mathbb{E}_{\tau} \underset{s,a\sim \mu}{\mathbb{E}} (\Delta^2(f,f'))^p\right)^{\frac{1}{2p}} 
\end{aligned}    
\end{equation}
where $\EE_\tau$ takes expectation over $\tau\sim \text{Unif}(0,1)$,  (i) follows from the definition of Wasserstein distance and (ii) follows from Jensen's inequality.

% are the corresponding quantile function to $\widehat{\mathcal{T}}_t^{\pi}\widehat{\eta}_{t-1}$ and $\mathcal{T}^{\pi}\widehat{\eta}_{t-1}$, respectively. We will use this property in the subsequent proof. 
Under Assumption 4.5, for each $t\in [T]$, 
by applying  \eqref{link_w_L_2}
with $p = 1$, $  \eta   =\wh \eta_t$ and $ \eta'    = \eta_t^*$,  we obtain
\begin{align*}
\wh  \varepsilon_{1,t}= 
 \overline{\mathcal{W}}_{1,\mu}(\wh \eta_t ,\mathcal{T}^{\pi}\widehat{\eta}_{t-1})
= \overline{\mathcal{W}}_{1,\mu}(\wh \eta_t ,  \eta_t ^*)
\leq 
  \|\widehat{f}_t-f^*_t\|_{2,\wt \mu }
 ~  \leq ~ c_0^{-\frac{1}{2}} ~ \big(\mathcal{L}_t(\widehat{f}_t)-\mathcal{L}_t(f^*_t)\big)^{\frac{1}{2}}, 
% \overline{\mathcal{W}}_{1,\mu}(\eta,\eta') \leq 
 % \|f-f'\|_{2, \mu\times \tau} \le c_0^{-\frac{1}{2}} ~ \big|\mathcal{L}_t(f)-\mathcal{L}_t(f') \big|^{\frac{1}{2}}.
\end{align*}
hence conclude the proof. 
\end{proof}

\section{Proof  of  Theorem 4.8  and Theorem 4.9}\label{sec: C}
%\section{Proof  of  Section 4.1}\label{sec: C}

Before presenting the proof of Theorems 4.8 and 4.9, %analyzing the Rademacher complexity, 
we first provide proof for a key property outlined in Section 3.1 of the main text. Specifically, for each $t\in[T]$ and any given pair $(s,a)$,  the random variable $Z_{t}(S',A')$  is identical in distribution to $ \widehat{f}_{t}(S', A',U)$.  Here, $S'\sim P(\cdot| s,a)$, $A'\sim \pi  (\cdot| S')$, and $U$ is independent of $(S',A')$, drawn from $\text{Unif}(0,1)$. %, $S'$ and $A'$ are generated according to  $S'\sim P(\cdot| s,a)$  and $A'\sim \pi  (\cdot| S')$, respectively. 
 This identity in distribution allows us to transfer a one-step distributional Bellman update into a task of quantile process regression.  Generally, let  $\{G(x'), x'=(s',a')\in \mathcal{S}\times \mathcal{A}\}$ be a collection of random variables. It suffices to prove the following proposition. 
\begin{proposition}\label{lem: quantile process generator}
For any $x'\in \mathcal{S}\times \mathcal{A}$,  let $f(x', \tau)$ be the quantile function corresponding to the random variable $G(x')$. Then $G(X')$ and $f(X',U)$ have the same distribution. 
\end{proposition}
\begin{proof}
For any $y\in \mathbb{R}$, by law of total expectation, we have
 \begin{align}\label{G_X_prime_law_exp}
        \mathbb{P}[G(X') \le y]& =  \mathbb{E}\Big[\mathbb{P}\big[G(X') \le y\big] \mid X'\Big].
\end{align}
Given $X'= x'$, let $F(x', \cdot)$ denote the cumulative distribution function of $G(x')$, which is the inverse function of $f(x',\cdot)$. Since $U$ is a uniform variable, simple algebra yields
\begin{align*}
 \mathbb{P}\big[ f(x',U) \le y \big]=  \mathbb{P}\big[ U \le F(x',y)\big]=F(x',y)=   \mathbb{P}\big[ G(x') \le y\big]
\end{align*}
so that together with \eqref{G_X_prime_law_exp} gives
\begin{align*}
\mathbb{P}\big[G(X') \le y \big]=\mathbb{E}\Big[\mathbb{P}\big[ f(X',U) \le y\big] \mid X'\Big]= \mathbb{P}\big[ f(X',U) \le y\big].
\end{align*}
We thus conclude the statement in  Proposition \ref{lem: quantile process generator}. 
\end{proof}

We proceed to analyze the quantile process regression with the neural network, utilizing Rademacher complexity to derive bounds on the one-step Bellman error.  Before delving into the formal analysis,
recall that given the width  $W$ and length $L$,  $\mathcal{F}=\mathcal{F}(W, L)$ denotes the  function class
consisting of all functions defined in (6) of the main text such that $\|f\|_\infty\le F$. 
The target function  $f^{*}_t$ at step $t$ is defined as 
$$
 f^{*}_t =\arg\min_{f}\mathcal{L}_t(f) = \arg\min_{f} \mathbb{E}_{X,Y_t,\tau}\b( \rho_{\tau}(Y_t -f(X,\tau))\b),
$$
where $X=(S, A)$ denotes the state-action random pair and  $Y_t = R + \gamma Z_{t-1}(S',A')$ is the target random response.
In our analytical framework for DQPOPE, for each $t\in [T]$, 
we assume the true target $ f^{*}_t$ belongs to the H\"{o}lder class $\mathcal{G}:=\mathcal{G}([0,1]^d, \beta, H)$, as defined in Definition 4.2 of the main text. The estimator
$\widehat{f}_t$ is defined as the solution to the optimization task 
$$
\wh f_t = \arg\min_{f\in\mathcal{F}}\mathcal{\widehat{L}}_{t}(f),\quad  \mathcal{\widehat{L}}_{t}(f)=\frac{1}{|\mathcal{D}_t|} \sum_{(s_i,a_i,r_i,s'_i)\in\mathcal{D}_t}\rho_{\tau_i}\left(y_i-f(x_i, \tau_i)\right),
$$
where, as formulated in the main text, 
$(x_i, y_i)$'s is drawn according to the distribution followed by $(X, Y_t)$ with $x_i= (s_i,a_i)$, and $\tau_i$ is independently sampled from $\text{Unif}(0,1)$.

For brevity, we will always suppress the subscript $t$ in the remaining part.  Additionally, without loss of generality,  we assume the $\mathcal{D}_t$ contains the first $n$ samples in $\mathcal{D}$, which reduces the empirical risk to 
\begin{align} \label{em_risk}
\mathcal{\widehat{L}}(f) = \frac{1}{n}\sum_{i=1}^n \rho_{\tau_i}\left(y_i-f(x_i, \tau_i)\right)    . 
\end{align}

The analysis of excess risk bound relies on the notion of Rademacher complexity, which serves as a useful tool for bounding the expectation of empirical process. 
Specifically, 
let 
$\{\sigma_i\}_{i=1}^n$ be independent Rademacher random variables (i.e. $1$ and $-1$  equiprobably).   
Then,  the empirical Rademacher complexity and the Rademacher complexity of $\mathcal{F}$ are respectively defined as
\begin{align}\label{def_Rade}
\widehat{\mathcal{R}}(\mathcal{F}) := \mathbb{E}_{\sigma} \Big[\sup_{f\in\mathcal{F}}\Big|\frac{1}{n}\sum_{i=1}^n \sigma_i f(x_i,\tau_i) \Big| \Big],~~~\mathcal{R}(\mathcal{F}):=\mathbb{E}[\widehat{\mathcal{R}}(\mathcal{F})],
\end{align}
where  $\mathbb{E}_{\sigma} $ takes expectation with respect to $\{\sigma_i\}_{i=1}^n$.
%and $(X,\tau)_1^{n}=\{(x_1,\tau_1),...,(x_n,\tau_n)\}$, respectively. 

 We are now ready to derive the excess risk bound.

\begin{theorem}[Excess risk bound, slow rate]\label{thm: bd_ slow_appendix}
%Consider a dataset $\mathcal{D} = \{(x_i,y_i,\tau_i)\}_{i=1}^n$ where $X_i,Y_i$ are i.i.d. sampled from random variable $X,Y$ respectively, and $\tau_i$ is sampled from $\mathrm{Unif}(0,1)$ independent of $(X_i,Y_i)$.
Suppose Assumption 4.4 is satisfied. 
With probability at least $1-2n^{-1}$, the excess risk satisfies
 \begin{align}\label{equ: upper_bound_B1}
\mathcal{L}(\widehat{f})-\mathcal{L}(f^{*})\leq C F\sqrt{\frac{W^2 L^2\log(W^2L) \log n}{n}}+\inf_{f\in\mathcal{F}}\left(\mathcal{L}(f)-\mathcal{L}(f^{*})\right),
\end{align}
where $C$ is a constant independent of $F,W,L,n $. Furthermore,
for sufficiently large $U, V\in \mathbb{N}_{+}$, 
setting width and length to be
 $ W=\mathcal{O}\left((s+1)^2 d^{s+1} U \log U\right) \ \text{and}\ L=\mathcal{O}\left((s+1)^2 V \log V\right),$
% if $n \geq \operatorname{Pdim}(\mathcal{F})$, we have 
%  \[
%  \inf_{f\in\mathcal{F}}\left(\mathcal{L}(f)-\mathcal{L}(f^{*})\right)\lesssim H(s+1)^2 d^{s+\frac{\beta \vee 1}{2}}(U V)^{-2 \beta / d}.
%  \] 
% Here, recall that $H,s,\beta$ are involved in the definition of H\"{o}lder class \ref{def: holder class}. 
if we choose $UV=\lfloor n^{\frac{d}{4\beta+2d}} \rfloor$,  when $n$ is sufficiently large, with probability at least $1-2n^{-1}$,
the excess risk has upper bound that
$$
\mathcal{L}(\widehat{f})-\mathcal{L}(f^{*})\leq C(s+1)^4 d^{s+\frac{\beta }{2}\vee 1} (\log n)^3 n^{-\frac{\beta}{2\beta+d}},
$$
where $C$ is a constant independent of $s,\beta, d, n$. 
\end{theorem}

\begin{proof}
% The proof is split into three steps. In the first step, we decompose the excess risk into two terms:  stochastic error and approximation error. 
% Then,
% we separately bound  the stochastic and approximation errors.
% Finally, we measure the approximation error of the function class $\mathcal{F}$ to the H\"{o}lder class {$\mathcal{G}$} and make a balance between these two types of error.
% \vspace{0.2in}

% \noindent{\it Step (i): Excess risk decomposition.}
%\noindent 
Note that $\widehat{f}$ minimizes the empirical risk function $\wh \cL(\cdot)$ over the function class $\cF$.
% , then for any $f\in \cF$, 
% we have $\widehat{\mathcal{L}}(f)\geq\widehat{\mathcal{L}}(\widehat{f})$.  
For any fixed $f\in \cF$, 
the excess risk satisfies
\begin{equation}\label{eq_excess_risk_decom}
\begin{aligned}
&\mathcal{L}(\widehat{f})-\mathcal{L}(f^{*}) \\
&= \mathcal{L}(\widehat{f})-\widehat{\mathcal{L}}(\widehat{f})+\widehat{\mathcal{L}}(\widehat{f})- \widehat{\mathcal{L}}(f)+\widehat{\mathcal{L}}(f) -\mathcal{L}(f) +\mathcal{L}(f)-\mathcal{L}(f^{*})\\
&\leq \mathcal{L}(\widehat{f})-\widehat{\mathcal{L}}(\widehat{f})  + \widehat{\mathcal{L}}(f) -\mathcal{L}(f) + \mathcal{L}(f)-\mathcal{L}(f^{*}), && \text{by $\widehat{\mathcal{L}}(f)\geq\widehat{\mathcal{L}}(\widehat{f})$}\\
&\leq 2~  \sup_{f\in\mathcal{F}} ~ \big|\mathcal{L}(f)-\widehat{\mathcal{L}}(f) \big| + \big(\mathcal{L}(f)-\mathcal{L}(f^{*})\big).
\end{aligned}
\end{equation}
% where the first inequality holds by the the optimality of $\widehat{f}$, i.e. $\widehat{\mathcal{L}}(f)\geq\widehat{\mathcal{L}}(\widehat{f})$.
Since the above inequality holds for arbitrary $f\in\mathcal{F}$, the following decomposition of excess risk holds. 
\begin{align}\label{excess_risk}
\mathcal{L}(\widehat{f})-\mathcal{L}(f^{*})\leq 2~\sup_{f\in\mathcal{F}} ~ \big|\mathcal{L}(f)-\widehat{\mathcal{L}}(f) \big| + \inf_{f\in\mathcal{F}}~ \big(\mathcal{L}(f)-\mathcal{L}(f^{*})\big).
\end{align}
The first term appearing on the right-hand side of \eqref{excess_risk} 
represents the stochastic error 
(also known as the empirical process), 
which is the primary focus of statistical learning. The second term represents the approximation error, which is deterministic in nature. Both components vary with different choices of the width and length of $\mathcal{F}$.  Once bounds for stochastic and approximation errors are available, we can immediately obtain the upper bound for the excess risk. 
We proceed to bound the stochastic and approximation errors.
\vspace{0.1in}

% \wang{
% \begin{lemma}[Excess risk decomposition I]\label{lem: Excess risk decomposition}
% For the estimator $\widehat{f}_t$ defined as the minimizer of  \eqref{equ: empirical loss}, its excess risk can be upper bounded as follows,
%    \begin{align}\label{equ: excess risk decomposition slow}
%    \mathcal{L}_t(\widehat{f}_t)-\mathcal{L}_t(f^*_t)\leq 2\sup_{f\in\mathcal{F}} \big|\mathcal{L}_t(f)-\widehat{\mathcal{L}}_t(f) \big| + \inf_{f\in\mathcal{F}}\left(\mathcal{L}_t(f)-\mathcal{L}_t(f^*_t)\right).
%    \end{align} 
% \end{lemma}
%  To bound the excess risk $\mathcal{L}_t(\widehat{f}_t)-\mathcal{L}_t(f^*_t)$, we first decompose it into two terms. The first term $\sup_{f\in\mathcal{F}} |\mathcal{L}_t(f)-\widehat{\mathcal{L}}_t(f) |$, represents the stochastic error (also known as the empirical process), which is the primary focus in statistical learning.  The second term $\inf_{f\in\mathcal{F}}(\mathcal{L}_t(f)-\mathcal{L}_t(f^*_t))$, represents the approximation error, which is deterministic 
%  in nature. Both components vary with different choices of the width and length of $\mathcal{F}$.  Once bounds for stochastic and approximation errors are available, we can immediately obtain the upper bound for the excess risk and one-step Bellman error. } 

\noindent{\it Bounding the stochastic error.} 
For short, 
write
\[
\cZ:= \sup_{f\in\mathcal{F}}~ |\mathcal{L}(f)-\widehat{\mathcal{L}}(f) |. 
\]
We start by deriving the upper bound on $\EE[\cZ]$ 
in terms of Rademacher complexity that 
\begin{align}\label{equ: symmetrization}\nonumber
\EE[\cZ] &= 
\mathbb{E} \Big[ \sup_{f \in \cF}~  \Big|\frac{1}{n}\sum_{i=1}^n \b(\mathbb{E}[\rho_{\tau_i} (y_i, f(x_i,\tau_i))]- \rho_{\tau_i} (y_i, f(x_i,\tau_i))\b)\Big|\Big] \nonumber\\
 & \stackrel{(i)}{\leq} 2 ~ \mathbb{E} \Big[ \sup_{f \in \cF} ~ \Big|\frac{1}{n} \sum_{i=1}^n \sigma_i \rho_{\tau_i} (y_i, f(x_i,\tau_i)) \Big|\Big]\nonumber\\
& \stackrel{(ii)}{\leq}4~  \mathbb{E} \Big[\sup_{f \in \cF} ~ \Big|\frac{1}{n}\sum_{i=1}^n \sigma_if(x_i,\tau_i)\Big|\Big]\nonumber\\
& = 4~  \mathbb{E}[\widehat{\mathcal{R}}(\mathcal{F}) ]= 4~ \mathcal{R}(\mathcal{F}),
\end{align}
where
%where  $\ell\circ\mathcal{F}=\{(x,y,\tau)\to \rho_{\tau}(y-f(x,\tau)): f\in \mathcal{F}\}$,  %and  $l_f(S_i):=\rho_{\tau_i}(Y_i-f(X_i,\tau_i))$, and the first inequality (i) follows from the Jensen’s inequality. (ii) and (iii) hold since the distribution remains unchanged after multiplying   Rademacher random variables. 
(i) follows from the symmetrization inequality and (ii) follows from Ledoux–Talagrand contraction inequality in Lemma \ref{lem: Ledoux_Talagrand_contraction} and the fact that $\rho_{\tau_i} (y_i, f(x_i,\tau_i))$ is 1-Lipschitz continuous in $f$ for each $i\in [n]$.

% We proceed to derive an upper bound on the Rademacher complexity.  
% %$\mathcal{F}((X,\tau)_1^{n})$, and $\mathcal{F}_{n}$ is a subset of $\mathbb{R}^{n}$. 
% Recall that $\mathcal{F}=\mathcal{F}(W,L)$ represents the function class of ReLU network (Definition 4.1) with $W$-width and $L$-length, 
% then we have
Define the projection of $\mathcal{F}$ on the set $\{(x_i,\tau_i)\}_{i=1}^n $ as 
\begin{align} \label{def_F_n}
 {\mathcal{F}}_n :   = \left\{\left(f(x_1, \tau_1), \ldots,f(x_n, \tau_n)\right)^\T:~  f\in\mathcal{F}\right\}.  
\end{align}
%Write $\mathcal{F}_{n}= \mathcal{F}((X,\tau)_1^{n})$ for brevity.
Then it follows from Dudley’s theorem that
\begin{align}\label{bd_Dudley}
\widehat{\mathcal{R}}(\mathcal{F}) = F~ \mathbb{E} \B[\sup_{f\in\mathcal{F}}~ \Big|\frac{1}{n}\sum_{i=1}^n \sigma_i f(x_i,\tau_i)/F \Big| \B] \le \frac{CF}{\sqrt{n}} \int_{0}^{1}\sqrt{\log N(t,\mathcal{F}_n/F,\|\cdot\|_{\infty})} ~ dt.
\end{align}
The covering number regarding $\cF$
is closely related to the pseudo-dimension $\operatorname{Pdim}(\mathcal{F})$  that is defined in \cref{sec: Supporting lemmas}.
By applying Lemma \ref{pseudo},  whenever $n\ge \operatorname{Pdim}(\mathcal{F})\ge 1$, we have
\begin{align*}
N\left(\delta, \mathcal{F}_n, \|\cdot\|_{\infty} \right) \leq\left(\frac{ e F  n}{\delta  \operatorname{Pdim}(\mathcal{F})}\right)^{\operatorname{Pdim}(\mathcal{F})} \quad \text{for any $\delta>0$,}
\end{align*}
which, together with \eqref{bd_Dudley}, yields
\begin{align*}
\widehat{\mathcal{R}}(\mathcal{F})
~\le ~& \frac{CF\sqrt{\operatorname{Pdim}(\mathcal{F})}}{\sqrt{n}} \int_{0}^{1}\sqrt{\log\B(\frac{ e  n}{t \operatorname{Pdim}(\mathcal{F})}\B)}~  dt\\
\le~ & \frac{CF\sqrt{\operatorname{Pdim}(\mathcal{F})\log n}}{\sqrt{n}} \int_{0}^{1}\sqrt{1+\log\left(\frac{ e}{t }\right)}~  dt. 
\end{align*}
% note that  the covering number is closely related to the pseudo-dimension $\operatorname{Pdim}(\mathcal{F})$  that is defined in section \ref{sec: Supporting lemmas}, and according to  
It is easy to see that $\sqrt{1+\log(e/t)} $ is integrable on $(0,1)$ so that
\begin{align*}
\widehat{\mathcal{R}}(\mathcal{F})
~\le ~\frac{CF\sqrt{\operatorname{Pdim}(\mathcal{F})\log n}}{\sqrt{n}}.
\end{align*}
Combining  with 
the upper bound on $\operatorname{Pdim}(\mathcal{F})$ from  Lemma \ref{pseudo_bound} results in 
\begin{align*}
\widehat{\mathcal{R}}(\mathcal{F}) \le CF\sqrt{\frac{W^2 L^2\log(W^2L) \log n}{n}}, 
\end{align*}
which further leads to the same upper bound on  ${\mathcal{R}}(\mathcal{F})$. 
% \begin{align*}
% {\mathcal{R}}(\mathcal{F}) = \mathbb{E} [\widehat{\mathcal{R}}(\mathcal{F})]\le CF\sqrt{\frac{W^2 L^2\log(W^2L) \log n}{n}}.
% \end{align*}
Plugging  this bound into \eqref{equ: symmetrization} gives 
\begin{align*}
\EE[\cZ] \leq C F\sqrt{\frac{W^2 L^2\log(W^2L) \log n}{n}}.
\end{align*}

 We proceed to bound the deviation of  $\cZ$  from its expectation. To this end, 
by noting that the quantile loss is  1-Lipschitz continuous, it is easy to verify that a replacement of $(x_i,\tau_i)$ in $\cZ$ with arbitrary $(\bar{x}_i,\bar{\tau}_i)$  change $\cZ$ by at most $2F/n$. 
Then,  invoking the bounded difference inequality in Lemma \ref{lem: bounded differences} yields that with probability at least $1-2n^{-1}$, 
\begin{align*}
\cZ - \mathbb{E}[\cZ ] \le  F \sqrt{\frac{2\log n}{n}}.
\end{align*}
By combining the bounds for both $\EE[\cZ]$  and  $\cZ - \mathbb{E}[\cZ ]$, we conclude that with probability at least $1-2n^{-1}$, 
\begin{align}\label{bd_esti_error}
\cZ= \sup_{f\in\mathcal{F}}~ |\mathcal{L}(f)-\widehat{\mathcal{L}}(f) | \le C F\sqrt{\frac{W^2 L^2\log(W^2L) \log n}{n}}.
\end{align}
Together with the error decomposition in \eqref{excess_risk},  the upper bound \eqref{equ: upper_bound_B1} follows. 
\vspace{0.1in}

\noindent{\it Bounding the approximation error.}
Observe that
\begin{equation} \label{appro_ineq}
  \begin{aligned}
\inf_{f\in\mathcal{F}}~ \left(\mathcal{L}(f)-\mathcal{L}(f^{*})\right)
&= \inf_{f\in\mathcal{F}}~ \mathbb{E}\big[\rho_\tau(Y-f(X,\tau))-\rho_\tau(Y-f^{*}(X,\tau))\big]\\
& \stackrel{(i)}{\leq} \inf_{f\in\mathcal{F}}~\mathbb{E}\left|f(X,\tau)-f^{*}(X,\tau)\right| = \inf_{f\in\mathcal{F}}~ \|f-f^{*}\|_{1},
\end{aligned}  
\end{equation}
where (i) holds due to the 1-Lipschitz continuity of the quantile loss.

By applying Lemma \ref{lem: jiao_approximation}, for sufficiently large $U, V\in \mathbb{N}^+ $, if we choose
the width  and length  as 
\begin{align}\label{choice_W_and_L}
 W=\mathcal{O}\left((s+1)^2 d^{s+1} U \log U\right),\ L=\mathcal{O}\left((s+1)^2 V \log V\right),   
\end{align}
then the following inequality holds. 
\begin{align*}
    \inf_{f \in \mathcal{F}}~ 
    \|f^{*}-f\|^2_{2} \leq  C~  H^2(s+1)^4 d^{2s+(\beta \vee 1)}(U V)^{-\frac{4\beta}{d}}.
\end{align*}
Using the inequality $\|f^{*}-f\|_{1}\le \|f^{*}-f\|_{2}$ and together with  \eqref{appro_ineq} yields
\begin{align}\label{equ: approximation error bound}
\inf_{f\in\mathcal{F}}~ (\mathcal{L}(f)-\mathcal{L}(f^{*}))\leq C H(s+1)^2 d^{s+\frac{\beta \vee 1}{2}}(U V)^{-\frac{2\beta}{d}}.
\end{align}
\vspace{0.03in}

Finally, summarizing the bounds of \eqref{excess_risk}, \eqref{bd_esti_error} and \eqref{equ: approximation error bound}  gives
\begin{equation}\label{final_bound}
 \begin{aligned}
&\mathcal{L}(\widehat{f})-\mathcal{L}(f^{*})
\\& \leq C\sqrt{\frac{W^2 L^2\log(W^2L) \log n}{n}}+C H(s+1)^2 d^{s+\frac{\beta \vee 1}{2}}(U V)^{-\frac{2\beta}{d}}\\
&\lesssim  \frac{(s+1)^4d^{s+1}(UV)(\log U\log V)\sqrt{\log(W^2L)}\sqrt{ \log n}}{\sqrt{n}}+ (s+1)^2 d^{s+\frac{\beta \vee 1}{2}}(U V)^{-\frac{2\beta}{d}} && \text{by \eqref{choice_W_and_L}.}
 \end{aligned}   
\end{equation}
By choosing $UV=\lfloor n^{\frac{d}{2d+4\beta}} \rfloor$, it is easy to deduce
\[
\max\{\log U,~ \log V\}\le \log n^{\frac{d}{2d+4\beta}}\le \log n
\]
and for sufficiently large $n$, 
\[
\log( W^2L)\leq \log( (s+1)^6d^{2s+2}U^4V^4)\leq  C ~ \log n.
\]
% where $C'$ depends on 
Combining  the above two bounds implies that the rightmost side of \eqref{final_bound} can be bounded by
\begin{align*}
& \frac{(s+1)^4d^{s+1}(UV) (\log n)^3}{\sqrt{n}}+ (s+1)^2 d^{s+\frac{\beta \vee 1}{2}}(U V)^{-\frac{2\beta}{d}}\\
&\le  2 (s+1)^4 d^{s+(\frac{\beta }{2}\vee 1)} (\log n)^3 n^{-\frac{\beta}{2\beta+d}} && \text{by $UV=\lfloor n^{\frac{d}{2d+4\beta}} \rfloor$, }
 \end{align*}
hence concluding the proof of \cref{thm: bd_ slow_appendix}. 
\end{proof}

\subsection{Proof of Theorem 4.9}
%For return distribution $\widehat{\eta}_{t} =\widehat{\mathcal{T}}_t^{\pi}\widehat{\eta}_{t-1}$, its Bayes optimal solution $\mathcal{T}^{\pi}\widehat{\eta}_{t-1}$ with its corresponding quantile function $f^*_t$ at each update $t\in[T]$.
\begin{theorem}[One-step Bellman error]\label{one step bellman error appendix}
Suppose Assumptions 4.4 and 4.5 are satisfied. For each $t\in[T]$, with the same choice of Length $L$ and width $W$ of the neural network $\mathcal{F}$ as that in Theorem \ref{thm: bd_ slow_appendix},  when $n$ is sufficiently large, with probability at least $1-2n^{-1}$,  the one-step Bellman error has upper bound that
\begin{align*}
\overline{\mathcal{W}}_{1}(\widehat{\eta}_{t},\mathcal{T}^{\pi}\widehat{\eta}_{t-1})\leq C(s+1)^2 d^{\frac{s}{2}+(\frac{\beta }{4}\vee \frac{1 }{2})}(\log n)^{\frac{3}{2}} n^{-\frac{\beta}{4\beta+2d}}
\end{align*}
where $C$ is a constant independent of $s,\beta, d, n$. 
\end{theorem}

% \begin{lemma}[One-step Bellman error]\label{one step bellman error appendix}
%  Suppose Assumption \ref{assum: Bellman completeness}, and \ref{assum: convexity} are satisfied. Then, for any $c>0$, with probability at least $1-2n^{-1}$, the one-step Bellman error satisfies
% \begin{align*}
% %\underset{s,a\sim \mu}{\mathbb{E}} d_1^{2}\big(\widehat{\eta}_{t}(s,a),\mathcal{T}^{\pi}\widehat{\eta}_{t-1}(s,a)\big) 
% \overline{\mathcal{W}}_{1,\mu}(\widehat{\eta}_{t},\mathcal{T}^{\pi}\widehat{\eta}_{t-1}) \leq   C\Big(\frac{ W^2 L^2\log (W^2L)\log n}{n}\Big)^{\frac{1}{2}} + \frac{2}{c_0}\inf_{f\in\mathcal{F}}(\mathcal{L}(f)-\mathcal{L}(f^*_t)),
% \end{align*}
% where $C$ is a constant independent of  $W,L,c_0, n$. Furthermore, with the same choice of Length $L$ and width $W$ of the neural network $\mathcal{F}$ as that in Theorem \ref{thm: bd_ slow_appendix},  
% if we choose $UV=\lfloor n^{\frac{d}{2d+4\beta}} \rfloor$, then for any $c>0$,  the one-step Bellman error has upper bound that
% \begin{align*}
% \overline{\mathcal{W}}_{1,\mu}(\widehat{\eta}_{t},\mathcal{T}^{\pi}\widehat{\eta}_{t-1}) \leq C(s+1)^4 d^{s+\frac{\beta }{2}\vee 1} (\log n)^3 n^{-\frac{\beta}{2\beta+d}}
% \end{align*}
% with probability at least $1-2n^{-1}$,
% where $C$ is a constant independent of $s,\beta, d, n$. 
% \end{lemma}
\begin{proof}
This statement follows a straightforward application of Lemma 4.7 and Theorem \ref{thm: bd_ slow_appendix}.
\end{proof}

\begin{theorem}\label{sub-optimality_slow_app}
Suppose Assumptions 4.3, 4.4, and 4.5 are satisfied. For each $t\in[T]$, with the same choice of Length $L$ and width $W$ of the neural network $\mathcal{F}$ as that in Theorem \ref{thm: bd_ slow_appendix},   
if  $T = \mathcal{O}(\xi\log N)$ with some constant $\xi\ge 0$, when $N$ is sufficiently large, with probability at least $1- c N^{-1}(\log N)^{2}$,   the sub-optimality of $\boldsymbol{\widehat{\eta}}_{T}$ has an upper bound that
\begin{align*}  
% \wh \varepsilon_{1,t}= 
\mathcal{W}_1(\boldsymbol{\eta}^{\pi},\boldsymbol{\widehat{\eta}}_{T})\leq \frac{C C_{\mu}^{\frac{1}{2}}}{(1-\gamma)^{\frac{3}{2}}}(\log N)^{\frac{5}{2}} N^{-\frac{\beta}{4\beta+2d}} +  \frac{C_{F,R} }{(1-\gamma)^{\frac{3}{2}}}N^{\frac{\xi\log \gamma}{2}},
\end{align*}
where $0<\gamma<1$, and $c, C$ are  constants  independent of $C_\mu, N,\gamma$ and $C_{F,R} = R_{max}+F$.  
\end{theorem}

\begin{proof}
Applying the sub-optimality decomposition in  Lemma 4.6 with $p=1$ yields
 \begin{align*}
\mathcal{W}_1(\boldsymbol{\eta}^{\pi},\boldsymbol{\widehat{\eta}}_{T})\leq \frac{2 C_{\mu}^{\frac{1}{2}}}{(1-\gamma)^{\frac{3}{2}}}\max_{0<t\leq T}\widehat{\varepsilon}_{1, t}  +  \frac{C_{F,R}}{(1-\gamma)^{\frac{3}{2}}}  N^{\frac{\xi\log \gamma}{2}}.
\end{align*}
Here, recall that $\widehat{\varepsilon}_{1, t}=\overline{\mathcal{W}}_{1} (\widehat{\eta}_{t},\mathcal{T}^{\pi}\widehat{\eta}_{t-1})$. Applying  Theorem \ref{one step bellman error appendix} and taking a union bound yields
\[
\max_{0<t\leq T}\widehat{\varepsilon}_{1, t} ~\le~ 
C(s+1)^2 d^{\frac{s}{2} + (\frac{\beta }{4}\vee \frac{1 }{2})}(\log n)^{\frac{3}{2}} n^{-\frac{\beta}{4\beta+2d}}
\]
with probability at least $1-2  Tn^{-1}$. Recall that the data is divided into $T$ parts of equal size, then we have $n=N/T$. 
Choosing $T=\mathcal{O}(\xi\log N)$ completes the proof.
\end{proof}

\section{Proof of Section 4.2}\label{sec: D}
% \redtext{In the rest of this appendix, we sometimes omit the reference distribution 
% $\nu$ in the notation for the $L_p(\nu)$-norm, denoting $\|\cdot\|_{p,\nu}$ as $\|\cdot\|_{p}$ for simplicity.}
In this section, we aim to establish a fast rate for excess risk using the standard localization technique developed by \cite{bartlett2005local}. 
We begin by briefly revisiting the notion of sub-root functions and introducing the local Rademacher complexity, which is applied to establish a fast rate.

\subsection{Sub Root Functions  and Local Rademacher complexity}\label{sec_sub_root_rade}
\begin{definition}[Sub root function]
A function $\psi: [0, \i) \to [0, \i) $ is sub-root 
if it is nonnegative, non-decreasing, and if $\delta \to \psi(\delta)/\delta$ is non-increasing for $\delta>0$. 
\end{definition}

A basic property of the sub-root function, provided in Lemma 3.2 of \cite{bartlett2005local}, states as follows. 

\begin{lemma}\label{lem_ineq_sub_root}
If $\psi: [0, \i) \to [0, \i)$ is a nontrivial\footnote{Not constant function $\psi\equiv 0$. } sub-root function, then it is continuous on $[0,\i) $ and the equation $\psi(\delta)= \delta $ has a unique positive solution $\delta^*$,  referred to as the fixed point.
Moreover, for all $\delta>0$,  $\delta \ge\psi(\delta)$ if and only if $\delta\ge \delta^*$. 
\end{lemma}
% As a counterpart to the Rademacher complexity defined in Section \ref{sec: C}, in order to derive a fast rate, we introduce
Recall the definition of Rademacher complexity  $\cR(\cdot)$  from \eqref{def_Rade}. 
In order to derive a fast rate, a useful approach relies on the ingredient of local Rademacher complexity. Specifically,  consider independent samples $\nu_1,\ldots, \nu_n$ drawn from a distribution  distribution $\varrho$ supported on $\RR^d$. For a generic function class $\cH$ from $\RR^d$ to $\RR$, let $T(\cdot): \cH \to  \RR^+$ be a functional (possibly random). 
Then the local Rademacher complexity typically takes the form of
\[
\cR \left(\{h \in  \cH, ~ T(h) \le \delta\}\right) =\EE
\Big[\sup_{h \in\cH, ~ T(h )\le \delta}~ \Big|\frac{1}{n}\sum_{i=1}^n \sigma_i  h(\nu_i)\Big| \Big].
\]
% where $\EE_{\sigma}$ is taken with respect to the Rademacher random variables $\sigma_1,\ldots, \sigma_n$. 
% Accordingly, the local Rademacher complexity is given by
% \[
% \cR \left(\{h \in  \cH, ~ T(h) \le \delta\}\right) 
% =\EE_{\nu}[\emloRa\left(\{h \in  \cH, ~ T(h) \le \delta\}\right)], 
% \]
% where $\EE_{\nu}$ is taken with respect to $\nu_1,\ldots, \nu_n$. 

In particular, our target in this paper is to investigate the statistical learning rate of ReLU network $\cF:=\cF(W,L)$ by using the samples $(x_1,\tau_1), \ldots, (x_n,\tau_n)$.
To this end, we take  $T(\cdot): \cF \to  \RR^+$ as  $f \to T(f)= \|f\|_2^2$. 
Then  the local Rademacher complexity  used in our analysis is specified as
\begin{align}\label{local_rade_popu}
\cR_\cF(\delta):= \cR\left(\{f \in  \cF, ~ \|f\|_2^2 \le \delta\}\right) =\EE
\Big[\sup_{f \in\cF, ~ \|f\|_2^2
\le \delta} ~ \Big|\frac{1}{n}\sum_{i=1}^n \sigma_i 
f(x_i,\tau_i)\Big| \Big].
\end{align}

With a slight modification of the proof 
of Lemma 3.4 in \cite{bartlett2005local}, 
it can be readily verified that $F \cR_\cF(\delta)$ is a sub-root function. 
A crucial quantity for determining the learning rate is the fixed point $\delta_\star: = \delta_\star(W,L)$ corresponding to $F \cR_\cF(\delta)$, that is
the  solution to 
\begin{align}\label{radius_1}
F \cR_\cF(\delta) =  \delta.
\end{align}
However, directly solving \eqref{radius_1} is significantly difficult. 
To address this problem, an alternative approach is to consider the empirical form of $\cR_\cF(\delta)$, given by
\begin{align}\label{equ: local_rade_empi}
\emloRa_\cF(\delta):= 
\emloRa\left(\{f \in  \cF, ~ \|f\|_n^2 \le \delta\}\right) 
=\EE_{\sigma}\Big[\sup_{f \in\cF,~  \|f\|_n^2
\le \delta} ~ \Big|\frac{1}{n}\sum_{i=1}^n \sigma_i 
f(x_i,\tau_i)\Big| \Big],
\end{align}
where $\EE_{\sigma}$ is taken exclusively  with respect to the Rademacher random variables $\sigma_1,\ldots, \sigma_n$. 
% $F \emloRa_\cF(\delta)$ 
% is also a sub-root function in view of Lemma 3.4 of \cite{bartlett2005local}.
By relating $\emloRa_\cF(\delta)$ to the covering number of 
$\cF_n$ defined in \eqref{def_F_n}, we can derive a solution to
\begin{align}\label{radius_2}
F \emloRa_\cF(\delta) \le \delta,
\end{align}
which serves as a sharp upper bound on  $\delta_\star$. 
% The argument 
% Applying  \cref{lem_bd_delta_star}  yields an upper bound on  $\delta_\star$ in terms of the solution to \eqref{radius_2}.
The detailed arguments are deferred to \cref{bd_delta_star}.

%\wang{where $\|f\|_n^2:=\frac{1}{n}\sum_{i=1}^nf^2(x_i,\tau_i)$ denotes the squared empirical ${L}_2(P_n)$-norm associated with the samples $\{(x_i,\tau_i)\}_{i=1}^n$. Move to notation section. }
% The local Rademacher complexity has been well-studied by \cite{bartlett2005local}, which is at the core of the learning theory for deriving a fast rate. 
% 

% Moreover, as shown by \cite{bartlett2005local}, both $\mathcal{R}(\delta, \mathcal{H})$ and $\widehat{\mathcal{R}}(\delta, \mathcal{H})$ are sub-root functions for any bounded function class $\mathcal{H}$. 

\subsection{Proof of Main Results}
To derive a fast rate, we require the existence of a function  $f_\mathcal{F}\in\mathcal{F}$  such that $\|f_\mathcal{F}-f^*\|^2_{2}=\inf_{f\in\mathcal{F} }\|f-f^{*}\|^2_{2}$, where $f_{\mathcal{F}}$ represents the projection of $f^*$ onto the network space $\mathcal{F}$. For brevity, we sometimes write  $\mathcal{A} = \|f_\mathcal{F}-f^*\|^2_{2}$ in this section.
%In the proof of this section,  we oftentimes write $\mathcal{A}= \|f_\mathcal{F}-f^*\|^2_{2}$ for short. 

Recall that $\wh f$ is obtained from $\cF =\cF(W,L)$. 
Define a sequence of  networks $\mathcal{F}_k: =\mathcal{F}(2^{k+1} W,L), k=1,2,3$ and write $\cF_0= \cF$. It is trivial to check that for each 
$k=0,1,2$,  any $f,g \in \mathcal{F}_k$ and $\alpha\in [0,1]$,  $\alpha f\pm (1-\alpha)g\in \mathcal{F}_{k+1}$.

For any $f$, define the function $\ell_f: \RR\times\RR^{d-1}\times (0,1) \to \RR$ as
\begin{align}\label{def_ell_f}
    \ell_f(y,x,\tau): = \rho_\tau\big(y, f(x,\tau) \big)- \rho_\tau\big(y, f_\cF(x,\tau) \big)
\end{align}
so that
\begin{align*}
\mathcal{L}(f)-\mathcal{L}(f_\cF)-\big(\widehat{\mathcal{L}}({f})-\widehat{\mathcal{L}}(f_\mathcal{F})\big)=  \frac{1}{n}\sum_{i=1}^n\big (\EE[ \ell_f ( y_i, x_i, \tau_i ) ]-\ell_f ( y_i, x_i, \tau_i)  \big). 
\end{align*}
For any fixed $\delta>0$, define the  random variable $Z_f$  as
\begin{align}\label{def_Z_f}
    Z_f : = \frac{\mathcal{L}(f)-\mathcal{L}(f_\cF)-\big(\widehat{\mathcal{L}}({f})-\widehat{\mathcal{L}}(f_\mathcal{F})\big)}{\delta^{-1/2}\|f-f_{\mathcal{F}}\|_2+1}. 
\end{align}
Further define the empirical process $\cZ: = \sup_{f\in\mathcal{F}_1}~  Z_f$.
% Our goal is to prove that $\sup_{f\in\mathcal{F}}{Z}\le C_F(
% \sqrt{\delta} \|f-f_\mathcal{F}\|_2+\delta)$ with high probability, which is equivalent to proving that
Define the auxiliary event
\begin{align*}
\mathcal{E}(\delta):=
\Big\{\big((y_1, x_1,\tau_1), \ldots(y_n, x_n,\tau_n)\big) : ~ \cZ\le CF^{-1} ~ \delta \Big\}. 
\end{align*}

\begin{theorem}\label{thm_bd_fast_appen}
Suppose that Assumptions 4.4 and 4.10 are satisfied.
With probability at least $1- c\exp(-  W^2 L^2 \log (W^2 L)\log n)$, the excess risk satisfies
\begin{align}\label{equ: upper_bound_C1}
\mathcal{L}(\widehat{f})-\mathcal{L}(f^*) ~  \leq  ~ C~   \frac{{W^2 L^2 \log (W^2 L)\log n}}{{n}}
%&+C~ \sqrt{\frac{W^2L^2\log(W^2L)\log n}{n}\mathcal{\mathcal{A}}}
+2~ \mathcal{A},
\end{align}
 where $c, C$ are some constants independent of $ W,L,n $.
 Furthermore, with the same choice of Length $L$ and width $W$ of the neural network $\mathcal{F}$ as that in Theorem \ref{thm: bd_ slow_appendix}, when $n$ is sufficiently large,    with probability at least $1-c \exp(- n^{\frac{2d}{2d+4\beta}}\log n)$,  the excess risk has upper bound that
$$
\mathcal{L}(\widehat{f})-\mathcal{L}(f^*)~  \leq  ~ C~  (s+1)^8 d^{2s+ (\beta \vee 2)} (\log n)^6n^{-\frac{2\beta}{2\beta+d}},
$$
where $c, C$ are  some  constants independent of $s,\beta, d, n$. 
\end{theorem}

%\begin{remark}
%For analysis simplification, we set the number of quantile levels $m=1$, thus the empirical loss degrades to $\mathcal{\widehat{L}}(f)=\frac{1}{n}\sum_{i=1}^n \rho_{\tau_i}\left(y_i-f(X_i, \tau_i)\right)$. Our proof consists of three steps. In the first step, we decompose the excess risk into two terms, stochastic error and approximation error. Then, in the second step, we show the stochastic error term can be bounded by the complexity of $\mathcal{F}$. Finally, we measure the approximation error of the function class $\mathcal{F}(W,L)$ to the H\"{o}lder class $\mathcal{G}(\mathcal{D}, \beta, H)$, and make a balance between the two types of error. \change{this remark is still needed here?}
%\end{remark}

%\begin{remark}
 % By simply integrating  the tail bound, the excess risk in expectation can be upper bounded as
 %   \begin{align*}
  %   \mathbb{E}[\mathcal{L}(\widehat{f})-\mathcal{L}(f^{*})]
 %    &=\int_{0}^\infty \mathbb{P}(\mathcal{L}(\widehat{f})-\mathcal{L}(f^{*})> t)dt\\
  %   &=\int_{0}^{\psi(n)} \mathbb{P}(\mathcal{L}(\widehat{f})-\mathcal{L}(f^{*})> t)dt +\int_{\psi(n)}^\infty \mathbb{P}(\mathcal{L}(\widehat{f})-\mathcal{L}(f^{*})> t)dt \\
  % &  \leq    \psi(n)+\int_{\psi(n)}^\infty c_1\exp(-c_2(s+1)^8d^{2s+2}n^{\frac{\beta}{\beta+d}}) dt
   % \end{align*}
%\end{remark}

\begin{proof}
Fix any $\delta>0$ and $W,  L \in \mathbb{N}^+ $  satisfying
\begin{equation}\label{condi_for_b}
\Delta(\delta): = C^2_0~  \left( \delta + \|f_{\mathcal{F}}-f^*\|_2^2\right)\le b_n, 
\end{equation}
where $b_n$ is the constant depending on $n$ introduced in Assumption 4.10  
and $C_0$ is some sufficiently large constant  depending  only by  $F, c_0, c'_0$.  
Observe that the  excess risk can be decomposed as
\begin{align}\label{equ: error decomposition local rademacher}\nonumber
\mathcal{L}(\widehat{f})-\mathcal{L}(f^*)
&=\mathcal{L}(\widehat{f})-\mathcal{L}(f_\mathcal{F})+\mathcal{L}(f_\mathcal{F})-\mathcal{L}(f^*)\nonumber\\
&\le \mathcal{L}(\widehat{f})-\mathcal{L}(f_\mathcal{F})+c_0'\|f_\mathcal{F} - f^{*}\|^2_{2}\nonumber\\
&\le \mathcal{L}(\widehat{f})-\mathcal{L}(f_\mathcal{F})-\big(\widehat{\mathcal{L}}(\widehat{f})-\widehat{\mathcal{L}}(f_\mathcal{F})\big) +c_0' \|f_\mathcal{F} - f^{*}\|^2_{2} ,
  %  \|\widehat{f} - f^{*}\|^2_{2,\mu} &\leq 2\|\widehat{f} - f_\mathcal{F}\|^2_{2,\mu} + 2\|f_\mathcal{F} - f^{*}\|^2_{2,\mu}\\\nonumber
   % &\leq 2 c_0 \big( \mathcal{L}(\widehat{f})-\mathcal{L}(f_\mathcal{F})\big) + 2\|f_\mathcal{F} - f^{*}\|^2_{2,\mu}\\\nonumber
   % &= 2 c_0 \big( \mathcal{L}(\widehat{f})-\mathcal{L}(f_\mathcal{F}) + \widehat{\mathcal{L}}(f_\mathcal{F})- \widehat{\mathcal{L}}(\widehat{f}) + \widehat{\mathcal{L}}(\widehat{f}) - \widehat{\mathcal{L}}(f_\mathcal{F}) \big) + 2\|f_\mathcal{F} - f^{*}\|^2_{2,\mu}\\
  %  &\leq 2 c_0 \big( \mathcal{L}(\widehat{f})-\mathcal{L}(f_\mathcal{F}) +  \widehat{\mathcal{L}}(f_\mathcal{F})- \widehat{\mathcal{L}}(\widehat{f}) \big) + 2\|f_\mathcal{F} - f^{*}\|^2_{2,\mu},
\end{align}
where
%we assume there exists $f_\mathcal{F}\in \mathcal{F}$ such that $f_\mathcal{F}= \arg\min_{f\in\mathcal{F}}\|f-f^{*}\|^2_{2,\mu}$, 
the first inequality holds by the local 
smoothness condition in Assumption 4.10 and the last inequality follows from the fact that $\widehat{\mathcal{L}}(f_\mathcal{F})\geq\widehat{\mathcal{L}}(\widehat{f})$ by 
the feasibility of $f_\mathcal{F}$ and the optimality of $\widehat{f}$. 

We proceed to separately bound each term appearing in the rightmost side of \eqref{equ: error decomposition local rademacher}:  bounding the stochastic error $\mathcal{L}(\widehat{f})-\mathcal{L}(f_\mathcal{F})-(\widehat{\mathcal{L}}(\widehat{f})-\widehat{\mathcal{L}}(f_\mathcal{F}))$ and bounding the approximation error $\|f_\mathcal{F}-f^*\|^2_2$. \vspace{0.15in}

\noindent\textit{Bounding the stochastic error.}
% Meanwhile, from \cref{}
% the local Rademacher complexity of  $\mathcal{F}^*$ remains unchanged up to a constant.  This nice property helps establish a fast rate with the ReLU neural network. 
%Further define $\mathcal{F}_2: =\mathcal{F}(4W,L)$ so that $f-f_{\mathcal{F}}\in \mathcal{F}_2$ for any $f\in \mathcal{F}_1$.
%We  consider the localized function space 
%around the projection $f_{\mathcal{F}}\in \mathcal{F}\subset \mathcal{F}^*$, that is defined as  $\mathcal{F}^*(\delta)=\{f\in \mathcal{F}^*: \|f\|_2\le \delta \}$.  
% The corresponding empirical version is defined as  the smallest solution $\widehat{\delta}_n$ of the inequality
% \begin{align}\label{radius_2}
% F\widehat{\mathcal{R}}(\delta,\mathcal{F}^{**})\le \delta^2.
% \end{align}
To bound the stochastic error, we need the following lemmas. 
Lemma \ref{basic_lem_1} states a bound from below on the occurrence probability of $\mathcal{E}(\delta)$, while Lemma \ref{basic_lem_2} states an upper bound on the stochastic error on $\mathcal{E}(\delta)$. 
%in terms of $\delta$ satisfying the inequality 
%$F \cR_{\cF_2}(\delta) \le  \delta$ and  approximation error $\|f_{\mathcal{F}}-f^*\|^2_2$. 
The proof of Lemmas  \ref{basic_lem_1} and \ref{basic_lem_2} can be found in  Section \ref{sec: 11.1}.

\begin{lemma}\label{basic_lem_1}
Fix any $\delta>0$ satisfying 
$F \cR_{\cF_3}
(\delta) \le  \delta$
% the inequality \eqref{radius_1} 
and 
$t>0$ satisfying $0<t \le \delta/F^2$.
Then the event $\mathcal{E}(\delta)$ occurs with probability at least  $1- e^{- nt}$, that is
\begin{align*}
    \mathbb{P}(\mathcal{E}(\delta))\ge 1-  e^{- nt}. 
\end{align*}
\end{lemma}

\begin{lemma}\label{basic_lem_2}
Fix any $\delta>0$  and 
$W,  L \in \mathbb{N}^+ $ satisfying $F \cR_{\cF_3}
(\delta) \le  \delta$
and \eqref{condi_for_b}. 
Then on  the event $\mathcal{E}(\delta)$, one has
\begin{align}\label{excess_risk_fast}
\mathcal{L}(\widehat{f})-\mathcal{L}(f_\mathcal{F})- \big(\widehat{\mathcal{L}}(\widehat{f})-\widehat{\mathcal{L}}(f_\mathcal{F})\big)\le 
C~ \big(\delta+\sqrt{\delta}~ \|f_\cF-f^*\|_2 \big). 
\end{align}
where $C$ is some constant depending only on $ F,c_0,c'_0$. 
\end{lemma}

Recall the definition of $\delta_\star=\delta_\star(W,L )$ in \cref{sec_sub_root_rade}. 
Let $\wt \delta_\star := \delta_\star (8 W, L)$ be the fixed point corresponding to the sub-root function $F \cR_{\cF_3}
(\delta)$. 
By  applying \cref{lem_bd_delta_star},
with probability at least $1-2 \exp({- C W^2 L^2 \log (W^2 L)\log n})$, one has the upper bound
\begin{align}\label{equ: delta_n}
\wt \delta_\star  \le C F^2~ \frac{{W^2 L^2 \log (W^2 L)\log n}}{{n}}=:\delta_n. 
\end{align}
Furthermore, invoking \cref{lem_ineq_sub_root} ensures that
$\delta_n$ satisfies the inequality $F \cR_{\cF_3}
(\delta) \le  \delta$.

% Then, $\widetilde{\delta}_n$ also satisfy the inequality \eqref{radius_1}.
Then applying Lemmas \ref{basic_lem_1} and \ref{basic_lem_2} with $\delta = \delta_n$ and $t=\delta_n/F^2$
conclude that 
 once 
\eqref{condi_for_b} holds, with probability at least $1-3\exp({- CW^2 L^2 \log (W^2 L)\log n})$, we have
\begin{align*}
\mathcal{L}(\widehat{f})-\mathcal{L}(f_\mathcal{F})- \big(\widehat{\mathcal{L}}(\widehat{f})-\widehat{\mathcal{L}}(f_\mathcal{F})\big)
\le  C ~ \big( \delta_n  + \sqrt{\delta_n}\|f_{\mathcal{F}}-f^*\|_2\big),
\end{align*}
which, together with the decomposition \eqref{equ: error decomposition local rademacher}, 
completes the proof for the upper bound \eqref{equ: upper_bound_C1}.
% Rearranging terms and noting  that $\mathcal{A} = \|f_{\mathcal{F}}-f^*\|_2^2$ gives
% \begin{align*}
%  &  \mathcal{L}(\widehat{f})-\mathcal{L}(f_\mathcal{F})- \big(\widehat{\mathcal{L}}(\widehat{f})-\widehat{\mathcal{L}}(f_\mathcal{F})\big)\\
% &\le  C \left((F+1)\frac{{W^2 L^2 \log (W^2 L)\log n}}{{n}}+ \sqrt{\frac{{W^2 L^2 \log (W^2 L)\log n}}{{n}}\mathcal{A}}\right). 
% \end{align*}
% we have
% \begin{align}\label{excess_risk_fast_1}\nonumber
% &\mathcal{L}(\widehat{f})-\mathcal{L}(f_\mathcal{F})- \big(\widehat{\mathcal{L}}(\widehat{f})-\widehat{\mathcal{L}}(f_\mathcal{F})\big)\\ \nonumber
% &\le 
% CC_F\big((C_F+1){\delta}^2 +{\delta}\|f_{\mathcal{F}}-f^*\|_2\big)\\\nonumber
% &\lesssim  (F+1)\frac{{W^2 L^2 \log (W^2 L)\log n}}{{n}} +\sqrt{\frac{{W^2 L^2 \log (W^2 L)\log n}}{{n}}\mathcal{A}}\\
% &\lesssim  (F+1)\frac{{W^2 L^2 \log (W^2 L)\log n}}{{n}}+\sqrt{\frac{{W^2 L^2 \log (W^2 L)\log n}}{{n}}\mathcal{A}},
% \end{align}
% where we recall $\mathcal{A}= \|f_{\mathcal{F}}-f^*\|_2^2$.
% Note that the requirement in \eqref{req_2} is less restrictive than that in \eqref{req_1}, so it is allowed to choose $t=C\frac{{W^2 L^2 \log (W^2 L)\log n}}{{n}}$.

\vspace{0.15in}

\noindent\textit{Bounding the approximation error.}
According to Lemma \ref{lem: jiao_approximation}, for sufficiently large $U, V\in \mathbb{N}^+  $, by choosing the  width and  length
\[
W=\mathcal{O}\left((s+1)^2 d^{s+1} U \log U\right), \  L=\mathcal{O}\left((s+1)^2 V \log V\right), 
\]
the approximation error can be bounded as
%\begin{align*}
 %   \inf_{f \in \mathcal{F}(W, L)}\|f^{*}-f\|^2_{2,\mu} \leq C H^2(s+1)^4 d^{2s+\beta \vee 1}(U V)^{-4 \beta / d}.
%\end{align*}
%Then, we have
\begin{align}\label{equ: approximation error bound 2}
\|f_\mathcal{F} - f^{*}\|^2_{2} = \inf_{f\in\mathcal{F}} \|f-f^{*}\|^2_{2} \leq C H^2(s+1)^4 d^{2s+\beta \vee 1}(U V)^{-\frac{4\beta}{d}}. 
\end{align}
\vspace{0.05in}
%where (i) holds due to the 1-Lipschitz continuity of the quantile loss. Then, plugging \eqref{equ: approximation error bound 2} into \eqref{equ: upper_bound_C1} yields \eqref{equ: upper_bound_C2}. 
% Combining with \eqref{equ: delta_n},  with probability at least $1-c\exp(- W^2 L^2 \log (W^2 L)\log n)$, the stochastic error satisfies
%  \begin{align*}
% &\mathcal{L}(\widehat{f})-\mathcal{L}(f_\mathcal{F})-\big(\widehat{\mathcal{L}}(\widehat{f})-\widehat{\mathcal{L}}(f_\mathcal{F})\big)\\
% &\le  C \frac{{W^2 L^2 \log (W^2 L)\log n}}{{n}}+C \sqrt{\frac{W^2L^2\log(W^2L)\log n}{n}}H(s+1)^2 d^{s+\frac{\beta \vee 1}{2}}(U V)^{-2 \beta / d}.
%  \end{align*}

Finally, 
%we balance stochastic and approximation errors by selecting proper $U$ and $V$. Specifically,  
by setting $UV=\lfloor n^{\frac{d}{2d+4\beta}} \rfloor$ and repeating the  similar arguments of proving  Theorem \ref{thm: bd_ slow_appendix},  for sufficiently large $n$,  one can deduce
\[
\frac{{W^2 L^2 \log (W^2 L)\log n}}{{n}} \le C~ (s+1)^8d^{2s+2} (\log n)^6 n^{-\frac{2\beta}{2\beta+d}}
\]
and 
\[
\cA =\|f_\mathcal{F} - f^{*}\|^2_{2} \le C(s+1)^4 d^{2s+(\beta \vee 1)}n^{-\frac{2\beta}{2\beta+d}}.
\]
By combining the bounds for the stochastic  and approximation errors, together with 
the error decomposition \eqref{equ: error decomposition local rademacher}, the second statement in \cref{thm_bd_fast_appen} follows. 
% and 
% we have
% \begin{align}\label{final_upp}\nonumber
% &\mathcal{L}(\widehat{f})-\mathcal{L}(f^*)\\\nonumber
%  &\leq \mathcal{L}(\widehat{f})-\mathcal{L}(f_\mathcal{F})-\big(\widehat{\mathcal{L}}(\widehat{f})-\widehat{\mathcal{L}}(f_\mathcal{F})\big) +c_0'\inf_{f\in\mathcal{F}}\|f - f^{*}\|^2_{2}\\\nonumber
% &\stackrel{(i)}{\lesssim}\frac{(s+1)^8d^{2s+2}(UV)^2 (\log n)^6}{{n}}+\frac{(s+1)^6d^{2s+1+\frac{\beta \vee 1}{2}} (UV)^{\frac{d-2\beta}{d}} (\log n)^3}{\sqrt{n}}+(s+1)^4 d^{2s+\beta \vee 1}(U V)^{-\frac{4\beta}{d}}\nonumber\\
% &\lesssim  (s+1)^8 d^{2s+\beta \vee 2} (\log n)^6n^{-\frac{2\beta}{2\beta+d}},
% \end{align}
% where (i) follows a similar argument as in the proof of  Theorem \ref{thm: bd_ slow_appendix} for sufficiently large $n$.
% Note that the established upper bound \eqref{final_upp} holds with a probability greater than
% $ 1-c\exp(-n^{\frac{2d}{2d+4\beta}}\log n).$
\end{proof}

\begin{remark}[{Discussion on the radius $b_n$ in Assumption 4.10}]
By inspecting the proof of  \cref{thm_bd_fast_appen}, 
% $b_n$  in  Assumption 4.10 is allowed to decay to zero as the sample size goes to infinity.
% Specifically, 
after choosing appropriate  $\delta$ and  $W, L$ in \eqref{condi_for_b}, $b_n = C_b (\log n)^6  n^{-\frac{2\beta}{ 2\beta+ d}}$ satisfies \eqref{condi_for_b}, 
% only needs to be larger than
% \begin{align*}
%     C^2_0~  \left( \delta_n  +  \|f_{\mathcal{F}}-f^*\|_2^2\right) =C_b  n^{-\frac{2\beta}{ 2\beta+ d}}, 
% \end{align*}  
where $C_b$ is some constant depending on $F, s, \beta, d, c_0, c'_0$. 
% Here, as we will see in Step (iii), $\|f_\mathcal{F}-f^{*}\|^2_{2}$ can be arbitrarily small as the neural network $\mathcal{F}$ becomes wider or deeper. Consequently, it suffices to require that the smoothness condition holds for $f$ in the neighborhood of the true target function $f^*$.  
\end{remark}

\begin{remark}
% The error decomposition \eqref{equ: error decomposition local rademacher} captures the stochastic error rooted in the statistical complexity and the approximation error stemming from model misspecification. 
% \redtext{Specifically, 
% the first term $\mathcal{L}(\widehat{f})-\mathcal{L}(f_\mathcal{F})-(\widehat{\mathcal{L}}(\widehat{f})-\widehat{\mathcal{L}}(f_\mathcal{F}))$, represents the stochastic error %(\change{stochastic error}), 
% which is the primary focus in statistical learning.  The second term $\inf_{f\in\mathcal{F}}\|f - f^{*}\|^2_{2}$, approximation error, is deterministic and depends on the choice of function class $\mathcal{F}$. repeat?} 
The local smoothness condition in Assumption 4.10 allows the approximation error quantified by $\inf_{f\in\mathcal{F}}\|f-f^*\|^2_{2}$, which differs from the slow rate established in Section \ref{sec: C}, where approximation error is quantified by $\inf_{f\in\mathcal{F}}\|f-f^*\|_{1}$. 
% More importantly, in this decomposition,  the stochastic error only involves the functions belonging to the estimation space $\mathcal{F}$, allowing us to control it using the statistical complexity of $\mathcal{F}$.
\end{remark}

\subsubsection{Proof of Theorem 4.12}
\begin{theorem}\label{thm_supp_412}
Suppose Assumptions 4.4 and 4.10 are satisfied. For  each $t\in[T]$, with the same choice of Length $L$ and width $W$ of the neural network $\mathcal{F}$ as that in Theorem \ref{thm: bd_ slow_appendix}, 
when $n$ is sufficiently large, with probability at least $1-c\exp(-n^{\frac{2d}{2d+4\beta}}\log n)$,  the one-step Bellman error has upper bound that
\begin{align*}
\overline{\mathcal{W}}_{1}(\widehat{\eta}_{t},\mathcal{T}^{\pi}\widehat{\eta}_{t-1})\leq C (s+1)^4 d^{s+ (\frac{\beta}{2} \vee 1)} (\log n)^3n^{-\frac{\beta}{2\beta+d}},
\end{align*}
where $c, C$ are constants independent of $s,\beta, d, n$. 
\end{theorem}
\begin{proof}
By inspecting the proof of Theorem \ref{thm_bd_fast_appen}, we find that under the  same conditions as in  Theorem \ref{thm_bd_fast_appen}, it holds with probability at least $1-c \exp(- n^{\frac{2d}{2d+4\beta}}\log n)$ that
\[
\|\wh f_t - f^*_t \|_2^2  ~  \leq  ~ b_n.
\]
This allows us to apply Assumption 4.10 to deduce the inequality in Lemma 4.7 through the same argument in the proof of Lemma 4.7. 
Then the statement in \cref{thm_supp_412}  follows a straightforward application of Theorem \ref{thm_bd_fast_appen}.
\end{proof}

\begin{theorem}
Suppose Assumptions  4.3, 4.4, and 4.10 are satisfied. For   each $t\in[T]$, with the same choice of Length $L$ and width $W$ of the neural network $\mathcal{F}$ as that in Theorem \ref{thm: bd_ slow_appendix}, 
if $T=\mathcal{O}(\xi\log N)$  with some constant  $\xi\ge 0$, when $N$ is sufficiently large, with probability at least $1-c\log N \exp(-(N/\log N)^{\frac{2d}{2d+4\beta}})$, 
the sub-optimality of $\boldsymbol{\widehat{\eta}}_{T}$ has an upper bound that
\begin{align*}  
\mathcal{W}_1(\boldsymbol{\eta}^{\pi},\boldsymbol{\widehat{\eta}}_{T})\leq \frac{C C_{\mu}^{\frac{1}{2}}}{(1-\gamma)^{\frac{3}{2}}}(\log N)^4 N^{-\frac{\beta}{2\beta+d}} +  \frac{C_{F,R} }{(1-\gamma)^{\frac{3}{2}}} N^{\frac{\xi\log \gamma}{2}},
\end{align*}
where $0<\gamma<1$, and $c, C$ are constants  independent of $C_\mu, N,\gamma$.  
\end{theorem}
\begin{proof}
The proof follows similar arguments to those in the proof of Theorem \ref{sub-optimality_slow_app}, so we omit it here.
\end{proof}

\subsection{Proof of Technical Lemmas for Theorem 
\ref{thm_bd_fast_appen}}\label{sec: 11.1}
\begin{proof}[Proof of Lemma \ref{basic_lem_1}]
Recall the definition of $\cZ$ from \eqref{def_Z_f}. 
We divide the proof into two steps: first bounding the expectation of $\mathcal{Z}$ and then bounding the deviation of $\mathcal{Z}$ from its expectation. \vspace{0.15in}

\noindent{\it Step (i):  Bounding $\mathbb{E}[\mathcal{Z}]$.}
 To bound the expectation of the empirical process $\mathcal{Z}$, we apply similar arguments as those in the proof of the bound \eqref{equ: symmetrization}, yielding
 \begin{equation}\label{sym_eq}
   \begin{aligned}
\mathbb{E}[\mathcal{Z}] = \EE\b[\sup_{f\in \cF_1 }~ Z_f \b] & =\mathbb{E} \Big[\sup_{f\in\mathcal{F}_1} ~ 
\frac{\frac{1}{n}\sum_{i=1}^n \b( \EE[\ell_f (y_i,x_i,\tau_i)]- \ell_f (y_i,x_i,\tau_i) \b)}{\delta^{-1/2}\|f-f_{\mathcal{F}}\|_2+1}\Big]\\
&\le4~ \EE \Big[\sup_{f\in {\mathcal{F}}_1} ~ 
\frac{
\big|\frac{1}{n}\sum_{i=1}^n \sigma_i \big(f(x_i,\tau_i)-f_\mathcal{F}(x_i,\tau_i)\big)\big|}{\delta^{-1/2}\|f-f_{\mathcal{F}}\|_2+1}
\Big].
\end{aligned}  
 \end{equation}
%where $l_f(S_i):=\rho_{\tau_i}(Y_i-f(X_i,\tau_i))$, (i) follows from the symmetrization inequality and  (ii) follows from Ledoux–Talagrand contraction inequality (Lemma \ref{lem: Ledoux_Talagrand_contraction}) and the fact that $l_f$ is 1-Lipschitz continuity. 
For any $f\in \mathcal{F}_1$, let
\[
\alpha_f  = \frac{1}{\delta^{-1/2}\|f-f_{\mathcal{F}}\|_2+1} 
\]
and define the convex combination of $f$ and  $f_{\mathcal{F}}  $ as
\[
g_f = \alpha_f ~  f + (1-\alpha_f) ~ f_{\mathcal{F}} \in \cF_2.  
\]
% Recall the definition of $\cF_1$, which includes the convex combinations of elements $f, f_{\mathcal{F}} \in \mathcal{F}$, it follows that  $\widetilde{f}\in \mathcal{F}_1$.
It is easy to verify that
\begin{align}\label{eq_1w}
\|g_f - f_{\mathcal{F}} \|_2=  \frac{\|f-f_{\mathcal{F}}\|_2}{\delta^{-1/2}\|f-f_{\mathcal{F}}\|_2+1}\le \sqrt{\delta}
\end{align}
so that together with \eqref{sym_eq} gives
\begin{align}\label{bound_expe}\nonumber
\mathbb{E}[\mathcal{Z}]
% {\le}&4~ \EE\Big[\sup_{f\in {\mathcal{F}}}\frac{
% \big|\frac{1}{n}\sum_{i=1}^n \sigma_i (f(x_i,\tau_i)-f_\mathcal{F}(x_i,\tau_i))\big|}{\delta^{-1/2}\|f-f_{\mathcal{F}}\|_2+1}
% \Big]\nonumber\\
{\le }&4~\EE\Big[\sup_{f\in {\cF_1}} ~ 
\big|\frac{1}{n}\sum_{i=1}^n \sigma_i (g_f (x_i,\tau_i)-f_\mathcal{F}(x_i,\tau_i))\big|
\Big]\nonumber\\
\le&4~\EE\Big[\sup_{f\in \mathcal{F}_2, ~ \|f- f_{\mathcal{F}}\|^2_2\le \delta} ~ \Big|\frac{1}{n}\sum_{i=1}^n \sigma_i(f(x_i,\tau_i)-f_{\mathcal{F}}(x_i,\tau_i))\Big|\Big]\nonumber\\
\le&8~ \EE\Big[\sup_{f\in \mathcal{F}_3, ~ \|f\|^2_2\le \delta} ~ \Big|\frac{1}{n}\sum_{i=1}^n \sigma_if(x_i,\tau_i)\Big|\Big] = 8~ \cR_{\cF_3}(\delta) 
\le \frac{8\delta}{F},
\end{align}
where the last inequality holds due to  $F\cR_{\cF_3}(\delta)\le \delta$. \vspace{0.15in }

\noindent{\it Step (ii): Bounding the deviation of $\mathcal{Z}$ from its expectation.}
% For simplicity, denote $S_i=(x_i,y_i,\tau_i)$.
To proceed, we employ the Talagrand concentration inequality stated in Lemma \ref{Tala_concen}. 
%To apply \cref{Tala_concen}, we need to bound the required first-order and second-order conditions.
Specifically, for each $1\le i\le n$, we observe that
\begin{align*}
{|l_f(y_i,x_i,\tau_i)|}
& =\big |\rho_{\tau_i} (y_i, f(x_i,\tau_i) ) -\rho_{\tau_i} (y_i, f_\cF (x_i,\tau_i) )\big|\\
&\le  |f(x_i,\tau_i) - f_\mathcal{F}(x_i,\tau_i) |\le 2F, %\tag{First-order condition}
\end{align*}
where the second step follows from the fact that $\rho_{\tau_i} (y_i, f(x_i,\tau_i) )$ is 1-Lipschitz continuous in 
$f$ and the last step holds due to  $\|f\|_{\infty}\le F$ for any $f\in \mathcal{F}$. 

Regarding $\EE[\ell_f(y_i,x_i,\tau_i )^2]$, we have 
 \begin{align*}
\EE[\ell_f(y_i,x_i,\tau_i )^2]
%=  \mathbb{E}\Big[\Big(\frac{l_f(x_i,y_i,\tau_i)-l_{f_{\mathcal{F}}}(x_i,y_i,\tau_i)}{\delta^{-1/2}\|f-f_{\mathcal{F}}\|_2+1}\Big)^2\Big]
\le
\mathbb{E}\Big[\Big(\frac{f(x_i,\tau_i)-f^*( x_i,\tau_i)}{\delta^{-1/2}\|f-f_{\mathcal{F}}\|_2+1}\Big)^2\Big]
= \|g_f - f_\mathcal{F}\|_2^2 \overset{\eqref{eq_1w}}{\le}  \delta.  %\tag{Second-order condition} 
\end{align*}
Therefore, for any $t>0$,  applying Talagrand’s concentration in  Lemma \ref{Tala_concen} with $b=4F$ and $V=\delta$
yields the following high-probability upper bound for $\mathcal{Z} $. 
\begin{align}\label{bd_eq_Z}
\mathbb{P}\l(\mathcal{Z}\le  2\mathbb{E}[\mathcal{Z}]+ C(\sqrt{\delta t} +  F t) \r)\ge 1-  e^{-nt}. %\asymp  c_1e^{-c_2 n\delta^2/F},
\end{align}
% where the last step follows from $F\ge1$. 
% \vspace{0.01in }
 
Finally,
combining  the displays 
\eqref{bound_expe}  and \eqref{bd_eq_Z} implies that for any $t>0$, with probability at least $1-e^{-nt}$,
\begin{align*}
\mathcal{Z}\le 2\mathbb{E}[\mathcal{Z}]+ C(\sqrt{\delta t} +  F t) 
\le \frac{16 }{F}\delta + C(\sqrt{\delta t} +  F t)
\end{align*}
For $t>0$
satisfying $t\le \delta /
F^2$, we deduce $\mathcal{Z} 
% \le \frac{8}{F}\delta + C\Big( \frac{1}{F}\delta  + \frac{1}{F} \delta\Big)\lesssim 
\le C ~ \delta/F$, 
hence complete the proof. 
%where the last step holds due to our assumption $F\ge1$ without loss of generality. 
\end{proof}

\begin{proof}[Proof of Lemma \ref{basic_lem_2}]
% On the event $\mathcal{E}(\delta)$, we aim to establish the upper bound on 
% the stochastic error $\mathcal{L}(\widehat{f})-\mathcal{L}(f_\mathcal{F})-(\widehat{\mathcal{L}}(\widehat{f})-\widehat{\mathcal{L}}(f_\mathcal{F}))$. 
% Recall that the critical radius $\delta_n$ is the smallest solution to \eqref{radius_1}.   
In the proof, let us keep track of the explicit dependence of the constant  $F$. 
Below we  prove  \cref{basic_lem_2} with $\Delta(\delta)$ in \eqref{condi_for_b} replaced  by its full version
\begin{align}\label{def_delta}
  \Delta:= \Delta(\delta)=C^2_0 \max\{F^{-1}, 1\}~  \left( F^{-1} \delta + \|f_{\mathcal{F}}-f^*\|_2^2\right).   
\end{align}
%where $C_0$ is some absolute constant. 
Define the function classes 
\begin{equation*}
    \begin{split}
&  \cF(\delta):= \big\{f\in \cF:~  \| f-f_{\cF}\|_2^2\le \Delta\big\}\\
& \cB(\delta):= \big\{f\in \cF_1:~  \| f-f_{\cF}\|_2^2= \Delta\big\}.
    \end{split}
\end{equation*}
If we could show $\wh f\in  \cF(\delta)$, 
% \begin{align}\label{Primary_target}
%     \|\wh f -f_{\cF}\|_2^2\le \Delta,
% \end{align}
then the proof is completed by observing  that on the event $\mathcal{E}(\delta)$, 
\begin{align*}
&\mathcal{L}(\widehat{f})-\mathcal{L}(f_\mathcal{F})-\big(\widehat{\mathcal{L}}(\widehat{f})-\widehat{\mathcal{L}}(f_\mathcal{F})\big)\\
&\le C F^{-1} ~ \big(\sqrt{\delta} \|\widehat{f}-f_\mathcal{F}\|_2+\delta\big)\\
&\le  C F^{-1}~  \big(\sqrt{\delta \Delta}+\delta\big)\\
&=
C F^{-1}~  \left(
C_0  \sqrt{ \max\{F^{-1}, 1\}  \left( F^{-1} \delta + \|f_{\mathcal{F}}-f^*\|_2^2\right)\delta }+\delta \right)\\
&\le 
{C ({C_0} (F^{-\frac{1}{2}}+1)   +1) F^{-1} ~  \left(
(F^{-\frac{1}{2}}+1)\delta+ \|f_\cF-f^*\|_2\sqrt{\delta} \right). }
\end{align*}

% \begin{align}
% & C F^{-1}~  \left(
% C_0  \sqrt{ \max\{F^{-1}, 1\}  \left( F^{-1} \delta + \|f_{\mathcal{F}}-f^*\|_2^2\right)\delta }+\delta \right)\\
% & 
% \le   C F^{-1}~  \left(
% C_0 \sqrt{\max\{F^{-1}, 1\} \delta} 
% (F^{-1/2}\sqrt{\delta}+ \|f_\cF-f^*\|_2)+\delta \right)\\
% & =    C F^{-1}~  \left(
% (C_0 \sqrt{\max\{F^{-1}, 1\}} F^{-1/2} +1)~ \delta
% + C_0 \sqrt{\max\{F^{-1}, 1\}}~  \sqrt{\delta} \|f_\cF-f^*\|_2 \right)\\
% & \le 
% C F^{-1}~  \left(
% (C_0 \sqrt{\max\{F^{-1}, 1\}} +1) ( F^{-1/2}+1)~ \delta
% + C_0 \sqrt{\max\{F^{-1}, 1\}}~  \sqrt{\delta} \|f_\cF-f^*\|_2 \right)\\
% & \le 
% C( C_0 (F^{-1/2}+1) +1) ) F^{-1}~  \left( ( F^{-1/2}+1)~ \delta
% +  \sqrt{\delta} \|f_\cF-f^*\|_2 \right)\\
% \end{align}

To proceed, we focus on proving $\wh f\in  \cF(\delta)$. To this end, by the 
optimality of $\wh f$ and the feasibility of $f_\cF$, we have
\begin{align}\label{optima_f_hat}
   \wh  \cL(\wh f)- \wh \cL(f_\cF)\le 0. 
\end{align}
Then it suffices to show %\redtext{the contrapositive proposition} 
that for any $f  \in  \cF\setminus  \cF(\delta)$, the following inequality holds.
\begin{align}\label{primary_target}
   \wh  \cL( f)- \wh \cL(f_\cF)> 0.
\end{align}
We claim that if the inequality \eqref{primary_target} holds for any $f\in \cB(\delta)$, 
%denoted by $\cB(\cF(\delta) )$\footnote{$\cB(\delta)$ equals $$\big\{f\in \cF: \| f-f_{\cF}\|_2^2= \Delta\big\}.$$}, 
it also holds for any $ f \in \cF \setminus \cF(\delta)$.
To verify this claim, consider any $f\in \cF \setminus \cF(\delta)$ and 
write
\[
\alpha = \frac{\sqrt{\Delta}}{\|f-f_{\cF} \|_2 } >0.
\]
Define  $\wt f$ as
\[
\wt f = \alpha f+(1-\alpha) f_\cF\in \cF_1. 
\]
It is easy to show that $\|\wt f- f_\cF\|_2^2= \Delta$, i.e. $\wt f \in  \cB(\delta) $.
% By writing
% we find that $\wt f= \alpha f+(1-\alpha) f_\cF$.
Substituting $\wt f$ into \eqref{primary_target} yields
\begin{equation}\label{comp_wt_f}
\begin{aligned}
      \wh  \cL(\wt f)- \wh \cL(f_\cF) =   &     \wh  \cL(\alpha f+(1-\alpha) f_\cF)- \wh \cL(f_\cF) \\
      \le &\alpha   \wh  \cL( f) + (1-\alpha) \wh  \cL( f_\cF )   - \wh \cL(f_\cF) =\alpha \big (      \wh  \cL(f)- \wh \cL(f_\cF)\big), 
\end{aligned}
\end{equation}
where the second step follows from Jensen's inequality and the convexity of the quantile loss function. 
Since $f$ is chosen arbitrarily, and in view of  \eqref{comp_wt_f}, it suffices to prove \eqref{primary_target} holds for all 
$f\in \cB(\delta)$.

To this end, for any $f\in \cB(\delta)$, observe that
\begin{equation}\label{relax_convex}
\begin{aligned}
&c_0 \|{f}-f_{\mathcal{F}}\|_2^2 - (2c'_0+2c_0) \|f_{\mathcal{F}}-f^*\|_2^2\\
&\stackrel{(i)}{\le} 2c_0 \|{f}-f^*\|_2^2 - 2c'_0\|f_{\mathcal{F}}-f^*\|_2^2
\\
&\stackrel{(ii)}{\le} 2\big(\mathcal{L}({f})-\mathcal{L}(f^*)-\big({\mathcal{L}}(f_{\mathcal{F}})-{\mathcal{L}}(f^*)\big)\big)\\
&=2\big(\mathcal{L}({f})-\mathcal{L}(f_\mathcal{F})\big),
\end{aligned}
\end{equation}
where (i) uses the elementary inequality that $(a+b)^2\le 2a^2+2b^2$,
(ii) follows from Assumption 4.10. Note that \eqref{relax_convex} essentially states a relaxed version of strong convexity condition around 
$f_\cF$. Indeed,  if $f_\cF=f^*$, i.e., without misspecification, then the leftmost second term of \eqref{relax_convex} vanishes.
% , and \eqref{relax_convex} is reduced to
% \[
% c_0 \|{f}-f_{\mathcal{F}}\|_2^2 \le 2\big(\mathcal{L}({f})-\mathcal{L}(f_\mathcal{F})\big). 
% \]
% % relates the squared error to the excess risk with quantile loss, which is crucial for proving the upper bound on the excess risk with quantile loss.

By using  \eqref{relax_convex}, 
on the event $\cE(\delta)$, we have
\begin{align*}
 \wh \cL (f_\cF)-   \wh \cL (f)  =&   \wh \cL (f_\cF)-   \wh \cL (f) - \big( \cL (f_\cF)-    \cL (f) \big) + \big( \cL (f_\cF)-    \cL (f) \big) \\
 \le &  CF^{-1} ~( \sqrt{\delta}  \|{f}-f_\mathcal{F}\|_2+\delta) -  \frac{c_0}{2} \|f- f_\cF\|_2^2 + (c_0+c'_0)\|f_\cF-f^*\|_2^2\\
 \le &  CF^{-1}~   (\sqrt{\delta\Delta} +\delta) -  \frac{c_0}{2} \Delta + (c_0+c'_0)\|f_\cF-f^*\|_2^2.
\end{align*}
% which, together with the elementary inequality 
% $\sqrt{ab}\le a+b$, implies
% \[
%  \wh \cL (f_\cF)-   \wh \cL (f) \le (C_F+1)\delta+\Delta + (c_0+c'_0)\|f_\cF-f^*\|_2^2-  \frac{c_0}{2} \Delta.
% \]
Define 
\[
\varphi(x): = \frac{c_0}{2} x^2 - 
CF^{-1} \sqrt{\delta} 
x- CF^{-1} \delta - (c'_0+ c_0)\|f_{\mathcal{F}}-f^*\|_2^2,
\]
and 
\[
U(C_0): = \max\left\{\frac{c_0 C_0^2 \max\{F^{-1}, 1\} }{2}  - C, ~ \frac{c_0 C_0^2 \max\{F^{-1}, 1\} }{2}- (c'_0+ c_0)\right\}. 
\]
Observe
that 
\begin{align*}
&\varphi\big(\sqrt{\Delta}\big)= 
\frac{c_0 C_0^2 \max\{F^{-1}, 1\} }{2} ( F^{-1} \delta + 
\|f_{\mathcal{F}}-f^*\|_2^2) -CF^{-1} \delta-   (c'_0+ c_0)\|f_{\mathcal{F}}-f^*\|_2^2- \\
&\hspace{2.8 in} C C_0 \max\{F^{-\frac{1}{2}}, 1\} F^{-1} \sqrt{\delta}
\sqrt{F^{-1} \delta +  \|f_{\mathcal{F}}-f^*\|_2^2} \\
&\le U(C_0)
\big( F^{-1}  \delta + \|f_{\mathcal{F}}-f^*\|_2^2\big)
-CC_0\max\{F^{-\frac{1}{2}}, 1\}  F^{- 1}\sqrt{\delta}
\sqrt{F^{-1} \delta +  \|f_{\mathcal{F}}-f^*\|_2^2}\\
&= C_0 
\left(\frac{U(C_0)}{C_0} \sqrt{ F^{-1}\delta + \|f_{\mathcal{F}}-f^*\|_2^2}-C \max\{F^{-\frac{1}{2}}, 1\} F^{-1}   \sqrt{\delta}\right)  \sqrt{F^{-1} \delta +  \|f_{\mathcal{F}}-f^*\|_2^2}.
 \end{align*}
%  where the last step holds for sufficiently large $C$ such that $\frac{c_0{C}^2-1}{C}\ge 1$. 
% Since  $\varphi(x_0)\le 0$, then we have $x_1\ge x_0$. By hiding the constants $c_0,c'_0$ and rearranging the terms, we proved \eqref{105_eq1}. 
Through simple algebraic manipulations, it is easy to verify  
that  $U(C_0)$ grows to infinity as $C_0\to \i$.
Then  choosing $C_0$ in \eqref{def_delta} to be sufficiently large ensures $\varphi\big(\sqrt{\Delta}\big)>0$, i.e.
\[
 \wh \cL (f_\cF)-   \wh \cL (f)<0,
\]
hence complete the proof.

\end{proof}

\subsection{Bounding $\delta_{\star}= \delta_{\star}(W,L )$} \label{bd_delta_star}

\begin{lemma}\label{lem_bd_delta_star}
Fix any width $W$ and length $L$ of the ReLU network $\cF:=\cF(W,L)$.
With probability at least $1- 2 \exp({- C W^2 L^2 \log (W^2 L)\log n})$, one has
\begin{align*}
\delta_\star \le C F^2~ \frac{{W^2 L^2 \log (W^2 L)\log n}}{{n}}. 
\end{align*}
\end{lemma}
\begin{proof}
In view of \cref{lem_bd_delta_star_2}, 
it suffices to find a solution to \eqref{radius_2}, which is the upper bound of  $\delta_\star$. 
Invoking Corollary 14.3 in \citet{wainwright2019high} implies that any  positive 
solution 
$\wt \delta$
to 
\begin{align}\label{cover_ineq}
\frac{64}{\sqrt{n}} ~ g(\delta):= \frac{64}{\sqrt{n}}\int_{\frac{\delta}{2F}}^{\sqrt{\delta}}
\sqrt{\log N(t,\cF,\|\cdot\|_{n})}
\d t\le \frac{\delta}{F}
\end{align}
must satisfy \eqref{radius_2}. The proof is then divided into two steps. \vspace{0.1in}

\noindent
{\it Step (i):  Bounding $g(\delta)$.}
For any $\delta>0$, 
applying the inequality $N(\delta,\cF,\|\cdot\|_{n})\le N(\delta,\cF_n, \|\cdot\|_{\infty})$ gives
\begin{align*}
g(\delta) {\le} 
\int_{0}^{\sqrt{\delta}} 
\sqrt{\log N(t,\cF_n ,\|\cdot\|_{\infty})} \d t.
\end{align*}
Applying Lemma \ref{pseudo} further implies
\begin{align}\label{bd_g_delta}
g(\delta)
\le\int_{0}^{\sqrt{\delta}}\bigg({{\operatorname{Pdim}(\cF)} \log\Big(\frac{ e F n}{t \operatorname{Pdim}(\cF)}\Big)}\bigg)^{\frac{1}{2}} \d t\le \sqrt{\delta \operatorname{Pdim}(\cF)}\int_{0}^{1} \bigg(\log\Big(\frac{ e F n}{t\sqrt{\delta}}\Big)\bigg)^{\frac{1}{2}} \d t,
\end{align}
where the second inequality
follows from the change of variable in integration and the fact that $\operatorname{Pdim}(\cF)\ge 1$.
By choosing $\delta$ satisfying 
\begin{align}\label{requ_delta}
n\delta  \ge F^2, 
\end{align}
we have
\begin{align*}
\int_{0}^{1} \bigg(\log\Big(\frac{ e F n}{t\sqrt{\delta}}\Big)\bigg)^{\frac{1}{2}}  \d t = &
\int_{0}^{1}\bigg(\log n^{3/2} +\log\Big(\frac{ eF}{t \sqrt{\delta n} }\Big)\bigg)^{\frac{1}{2}} \d t\\
\le &
\sqrt{\log n} \int_{0}^{1}\left(\frac{3}{2} +\log\Big(\frac{ e}{t  }\Big)\right)^{\frac{1}{2}} \d t.
\end{align*}
Since  the rightmost integral is a constant, by writing 
\[
C'= \int_{0}^{1}\left(\frac{3}{2} +\log\Big(\frac{ e}{t  }\Big)\right)^{\frac{1}{2}} \d t
\]
and combining with \eqref{bd_g_delta}, we arrive at
\[
g(\delta)\le C' {\sqrt{\delta \operatorname{Pdim}(\cF)\log n}}.
\]
% &\stackrel{(ii)}{\le} 
% \frac{64\delta\sqrt{\operatorname{Pdim}(\mathcal{F}^{**})}}{\sqrt{n}} \int_{0}^{1}\sqrt{\frac{3}{2}\log n+\log\Big(\frac{e}{t}}\Big)dt\\
% &\le
% C_1\frac{64\delta\sqrt{\operatorname{Pdim}(\mathcal{F}^{**})\log n}}{\sqrt{n}},
%$\delta=C\frac{\sqrt{\operatorname{Pdim}(\mathcal{F})\log n}}{{\sqrt{n}}}$.
Finally, applying Lemma \ref{pseudo_bound} implies that for $n\ge \operatorname{Pdim}(\mathcal{F})$ and $\delta$ satisfying \eqref{requ_delta}, the following inequality holds. 
\begin{align}\label{final_bd_g}
g(\delta) 
\le C' {\sqrt{\delta W^2 L^2 \log (W^2 L)\log n}}.
\end{align}

\noindent
{\it Step (ii): Bounding $\delta_\star$.}
In view of \eqref{cover_ineq} and   \eqref{final_bd_g}, we find that if $\delta$ satisfies both \eqref{requ_delta}  and the inequality 
\begin{align}\label{ineq_delta}
C'~ \frac{\sqrt{\delta W^2 L^2 \log (W^2 L)\log n}}{\sqrt{n}}\le \frac{\delta}{F}, 
\end{align}
it must satisfy \eqref{radius_2}.
By solving \eqref{ineq_delta}, we conclude that  
\begin{align*}
\wt {\delta} =  CF^2~ \frac{{W^2 L^2 \log (W^2 L)\log n}}{{n}}
\end{align*}
satisfies \eqref{radius_2}. 
Applying Lemma \ref{lem_bd_delta_star_2} with $\delta=\wt \delta$
%implies that for any 
% \begin{align}\label{req_1}
%  0<t\le \frac{\wt {\delta}}{F^2} =
%  C\frac{{W^2 L^2 \log (W^2 L)\log n}}{{n}},   
% \end{align}
% with probability at least $1-e^{-nt}$, we have $\delta_\star 
% \le
% C \wt {\delta}$, which 
concludes the statement of \cref{lem_bd_delta_star}. 
\end{proof}

The following lemma states that $\delta_\star$ can be bounded 
by any solution to  \eqref{radius_2} (up to a constant).

\begin{lemma}\label{lem_bd_delta_star_2}
Fix any  $\delta>0$ satisfying the inequality \eqref{radius_2} and any width $W$ and length $L$ of the ReLU network $\cF :=\cF(W,L)$.
For any $t>0$ satisfying $0<t\le \delta/F^2$,   it holds with probability at least $1- 2 e^{-nt}$ that 
$\delta_\star 
\le C~ \delta$. 
\end{lemma}

\begin{proof}
Recall the definitions of 
$\mathcal{R}_\cF (\delta)$ from \eqref{local_rade_popu} and $\wh\cR_\cF (\delta)$  from \eqref{equ: local_rade_empi}. 
% \begin{align*}
% \mathcal{R}(\delta,\mathcal{F}^{**})= \mathbb{E}_{X,\tau}\mathbb{E}_{\sigma}\Big[\sup_{f\in\mathcal{F}^{**}, \|f\|_2\le \delta}\Big|\frac{1}{n}\sum_{i=1}^n \sigma_i f(x_i,\tau_i) \Big| \Big],
% \end{align*}
% and
% \begin{align*}
% \widehat{\mathcal{R}}(\delta,\mathcal{F}^{**})=\mathbb{E}_{\sigma}\Big[\sup_{f\in\mathcal{F}^{**}, \|f\|_n
% \le \delta}\Big|\frac{1}{n}\sum_{i=1}^n \sigma_if(x_i,\tau_i)\Big| \Big].
% \end{align*} 
Further, define an intermediate random variable as 
\begin{align}\label{def_wt_R}
\wt \cR_\cF (\delta):=
\EE_ {\sigma}\Big[\sup_{f\in\cF, ~ 
\|f\|_2^2  \le \delta}~ \Big|\frac{1}{n}\sum_{i=1}^n \sigma_i f(x_i,\tau_i)\Big| \Big].
\end{align}
By taking expectation on both sides of the above inequality, we observe that $\cR_\cF (\delta)= \EE [\wt \cR_\cF (\delta) ]$. 
Recall that $\|f\|_\infty\le F$ for every $f\in \cF$.
Applying Theorem 16 in \cite{boucheron2003concentration} implies that for any $t>0$, with  probability at least $1-e^{-nt}$, we have
\begin{align} \label{bd_by_wt_R}
\cR_\cF (\delta)\le \wt \cR_\cF (\delta)+\sqrt{F t~ 
\cR_\cF (\delta) }. 
\end{align}
Furthermore, an application of \cref{lem_subset} implies that for any $\delta$ such that
\begin{align}\label{requrie_delta}
\delta \ge 20 F~  \cR_\cF (\delta)+11F^2 t,
\end{align}
with probability at least $1-e^{-nt}$, we have
\[
\wt \cR_\cF (\delta)\le \wh \cR_\cF ( {2}\delta). 
\]
Together with \eqref{bd_by_wt_R} yields
\begin{align*}
\cR_\cF(\delta)\le \wh \cR_\cF ({2}\delta)+
\sqrt{Ft~ \cR_\cF (\delta) }.
\end{align*} 
By adding and multiplying terms, for $\delta$ satisfying \eqref{requrie_delta},
we have
\begin{align}\label{eq_1121}
20F ~ 
{\cR_\cF} (\delta)+ 11F^2t\le 20F~ \wh\cR_\cF ({2}\delta)+11F^2t+
20 F\sqrt{Ft ~ \mathcal{R}_\cF(\delta) } 
\end{align}
% The remaining proof is based on the concept of a sub-root function. 
% We are ready to proceed with our proof by utilizing the sub-root function. 
Define the  functions
\begin{equation*}
\begin{split}
\psi_1(\delta)&=20F~\cR_\cF(\delta)+11F^2t, \\
\psi_2(\delta)& =20F~ \widehat\cR_\cF( {2}{\delta})+ 11F^2t+F\sqrt{20 {\delta} t}.
\end{split}
\end{equation*}
It is trivial to check that both $\psi_1$ and $\psi_2$ are sub-root. 
Let $\delta_1$ be the fixed point of  $\psi_1$.
% , that is the solution to
% \begin{align*}
% 10F\cR_\cF(\delta,\cF)+11F^2t=\delta^2.
% \end{align*}
Since $\delta_\star = F \cR_\cF(\delta_\star) \le \psi_1(\delta_\star)$,   
invoking \cref{lem_ineq_sub_root} gives  $\delta_\star \le \delta_1$.
% if $\delta_\star>\delta_1$, by the monotonicity of $\psi_1(\delta)/\sqrt{\delta}$, 
% we have
% \begin{align*}
% \frac{\psi_1(\delta_\star)}{\sqrt{\delta_\star}}\le     \frac{\psi_1(\delta_1)}{\sqrt{\delta_1}}  = \sqrt{\delta_1}, 
% \end{align*}
% implying $\psi_1(\delta_\star) \le\sqrt{\delta_1 \delta_\star}$. 
% It follows that
% \begin{align*}
% \delta_\star = F \mathcal{R}(\delta_\star ,\cF) \le & \frac{1}{10}(10F  \mathcal{R}( \delta_n,\cF) +11F^2t)\\
% = & \frac{1}{10}\psi_1(\delta_n)\le \frac{1}{10}\delta_n \tilde{\delta}_n.
% \end{align*}
% Rearranging yields that $\delta^2_n\le \frac{1}{100}\tilde{\delta}_n^2$, which contradicts with the assumption $\delta_n>\tilde{\delta}_n$. Then, we conclude that $\delta_n^2\le \tilde{\delta}_n^2$.

Below, we aim to bound $\delta_1 $ by $C\delta$, where  $\delta>0$ satisfies the conditions in \cref{lem_bd_delta_star_2}.
To this end, applying \eqref{eq_1121} with  $\delta= \delta_1$
yields
\begin{align*}
\delta_1
= 20 F~ 
\cR_\cF(\delta_1)+ 11F^2t
\le& 20 F~ \widehat\cR_\cF({2}\delta_1)+11F^2t+
20 F\sqrt{Ft~ \cR_\cF(\delta_1) }\\
\le& 20 F ~ \widehat\cR_\cF( {2}\delta_1)+ 11F^2t + F \sqrt{20 \delta_1 t}= \psi_2(\delta_1),
\end{align*}
where the third step holds 
since $20F~\cR_\cF(\delta_1)  \le \psi_1(\delta_1)= \delta_1$.
Due to the fact that ${F\widehat\cR_\cF({\delta})}/ {\sqrt{\delta}}$ is non-increasing  in $\delta$,  for $\delta$ satisfying 
$F \widehat\cR_\cF({\delta})\le \delta$,  we have
\begin{align*}
\frac{F\widehat\cR_\cF({2}{\delta})}{\sqrt{{2}\delta} }\le  \frac{F\widehat\cR_\cF({\delta})}{\sqrt{\delta}}\le \sqrt{\delta}.
\end{align*}
Together with our assumption that $t\le \delta/F^2$, $\psi_2(\delta)$ satisfies 
\begin{align*}
\psi_2(\delta)=20F~ \widehat\cR_\cF({2}{\delta})
+11F^2t+
F\sqrt{20 {\delta} t}\le (\sqrt{20}+11+20\sqrt{2}) \delta. 
\end{align*}
In the case of $\delta_1> \delta$, since $\psi_2(\delta)/\sqrt{\delta}$ is non-increasing  in $\delta$, we have
\begin{align*}
\sqrt{\delta_1}\le   \frac{\psi_2(\delta_1)}{\sqrt{\delta_1}} \le \frac{\psi_2({\delta})}{\sqrt{\delta}} \le (\sqrt{20}+11+20\sqrt{2})\sqrt{ \delta}.
\end{align*}
By combining this result with the case $\delta_1\le \delta$, we deduce $\delta_1\le C {\delta}$.

Putting the pieces together, we conclude the bound in Lemma \ref{lem_bd_delta_star_2}.
\end{proof}

\begin{remark}
Our proof for bounding the critical radius presents a more refined analysis compared to the proof by \cite{farrell2021deep}, which requires a lower bound on $\delta_\star$ that $\delta_\star \ge CF\sqrt{\log n/n}$ (See Section A.2.3 in \cite{farrell2021deep} for more details). This condition may be difficult, or even impossible,  to verify.
% At the same time, we improve the result from Theorem 4.2 of \cite{bartlett2005local} in two aspects: (i) 
% % in \cite{bartlett2005local}, the lower bound on $\delta_n^2$ ($r^*$ in their context) is also needed; 
% (ii) without 
% (iii)
% Compared to \cite{wainwright2019high}, 
\end{remark}

The following lemma concerns local Rademacher complexity and is closely related to Corollary 2.2 in  \cite{bartlett2005local}. The only difference is that \cite{bartlett2005local} defines the local Rademacher complexity in a slightly different form. Since their argument can be repeated for our setting, we omit the proof.

\begin{lemma}\label{lem_subset}
Let $\mathcal{H}$ be a class of functions mapping $\mathcal{X}$ into $[-b,b]$ with $b>0$.
For every $t>0$ and every $\delta$ satisfying
\[
\delta \ge  20b ~\cR \left(\{h \in  \cH, ~ \|h\|_2^2  \le \delta\}\right) + \frac{11 b^{2} t}{n},
\]
we have, with probability at least $1-e^{-t}$,
\[
\bigl\{h\in\mathcal{H}: \|h\|_2^2 \le \delta \bigr\}
\;\subseteq\;
\bigl\{h \in\mathcal{H }:  \|h\|_n^2 \le 2\delta\bigr\}.
\]
\end{lemma}

\subsection{Comparisons with Existing Works}
This study introduces a novel approach to off-policy evaluation (OPE) from a distributional perspective, leveraging quantile process regression. We establish theoretical guarantees for the proposed method, DQPOPE, utilizing ReLU network approximation, and highlight its advantages over standard value-based OPE methods. Additionally, our theoretical results make significant contributions to the learning theory framework, providing insights of independent interest.

In this section, we revisit the existing literature on ReLU network approximations and emphasize the main refinements and distinctions introduced in our work. A summary of these comparisons is presented as follows:

%This work introduces a novel technical method for OPE from the distributional perspective using the quantile process regression.  We establish theoretical guarantees for the proposed method DQPOPE with the ReLU network approximation and highlight its advantages over standard value-based OPE. Furthermore, there are also several contributions in our theoretical results to the learning theory framework, offering insights of independent interest.  In this part, we revisit the existing literature on ReLU network approximations, emphasizing our main refinements and key distinctions compared to existing work. A summary of these comparisons is presented as follows:

\begin{table}[!ht]
\footnotesize
\centering
\caption{Summary of learning rates.}
\label{table: summary}
\begin{tabular}{l |c | c}
\toprule
\text { Rate } & \text { With model misspecification } & \text { Required }\\
\hline
\text{Slow }& \text { / }  & \text { No need} \\
\text{Fast } & \text {No} & \text {Local strong convexity } \\
\text{Fast}& \text {Yes}  & \text { Local strong convexity and  local smoothness }\\
\bottomrule
\end{tabular}
\end{table}
\vspace{0.1in}

\noindent{\bf (i) Comparison with squared loss \citep{SchmidtHieber2020,fan2020theoretical}.}
For squared loss, the working model is typically expressed as $Y=f^*(X)+\varepsilon$, where $\varepsilon$ represents an additive noise term. Given i.i.d. samples $\{(x_i,y_i)\}_{i=1}^n$ from $(X,Y)$,  proofs often rely on the following basic inequality:
%In the squared loss setting, proofs heavily rely on the following basic inequality that 
\begin{align*}
  \frac{1}{n}\sum_{i=1}^n  (\widehat{f}(x_i)-f^*(x_i))^2 \le  \frac{1}{n}\sum_{i=1}^n  (f_\mathcal{F}(x_i)-f^*(x_i))^2  +\frac{2}{n}\sum_{i=1}^n \varepsilon_i (\widehat{f}(x_i)- {f}_\mathcal{F} (x_i)).
\end{align*}
In this inequality, the first term on the right-hand side can be effectively controlled using Lemma \ref{lem: jiao_approximation} while the second term, along with the difference between
$ \frac{1}{n}\sum_{i=1}^n  (\widehat{f}(x_i)-f^*(x_i))^2 $ and its population counterpart, can be addressed using standard learning theory techniques.

However, for quantile loss, such a straightforward decomposition is unavailable, necessitating a fundamentally different analytical approach. By leveraging the local smoothness condition outlined in Assumption 4.10 of the main text, we establish a novel error decomposition, as shown in \eqref{equ: error decomposition local rademacher}.
%In contrast, such an explicit decomposition does not exist for the quantile loss, necessitating a different approach in our analysis. Specifically, by using the local smoothness condition in Assumption 4.10 of the main text, a new error decomposition can be established as shown in \eqref{equ: error decomposition local rademacher}. 

%\textbf{Key Contributions to Bounding Stochastic Error}

We derive three critical results for bounding the stochastic error in the error decomposition:
\begin{enumerate}
    \item Tail Probability Bound (Lemma \ref{basic_lem_1}):
This lemma provides an upper bound for the tail probability of the empirical process, ensuring robust control of the stochastic term.
\item Excess Risk and Squared Error Relation (Lemma \ref{basic_lem_2}):
By relating the excess risk of the quantile loss to the squared error via Assumption 4.10, we establish an upper bound for the stochastic error in terms of $\delta_\star$, the fixed point of $F \cR_{\cF} (\delta)$.
\item Bounding $\delta_\star$ (Lemmas \ref{lem_bd_delta_star} and \ref{lem_bd_delta_star_2}): These lemmas provide explicit upper bounds for $\delta_\star$ in terms of the width and depth of the ReLU network class $\mathcal{F} $.
\end{enumerate}
By appropriately selecting the width and depth of $\mathcal{F} $, we demonstrate that the stochastic error term can achieve a fast convergence rate, underscoring the advantages of our approach in both theory and practice.

%Three key results are derived for bounding the stochastic error appearing in the error decomposition.
%Lemma \ref{basic_lem_1} states the upper bound for the tail probability of the empirical process.
%In Lemma \ref{basic_lem_2}, we first relate the excess risk of quantile loss to the form of squared error by using Assumption 4.10 of the main text. 
%This enables us to establish an upper bound for the stochastic error in terms of $\delta_\star$, defined as the fixed point of $F \cR_{\cF} (\delta)$.
%Lemmas \ref{lem_bd_delta_star} and \ref{lem_bd_delta_star_2}  further provide an upper bound for $\delta_\star$ in terms of the width and depth of $\mathcal{F} $.
%Finally, by appropriately selecting the width and depth of $\mathcal{F} $, the stochastic error term is shown to achieve a fast rate.   
% Notably, the approximation error is well-studied in the literature, allowing us to leverage existing results directly.

\vspace{0.1in}

\noindent
{\bf (ii)  Comparison with other losses  \citep{shen2024, farrell2021deep}.}
While \cite{shen2024} also explores quantile regression estimation using ReLU neural network approximation, their results establish only a slow rate for the expected excess risk. In contrast, through a more precise high-probability analysis, we show that the excess risk for quantile loss can achieve a significantly faster rate, aligning with the minimax optimal rate. Additionally, \cite{farrell2021deep} demonstrates a fast rate for the excess risk of Lipschitz loss with the ReLU network approximation. However, as outlined in \eqref{equ: error decomposition local rademacher}, our decomposition of the total error differs fundamentally.

Unlike \cite{farrell2021deep}, where the true target $f^*$ is incorporated into the stochastic error (see Section A.1 of \cite{farrell2021deep}), we take a more direct approach to analyze the stochastic error, avoiding the assumption that $f-f^*\in\mathcal{F}$. Moreover, our analysis provides a tighter upper bound on the approximation error measured in $\|\cdot\|_2$, which is an improvement over the $\|\cdot\|_{\infty}$ bound used in \cite{farrell2021deep}. 

Additionally, we relax the convexity and smoothness conditions required in \cite{farrell2021deep}, replacing them with localized versions. In this setup, the locality parameter $b_n$ is allowed to depend on the sample size $n$ and can shrink to 0 as $n$ increases.

Finally, in the absence of misspecification (i.e. $f^*\in \mathcal{F}$), the error analysis is well-established, as detailed in Chapter 14 of \cite{wainwright2019high}, where the local smoothness condition is not required. The key innovation of our work lies in the error analysis for excess risk under misspecification, which necessitates the introduction of the local smoothness condition. Notably, achieving a slow rate does not require convexity or smoothness conditions, regardless of whether misspecification is present. We explicitly outline the necessary conditions to achieve either a slow or fast rate under varying scenarios, summarizing these results in Table \ref{table: summary}.

\section{Proof of Section 4.3}\label{sec: proof of proposition 4.13}

\begin{proof}[Proof of Proposition 4.13]

The proof of Proposition 4.13 follows from a straightforward decomposition via the triangle inequality
$$
 \big|\widehat{V}_K-V^\pi\big|\le  \big|\widehat{V}_K-\widehat{V}\big|+ \big|\widehat{V}-V^\pi\big|,
$$
where $\widehat{V}_K = \frac{1}{K}\sum_{k=1}^{K}\widehat{f}_T(s_{0,k},a_{0,k},\tau_k)$,
${V}^\pi= \mathbb{E}_{Z\sim \boldsymbol{\eta^\pi} }[Z]$, $\widehat{V}= \mathbb{E}_{Z\sim \boldsymbol{\widehat{\eta}_T}}[Z]$. 
The two terms on the RHS of the above inequality can be immediately bounded using previous results and Hoeffding’s inequality, respectively.\\

\noindent
{\it Bounding $|\widehat{V}-V^\pi|$.}
It  follows  from Lemma \ref{lem: Wasserstein bound expectation} that 
\[
|\widehat{V}-V^\pi|\leq \mathcal{W}_1(\boldsymbol{\widehat{\eta}_T},\boldsymbol{\eta^{\pi}}). 
\]
 Applying Theorem 4.12 of the main text, with probability at least $1-c\log N \exp(-(N/\log N)^{\frac{2d}{2d+4\beta}})$, we have
 $$
\big|\widehat{V}-V^\pi\big| \le \frac{C C_{\mu}^{\frac{1}{2}}}{(1-\gamma)^{\frac{3}{2}}}(\log N)^4 N^{-\frac{\beta}{2\beta+d}} +  \frac{C_{F,R}}{(1-\gamma)^{\frac{3}{2}} }  N^{\frac{\xi \log \gamma}{2}}.
$$
\vspace{0.04in}

\noindent
{\it Bounding $|\widehat{V}_K-\widehat{V}|$.}
 By invoking Hoeffding’s inequality, with probability at least $1-K^{-1}$, we have
$$
\big|\widehat{V}_K-\widehat{V}\big| = \Big|\frac{1}{K}\sum_{k=1}^{K}\widehat{f}_T(s_{0,k},a_{0,k},\tau_k) -\mathbb{E}_{s_0\sim\rho,a_0\sim\pi,\tau\sim\mathrm{Unif}(0,1)} \widehat{f}_T(s_0,a_0,\tau)\Big| \le C F\sqrt{\frac{\log K}{K}}.
$$
\vspace{0.04in}

Finally, 
by applying the union bound,  with probability at least $1-c\log N \exp(-(N/\log N)^{\frac{2d}{2d+4\beta}})-K^{-1}$, we have
$$
\big|\widehat{V}_K-V^\pi\big| \le \frac{C C_{\mu}^{\frac{1}{2}}}{(1-\gamma)^{\frac{3}{2}}}(\log N)^4 N^{-\frac{\beta}{2\beta+d}} + C F\sqrt{\frac{\log K}{K}} + \frac{C_{F,R}}{(1-\gamma)^{\frac{3}{2}} }  N^{\frac{\xi \log \gamma}{2}}.
$$
By selecting $K\ge C  N^{\frac{2\beta}{2\beta+d}} (\log N)^{-6}$, the last term $\sqrt{\frac{1}{K}}$ can be absorbed into the first term. This completes the proof. 
%and the error bound for mean value estimation ultimately attains $N^{-\frac{\beta}{2\beta+d}}$
\end{proof}

\begin{lemma}\label{lem: Wasserstein bound expectation}
    Consider any random variables $Z_1\sim \nu_1, Z_2\sim\nu_2$.
For any $p\ge 1$,     we have 
\[
|\mathbb{E}Z_1 - \mathbb{E}Z_2|\leq \mathcal{W}_p(\nu_1,\nu_2). 
\]
\end{lemma}

\begin{proof}
By the definition of $p$-Wasserstein distance,
\begin{align*}
\mathcal{W}_p(\nu_1,\nu_2) &= \Big(\int_{0}^1 \big|f_{\nu_1}(\tau)-f_{\nu_2}(\tau)\big|^pd\tau \Big)^{\frac{1}{p}} &&\text{by Jensen's inequality} \\
&\ge \int_{0}^1 \big|f_{\nu_1}(\tau)-f_{\nu_2}(\tau)\big|d\tau \\
        &\ge \big| \int_{0}^1 f_{\nu_1}(\tau)-f_{\nu_2}(\tau)d\tau\big| \\
        & = \big|\mathbb{E}Z_1 - \mathbb{E}Z_2\big|,
    \end{align*}
which completes the proof. 
\end{proof}

%%%%%%%%%%%%%%%%%%%%%%%%%%%%%%%%%%%%%%%%%%%%%%%%%%%%%%%%%%%%%%%%%%%%%%%%%%%%%%%%%%

\section{Further Discussions on Assumptions} \label{sec: further discussions on assumptions}

\subsection{Discussions on sub-Gaussian Reward Condition}

In this section, we clarify that the bounded reward condition can be relaxed to a milder sub-Gaussian reward condition. From equation \eqref{bd_b_1} in the proof of Lemma 4.6, it suffices to bound 
$\mathcal{W}_p( \eta_0 (s,a), \eta^{\pi}  (s,a) )$ for any $(s,a)\in \mathcal{S}\times \mathcal{A}$ in the case of sub-Gaussian reward. 

For $t\ge1$, $u>0$, there exists a certain positive constant $C$ such that the  tail  of the sub-Gaussian rewards $R_t$ decays as 
$$
\mathbb{P}\big(|R_t|\ge u \big) \le 2 e^{-u^2/C}.
$$
This tail bound implies for any $t\ge1$, $R_t$ has a bounded $p$-th moment such that
\begin{align} \label{equ:p-th moment bounded}
\mathbb{E}\big[R_t^p\big]^{1/p}\le K\sqrt{p},
\end{align} 
where $K$ is some positive constant. Note that for any $p\ge 1$, the $p$-Wasserstein distance between $X\sim\mu$ and $Y\sim\nu$ can be bounded as follows,
\begin{align*}
\mathcal{W}_p(\mu, \nu) &= \Big( \inf_{\pi\in\Pi(\mu,\nu)} \int_{\mathbb{R}\times\mathbb{R}}\big\|x-y\big\|^p d\pi(x,y) \Big)^{\frac{1}{p}} \\
&  = \inf_{\pi\in\Pi(\mu,\nu)}\Big( \mathbb{E}_{(X,Y)\sim\pi}\big\|X-Y\big\|^p \Big)^{\frac{1}{p}}   \\
&\leq \Big(\mathbb{E}\big\|X\big\|^p \Big)^{\frac{1}{p}}  + \Big(\mathbb{E}\big\|Y\big\|^p \Big)^{\frac{1}{p}}    ~~~~~~ \text{by Minkowski inequality}.
%&\leq C_p \left( \| X \|_{\psi_2} + \| Y \|_{\psi_2} \right),
\end{align*}

To bound $\mathcal{W}_p( \eta_0 (s,a), \eta^{\pi}  (s,a) )$, let $Z_1\sim\eta_0(s,a)$ with $|Z_1|\leq F$, where $\eta_0(s,a)$ is the initialization distribution corresponding to the quantile function of $\widehat{f}_0 (s,a,\cdot)$, which is the initialization of the neural network. Let $Z_2\sim\eta^{\pi}(s,a)$ with $Z_2\leq \sum_{t\ge0} \gamma^t R_t$. Thus, for any $p\ge 1$, we have 
\begin{align*}
    \mathcal{W}_p(\eta_0 (s,a), \eta_\pi (s,a)) \le &  \Big(\mathbb{E}|Z_1|^p \Big)^{\frac{1}{p}}  + \Big(\mathbb{E}|Z_2|^p \Big)^{\frac{1}{p}} \\
     \le & F+ \Big(\mathbb{E}|Z_2|^p \Big)^{\frac{1}{p}} \\
    \le & F+ \left(\mathbb{E}\Big[\Big(\sum_{t\ge 0}\gamma^tR_t\Big)^p\Big]\right)^{1/p} \\
    \le & F+ \sum_{t\ge0} \gamma^t \Big(\mathbb{E}\Big[R_t^p\Big]\Big)^{1/p}\\
    \le & F+ K\sqrt{p} \sum_{t\ge 0} \gamma^t \\
    =& F+ \frac{K\sqrt{p}}{1-\gamma}. %(\sum_{t\ge1} \gamma^t = \frac{\gamma}{1-\gamma})?. 
\end{align*}

In consequence, under the sub-Gaussian reward condition, substituting the bound above into the decomposition in Lemma 4.6 yields
$$
\mathcal{W}_p(\boldsymbol{\eta}^{\pi},\boldsymbol{\widehat{\eta}}_{T})~ \leq~  \frac{2 C_{\mu}^{\frac{1}{2p}}}{(1-\gamma)^{\frac{3}{2}}}~ \max_{0<t\leq T}\widehat{\varepsilon}_{p, t} + \frac{\gamma^{\frac{T}{2}}}{(1-\gamma)^{\frac{3}{2}}} \left( F  +  K\sqrt{p} \right).
$$
Thus, the statistical error (first term) is unchanged, and the algorithmic error (second term) replaces the constant $F+R_{max}$ with the quantity $F  +  K\sqrt{p}$. In essence, the sub-Gaussian reward condition can be relaxed to a bounded $p$-th moment condition \eqref{equ:p-th moment bounded} by inspecting the derivation above. Recent work \citep{gerstenberg2024policy} also investigates distributional OPE for fixed quantile levels under unbounded reward conditions.

\subsection{ Discussions on Assumption 4.4}\label{sec: discussion on Assumption 4.4}

In this section, we provide a detailed discussion of Assumption 4.4. Specifically, we assume that after applying the distributional Bellman operator to a quantile function that lies within the deep neural network class $\mathcal{F}$, the resulting target quantile function belongs to the smooth H\"{o}lder class $\mathcal{G}$. Similar assumptions are also used in recent literature on deep quantile regression \citep{padilla2022quantile,shen2024}, where the target quantile function is similarly assumed to lie within a smooth function class. While the requirement for smoothness may seem restrictive in environments with discrete rewards, we note that reward functions are manually designed in RL and can be constructed to ensure continuity. 

We further elaborate on Assumption 4.4 in the context of distributional with a concrete example. Essentially,  it suffices to show that for the quantile function $f\in \mathcal{F}$ corresponding to the distribution $\eta(s,a)\in\Delta(\mathbb{R})$,  the target quantile function of $(\mathcal{T}^{\pi}\eta)(s,a)$ belongs to the  H\"{o}lder class.

Without loss of generality, we let state space $\mathcal{S} = (0,1)^{d-1}\subseteq \mathbb{R}^{d-1}$, action space $\mathcal{A} = \{1\}$ (single action), and quantile level $\tau\in(0,1)$. For any state $s  \in \mathcal{S}$, let the reward function $r(s) = \|s\|_2$ be deterministic, and let the transition be uniform $P(\cdot|s) = \mathrm{Unif}((0,1)^{d-1})$. By the definition of the distributional Bellman operator from equation (1) of the main text, we have
\begin{align*}
(\mathcal{T}^{\pi}\eta)(s) := &  \int_{s'\in\mathcal{S}} P(s' \mid s)  (g_{\gamma,r})_{\#} \eta(s') ds'\\
 %= & ~ r(s,a) + \gamma\sum_{s'\in\mathcal{S},a'\in\mathcal{A}} \pi(a'\mid s') P(s' \mid s, a)  \eta(s',a') \\
 = & ~ r(s) + \gamma\int_{s'\in\mathcal{S}} P(s' \mid s)  \eta(s')  ds'.
\end{align*}
Suppose the $\eta(s) = \mathrm{Unif}([a,b])$, for any $s\in \mathcal{S}$, then the quantile function of $\eta(s)$  is given by 
$$
f(s,\tau) = a + \tau (b-a).
$$
This simple linear function $f(s,\tau)$ belongs to the deep neural network class $\mathcal{F}$ with appropriate parameters.  Note that the return distribution $(\mathcal{T}^{\pi}\eta)(s)$ forms a mixture distribution such that
\begin{align*}
(\mathcal{T}^{\pi}\eta)(s)   %&~ \operatorname{law}\big( r(s) + \gamma Z(S') \mid s \big) \\
=& ~ r(s) + \gamma \int_{s'\in\mathcal{S}} P(s' \mid s)  \eta(s')  ds' \\
=& ~ \|s\|_2 + \gamma \mathrm{Unif}([a,b]) \\
=& ~  \mathrm{Unif}([\gamma a+ \|s\|_2,\gamma b+\|s\|_2]).
\end{align*}
Thus, the quantile function of $(\mathcal{T}^{\pi}\eta)(s)$ is $f(s,\tau) = \|s\|_2 + a\gamma + \tau \gamma (b-a)$. This function is clearly a linear function, and it belongs to the H\"{o}lder class $\mathcal{G}((0,1)^d,\beta,H)$  with certain parameters $\beta$ and $H$. %The right-hand side of the last equation can be described by the mixture of their cumulative distribution functions. 

In essence, this example illustrates that if both reward function $r(\cdot)$ and the transition function $P( \cdot \mid s)$ are sufficiently smooth, specifically belonging to the H\"{o}lder class  (in this case, we set them to be constants, which naturally satisfy this condition), then the quantile function of the resulting mixture distribution will also inherit this smoothness. This follows from the closure property of the H\"{o}lder class under addition and multiplication.

\subsection{Discussions on Assumption 4.5 and 4.10}\label{sec: discussion on convexity}
In this section, we provide a more detailed discussion of Assumption 4.5 and Assumption 4.10 within the main text. Assumption 4.5 essentially imposes a convexity condition on the population risk  $\mathcal{L}(\cdot)$  at the target function $f^{*}$. Assumption 4.10 imposes a local $c_0$-strong convexity condition and further requires a smoothness condition. Following the notations in Section \ref{sec: proof of error propagation}, we provide a sufficient condition that ensures these two assumptions. Specifically,  we provide a detailed proof to show Assumption 4.10, while the satisfaction of Assumption 4.5 is less restrictive, requiring only $|\delta|\le B$ rather than $|\delta|>0$ in the following Assumption \ref{assum: local_cond}. For a more thorough exploration, interested readers are encouraged to refer to Chapter 14 of \citet{wainwright2019high}.

%\ref{assum: convexity} in the main text imposes a local $c_0$-strongly convexity condition on the population loss  $\mathcal{L}(\cdot)$  at the target function. In this section, following the notations in Section A.2, we focus on this assumption and provide a sufficient condition that ensures the local $c_0$-strongly convexity of $\mathcal{L}(\cdot)$. 
\begin{assumption}\label{assum: local_cond}
There exist constants $B > 0$ and $h,h' > 0$ such that for any $|\delta|\leq B$, we have
    \begin{align*}
      \left|\delta\right| h \stackrel{(i)}{\le} \left|F_{Y|X=x}(f^{*}(x,\tau)+\delta)-F_{Y|X=x}(f^{*}(x,\tau))\right| \stackrel{(ii)}{\le}\left|\delta\right| h',
    \end{align*}
for all $\tau\in(0,1)$ and $x\in \mathcal{X}$, where $F_{Y|X=x}$ denotes the conditional distribution function of $Y$ given $X = x$.
\end{assumption}
Assumption \ref{assum: local_cond} implies that there exists a neighborhood around $f^{*}(\cdot,\tau)$ in which the conditional cumulative distribution function of $Y$ given $X=x$ is well behaved. The inequality (i) requires that the condition density function near $f^{*}(x,\tau)$ is bounded away from zero, and the inequality (ii) requires that the condition density function near $f^{*}(x,\tau)$ is bounded. Similar conditions are considered in prior works \citet{ belloni2011, madrid2022risk}. %Note that Assumption \ref{assum: local_cond} is slightly more relaxing than the condition in \citet{lian2022distributed}, which requires the condition density function at $f^{*}(x,\tau)$ is uniformly bounded away from zero. Additionally,  
Note that Assumption \ref{assum: local_cond} is weaker than the condition in \citet{belloni2011}, which further requires the condition density function to be continuously differentiable and bounded away from zero uniformly for all quantile level in $(0, 1)$ and all $x$ in the support $ \mathcal{X}$. 

 %We first establish the proof for local $c_0$-strong convex property of the population risk in Lemma \ref{lem: convexity_appendix} by applying the inequality (i) in Assumption \ref{assum: local_cond}. 

\begin{lemma}[Strong convexity and smoothness]\label{lem: convexity_appendix} 
Under the Assumption \ref{assum: local_cond}, for any $f: \mathcal{X} \times(0,1) \rightarrow$ $\mathbb{R}$, 
 and $\tau$ is independent of $X$, we have
$$
c_0\mathbb{E}_{X,\tau}\Big(f(X,\tau)- f^{*}(X,\tau)\Big)^2 \leq \big|\mathcal{L}(f)-\mathcal{L}(f^{*})\big|\leq c'_0\mathbb{E}_{X,\tau}\Big(f(X,\tau)- f^{*}(X,\tau)\Big)^2 ,
$$
where $c_0, c'_0$ are some constants depending on $B, h, h'$ (defined in Assumption \ref{assum: local_cond}) and $F$.%$c_{0}=\max \{\frac{h}{2}, \frac{hB}{4F}\}$, 

\end{lemma}
\begin{proof}
We first prove the LHS inequality.

For any $\tau\in (0,1)$ and $x\in \mathcal{X}$,
let $a=Y-f^{*}(x,\tau)$ and $b=f(x,\tau)-f^{*}(x,\tau)$. By using Knight's identity  
(Equation B.3 in \citet{belloni2011}) that $\rho_\tau(a-b)-\rho_\tau(a)=-b(\tau-\textbf{1}_{\{a\le 0\}})+\int_{0}^b(\textbf{1}_{\{a\le t\}}-\textbf{1}_{\{a\le 0\}})dt$ , we have
\begin{align}\label{lemA8_eq1}\nonumber
  &\mathbb{E}\big[\rho_\tau( Y-f(x,\tau) )-\rho_\tau(Y -f^{*}(x, \tau) )\mid  X=x\big] \nonumber \\
  &= - \mathbb{E}[(f(x,\tau)-f^{*}(x, \tau))(\tau-\textbf{1}_{\{Y \leq f^{*}(x,\tau)\}})\mid  X=x] \nonumber\\
   &\quad\quad+ \mathbb{E}\Big[\int_0^{f(x,\tau)-f^{*}(x,\tau)}(\textbf{1}_{\{Y\le f^{*}(x,\tau)+t\}}-\textbf{1}_{\{Y\le f^{*}(x,\tau)\}})dt\mid  X=x \Big]. 
\end{align}
Recall the definition of $f^{*}$, we have
\begin{align*}
  &\mathbb{E}[(f(x,\tau)-f^{*}(x, \tau))(\tau-\textbf{1}_{\{Y \leq f^{*}(x,\tau)\}})\mid X=x]\\
  &= (f(x,\tau)-f^{*}(x, \tau))\mathbb{E}[(\tau-\textbf{1}_{\{Y \leq f^{*}(x,\tau)\}})\mid X=x]=0.
\end{align*} 
Now we consider the second term in the right hand of \eqref{lemA8_eq1}. Following from  Fubini's theorem, we have
\begin{align*}
&\mathbb{E}\Big[\int_0^{f(x,\tau)-f^{*}(x,\tau)}(\textbf{1}_{\{Y\le f^{*}(x,\tau)+t\}}-\textbf{1}_{\{Y\le f^{*}(x,\tau)\}})dt\mid X=x\Big]\\
&=\int_0^{f(x,\tau)-f^{*}(x,\tau)}\mathbb{E}[\textbf{1}_{\{Y\le f^{*}(x,\tau)+t\}}-\textbf{1}_{\{Y\le f^{*}(x,\tau)\}}\mid X=x]dt\\
&=\int_0^{f(x,\tau)-f^{*}(x,\tau)}\big({F}_{Y|X=x} (f^{*}(x,\tau)+t)-{F}_{Y|X=x} (f^{*}(x,\tau))\big)dt.
\end{align*}
We proceed with the proof by analyzing it across three cases.

{\it Case 1:}  If $| f(x,\tau)-f^{*}(x,\tau)|\le B$, we have
\begin{align*}
&  \int_0^{f(x,\tau)-f^{*}(x,\tau)}\big({F}_{Y|X=x} (f^{*}(x,\tau)+t)-{F}_{Y|X=x} (f^{*}(x,\tau))\big)dt
\\
&\ge\int_0^{f(x,\tau)-f^{*}(x,\tau)} h|t|dt=\frac{h}{2}(f(x,\tau)-f^{*}(x,\tau) )^2,
\end{align*}
where the inequality follows the inequality (i) in Assumption \ref{assum: local_cond}.

{\it Case 2:}  If $f(x,\tau)-f^{*}(x,\tau)>B$, we have
\begin{align*}
 & \int_0^{f(x,\tau)-f^{*}(x,\tau)}\big({F}_{Y|X=x} (f^{*}(x,\tau)+t)-{F}_{Y|X=x} (f^{*}(x,\tau))\big)dt\\
 & \stackrel{(i)}{\ge}\int_{\frac{B}{2}}^{f(x,\tau)-f^{*}(x,\tau)} \big({F}_{Y|X=x} (f^{*}(x,\tau)+B/2)-{F}_{Y|X=x} (f^{*}(x,\tau))\big) dt\\
 &\ge \frac{hB}{2}(f(x,\tau)-f^{*}(x,\tau) -B/2)\\
 &\ge \frac{hB}{4}(f(x,\tau)-f^{*}(x,\tau))\\
 & \stackrel{(ii)}{\ge} \frac{hB}{4F}(f(x,\tau)-f^{*}(x,\tau))^2,
\end{align*}
where (i) holds due to the monotonicity of the conditional distribution function and (ii) follows from $\|f\|_\infty\le F$. 

{\it Case 3:}  If $f(x,\tau)-f^{*}(x,\tau)<-B$, by applying the same argument as in the proof for case 2,  we have
\begin{align*}
 &\int_0^{f(x,\tau)-f^{*}(x,\tau)}\big({F}_{Y|X=x} (f^{*}(x,\tau)+t)-{F}_{Y|X=x} (f^{*}(x,\tau))\big)dt\ge\frac{hB}{4F}(f(x,\tau)-f^{*}(x,\tau))^2.
\end{align*}
Then, combining the three cases and let $c_0=\max\{\frac{h}{2}, \frac{hB}{4F}\}$, we have
\begin{align*}
\mathbb{E}\Big[\int_0^{f(x,\tau)-f^{*}(x,\tau)}(\textbf{1}_{\{Y\le f^{*}(x,\tau)+t\}}-\textbf{1}_{\{Y\le f^{*}(x,\tau)\}})dt\mid X=x\Big]
    \ge c_0 (f(x,\tau)-f^{*}(x,\tau))^2.
\end{align*}
Taking expectation over $X$ and $\tau$ yields 
$$c_0\mathbb{E}_{X,\tau}(f(X,\tau)- f^{*}(X,\tau))^2 \leq \big|\mathcal{L}(f)-\mathcal{L}(f^{*})\big|.
$$

Similar arguments give the other direction that $$
\big|\mathcal{L}(f)-\mathcal{L}(f^{*})\big| \leq c'_0\mathbb{E}_{X,\tau}(f(X,\tau)- f^{*}(X,\tau))^2.
$$
\end{proof}

%%%%%%%%%%%%%%%%%%%%%%%%%%%%%%%%%%%%%%%%%%%%%%%%%%%%%%%%%%%%%%%%%%%%%%%%%%%%%%%%%%%%

\section{ $\beta$-mixing for Dependent Sequence }\label{sec: beta-mixing for dependent sequence}

%The results thus far assume that the data samples are independent. To better analyze the sequential data in RL, we extended the discussion beyond the i.i.d. To help analysis, we introduce the Definition \ref{def: mixing_definition} of $\beta$-mixing, a commonly used framework to quantify dependence in time series data. 

The preceding analysis assumes independent observations. To accommodate sequential dependence in RL data, we extend the analysis to strictly stationary $\beta$-mixing sequences. Following the notation in Section 3 of the main text, we denote the data sequence as $\left\{S_i\right\}_{i=1}^n:=\{(x_i,y_i,\tau_i)\}_{i=1}^n$.

\begin{definition}[$\beta$-mixing]\label{def: mixing_definition} 
Let $\left\{W_t\right\}_{t \geq 1}$ be a stochastic process and denote the collection $\left(W_1, \ldots, W_t\right)$ as $W^{1: t}$, where we allow $t=\infty$. Let $\sigma\left(W^{i: j}\right)$ denote the $\sigma$-algebra generated by $W^{i: j}$ for $i \leq j$. The $s$-th $\beta$-mixing coefficient of $\left\{W_t\right\}_{t \geq 1}$ is 
$$
\beta_s=\sup _{t \geq 1} \mathbb{E}\left[\sup _{B \in \sigma\left(W^{t+s: \infty}\right)}\left|P\left(B \mid W^{1: t}\right)-P(B)\right|\right] .
$$
The process is said to be $\beta$-mixing if $\beta_s \rightarrow 0$ as $s \rightarrow \infty$. It is exponentially $\beta$-mixing with parameters $\bar{\beta}, b, q>0$ if $\beta_s \leq \bar{\beta} \exp \left(-b s^q\right), s\ge 0$. 
\end{definition}

The $\beta$-mixing condition is standard for quantifying temporal dependence in time series \citep{chen2006estimation,wong2020lasso} and Markov decision process sequences \citep{antos2007fitted,antos2008learning,lazaric2012finite}.  To handle such dependence, we use the independent-block (IB) construction of \citet{yu1994rates}. The sample $\{S_i\}_{i=1}^n$ is partitioned into $2\mu_n$ nonoverlapping blocks of length $a_n$, where $n=2a_n\mu_n$, and every other block is replaced by an independent copy. This yields an approximately independent block sequence, which allows us to apply empirical process tools developed for the i.i.d. setting

%A commonly used method to deal with $\beta$-mixing sequence is the independent block (IB), which was introduced by \cite{yu1994rates}. This method partitions the data into several non-overlapping blocks. By appropriately choosing the block size and the number of blocks, the IB technique ensures that the blocks are approximately independent. This allows the original problem to be transformed into an equivalent problem under independence, enabling the use of standard tools developed for i.i.d. settings.

%We divide samples $\left\{S_i\right\}_{i=1}^n$ into $2 \mu_n$ blocks with length $a_n\left(n=2 a_n \mu_n\right)$ and replace half of these blocks with independent copies. This construction transforms the original problem into the analysis of an independent block (IB) sequence, which enables the application of standard tools developed for independent data.  See Section \ref{sec: proof of dependent data} in the Appendix for more details. Therefore, we can derive the upper bound of the excess risk when the samples are $\beta$-mixing, as stated in the following theorem.

%obtain the Rademacher complexity of a function class $\mathcal{F}(W, L)$,

\begin{theorem}[Excess risk bound with dependent data]\label{thm: main results with dependent}
If $\left\{S_i\right\}_{i=1}^n:=\{x_i,y_i,\tau_i\}_{i=1}^n$ is strictly stationary $\beta$-mixing  defined in Definition \ref{def: mixing_definition},  
   \begin{align}\label{equ: dependent data bound 1}
        \mathbb{E}\Big(\mathcal{L}(f_t)-\mathcal{L}(f^*_t) \Big)\lesssim \sqrt{\frac{W^2 L^2\log(W^2L) \log \mu_n}{\mu_n}}+2 \mu_n \beta_{a_n} + \inf_{f\in\mathcal{F}}\left(\mathcal{L}(f)-\mathcal{L}(f^*_t)\right),
    \end{align}
where $\beta_{a_n}\leq \bar{\beta} \exp \left(-b a_n^q\right)$ for some constant $\bar{\beta}, b, q>0$ and for $a_n>0$.  With
 the same choice of Length $L$ and width $W$ of the neural network $\mathcal{F}$ as in Theorem 4.8,
    \begin{align*}
   \mathbb{E}\Big(\mathcal{L}(f_t)-\mathcal{L}(f^*_t) \Big) \leq C (\log \mu_n)^{3}\mu_n^{-\frac{\beta}{2\beta+d}} + 2 \mu_n \beta_{a_n},
    \end{align*}
where $C$ is a constant independent of $\mu_n,\beta_{a_n}$. Moreover, if $\mu_n=n/(\log n)^\xi$ for some constant $\xi>0$, we further have
\begin{align}\label{equ: dependent data bound 2}
   \mathbb{E}\Big(\mathcal{L}(f_t)-\mathcal{L}(f^*_t) \Big) \leq C (\log n)^{3}\left(\frac{n}{(\log n)^{\xi}}\right)^{-\frac{\beta}{2\beta+d}} +\frac{2n\bar{\beta}}{(\log n)^{\xi}}e^{-b\frac{(\log n)^{\xi q}}{2^q}}.
    \end{align}
\end{theorem}

The bound \eqref{equ: dependent data bound 1} implies that, in addition to the width $W$ and length $L$ of neural networks, the dependent-data error depends on the effective number of IB blocks $\mu_n$ and the mixing term $\beta_{a_n}$. Setting $\mu_n=n/(\log n)^\xi$, and hence  $a_n=\frac{1}{2}(\log n)^{\xi}$ \citep{liang2009strong,hang2017bernstein}, yields the bound  \eqref{equ: dependent data bound 2}.
Due to the use of the IB technique, the effective sample size is reduced from $n$ to $\mu_n$,  which is the inherent cost of handling dependence in the sequential data. Under the condition of  $\log n \leq b\frac{(\log n)^{\xi q}}{2^q}$ ensuring $\lim_{n\to\infty}\frac{2n\bar{\beta}}{(\log n)^{\xi}}e^{-b\frac{(\log n)^{\xi q}}{2^q}}=0$, the additional dependence term vanishes. Moreover, if $ne^{-b\frac{(\log n)^{\xi q}}{2^q}}\leq n^{-\frac{\beta}{2\beta+d}}$, the overall excess risk  attains the rate $\mathcal{O}(n^{-\frac{\beta}{2\beta+d}})$ up to logarithmic factors.

\subsection{Proof of Theorem \ref{thm: main results with dependent}}\label{sec: proof of dependent data}

The proof relies on the independent-block argument of \citet{yu1994rates}. Starting from a strictly stationary $\beta$-mixing sequence $\{Z_i\}_{i=1}^n$, we construct alternating blocks and replace every other block by an independent copy. This reduces the analysis of the dependent sequence to that of an approximately independent one and makes it possible to control the relevant empirical process.

For any integers pair $\left(a_n, \mu_n\right)$ with $n=2 a_n \mu_n$, we divide the strictly stationary $n$-sequence $\left\{Z_i\right\}_{i=1}^n$ into $2 \mu_n$ blocks with length of  $a_n$:
\begin{equation*}
\begin{aligned}
& \underbrace{Z_1, \ldots, Z_{a_n}}_{H_1}, ~\underbrace{Z_{a_n+1}, \ldots, Z_{2 a_n}}_{T_1}, ~\underbrace{Z_{2 a_n+1}, \ldots, Z_{3 a_n}}_{H_2}, ~\underbrace{Z_{3 a_n+1}, \ldots, Z_{4 a_n}}_{T_2}, ~\ldots, \\
&\underbrace{Z_{(2 \mu_N-2) a_n+1}, \ldots,  Z_{(2 \mu_N-1) a_n}}_{H_{\mu_n}},~   \underbrace{Z_{\left(2 \mu_N-1\right) a_n+1}, \ldots, Z_{2 \mu_N a_n}}_{T_{\mu_n}}  .
\end{aligned}
\end{equation*}
For $1 \leq j \leq \mu_n$, define%denote the indices in the blocks alternately by $H$'s and $T$'s such that 
\begin{align*}
H_j & :=\left\{i: 2(j-1) a_n+1 \leq i \leq(2 j-1) a_n\right\}, \\
T_j & :=\left\{i:(2 j-1) a_n+1 \leq i \leq(2 j) a_n\right\}.
\end{align*}
Here, we denote $H =\bigcup_{j=1}^{\mu_n}\{H_j\}$. Denote the random variables that correspond to the $H_j$ and $T_j$ indices as
\begin{align*}
& Z\left(H_j\right)=\left\{Z_i, i \in H_j\right\}=\left\{Z_{2(j-1) a_n+1}, \cdots, Z_{(2 j-1) a_n}\right\}, \\
& Z\left(T_j\right)=\left\{Z_i, i \in T_j\right\}=\left\{Z_{(2 j-1) a_n+1}, \cdots, Z_{(2 j) a_n}\right\} .
\end{align*}
Denote the whole sequence of $H$-blocks and $T$-blocks as 
\begin{align*}
 Z_{a_n}&:=\left\{Z(H_j): j=1,2, \ldots, \mu_n\right\}, \\
 Z_{1, a_n}&:=\left\{Z(T_j): j=1,2, \ldots, \mu_n\right\} .
\end{align*}

Next, corresponding to every second block $H_i$, we introduce block-independent 
"ghost" samples such that
\begin{equation*}
\begin{aligned}
& \underbrace{Z'_1, \ldots, Z'_{a_n}}_{H_1},  ~\underbrace{Z'_{2 a_n+1}, \ldots, Z'_{3 a_n}}_{H_2}, \ldots,~\underbrace{Z'_{(2 \mu_N-2) a_n+1}, \ldots,  Z'_{(2 \mu_N-1) a_n}}_{H_{\mu_n}}.
\end{aligned}
\end{equation*}
Denote $\{\Xi(H_j): j=1, \ldots, \mu_n\}$, where the sequence $\Xi(H_j)=\{Z'_i: i \in H_j\}$ is independent of $\left\{Z_i\right\}_{i=1}^n$. Here, each block of $\Xi(H_j)$ has the same joint marginal distribution as the corresponding block in the original data. We call this constructed sequence the independent block $a_n$-sequence (IB sequence), and denote the IB sequence as 
$$
\Xi_{a_n}:= \left\{Z'(H_j): j=1,2, \ldots, \mu_n\right\}.
$$  
This construction is the key device for transferring the dependent problem to an approximately independent one.

%Next, we give some definitions. 
Let $\{\sigma_i\}_{i=1}^n$ be i.i.d Rademacher random variables, and $\sigma_i$ is independent of $\{Z_i\}_{i=1}^n$ and $\{Z'_i\}_{i=1}^n$. For a measurable function $g$, define
$$
P_n g:=\frac{1}{n} \sum_{i=1}^n \sigma_i g\left(Z_i\right) .
$$
For the original sequence $Z_{a_n}$, we write
$$
\widetilde{Y}_{j, g}(Z_{a_n}):=\sum_{i \in H_j} \sigma_i g(Z_i) .
$$
For the constructed IB sequence $\Xi_{a_n}$, define
$$
W_{j, g}(\Xi_{a_n}):=\sum_{i \in H_j} \sigma_i g(Z'_i) .
$$

\begin{lemma}[Lemma 4.1 in \cite{yu1994rates}]\label{lem: yu1994}
Let the distributions of $Z_{a_n}$ and $\Xi_{a_n}$ be $\mathcal{Q}$ and $\widetilde{\mathcal{Q}}$, respectively. For any measurable function $h$ on $\mathbb{R}^{\mu_n a_n}$ with bound $M$,
$$
\left|\mathcal{Q} h\left(Z_{a_n}\right)-\widetilde{\mathcal{Q}} h\left(\Xi_{a_n}\right)\right| \leq M\left(\mu_n-1\right) \beta_{a_n} .
$$
\end{lemma} 
Lemma \ref{lem: yu1994} is the key to connecting the mixing sequence and
the independent block sequence.

\begin{lemma}\label{lem: mixing_lemma}
Suppose that $\mathcal{F}_{M}$ is a function class bounded by $M$, then we have
$$
\mathbb{E}_{(Z, \sigma)}\left(\sup _{g \in \mathcal{F}_{M}}\left|P_n g\right|\right) \leq \mathbb{E}_{(Z^{\prime}, \sigma)}\left(\sup _{g \in \mathcal{F}_{M}}\left|\frac{1}{\mu_n} \sum_{j=1}^{\mu_n} \frac{W_{j, g}\left(\Xi_{a_n}\right)}{a_n}\right|\right)+2 M \mu_n \beta_{a_n} .
$$
where  $\mathbb{E}_{(Z,\sigma)}$ denotes the expectation taken over $\{Z_i,\sigma_i\}_{i=1}^n$.
\end{lemma}

\begin{proof}

Note that the strictly $\beta$-mixing process $Z_{a_n}:=\left\{Z\left(H_j\right) ; j=1, \ldots, \mu_n\right\}$ has the same distribution as $Z_{1, a_n}:=\left\{Z\left(T_j\right) ; j=1, \ldots, \mu_n\right\}$. Then, we have
\begin{align*}
& \mathbb{E}_{(Z, \sigma)}\left(\sup _{g \in \mathcal{F}_{M}}\left|P_n g\right|\right) \\
& =\mathbb{E}_{(Z, \sigma)}\left(\sup_{g \in \mathcal{F}_{M}}\left|\frac{1}{n} \sum_{j=1}^{\mu_n} \widetilde{Y}_{j, g}\left(Z_{a_n}\right)+\frac{1}{n} \sum_{j=1}^{\mu_n} \widetilde{Y}_{j, g}\left(Z_{1, a_n}\right)\right|\right) \\
& \leq \mathbb{E}_{(Z, \sigma)}\left(\sup _{g \in \mathcal{F}_{M}}\left|\frac{1}{n} \sum_{j=1}^{\mu_n} \widetilde{Y}_{j, g}\left(Z_{a_n}\right)\right|\right)+\mathbb{E}_{(Z, \sigma)}\left(\sup _{g \in \mathcal{F}_{M}}\left|\frac{1}{n} \sum_{j=1}^{\mu_n} \widetilde{Y}_{j, g}\left(Z_{1, a_n}\right)\right|\right) \\
& \overset{(i)}{=}2 \mathbb{E}_{(Z, \sigma)}\left(\sup _{g \in \mathcal{F}_{M}}\left|\frac{1}{n} \sum_{j=1}^{\mu_n} \widetilde{Y}_{j, g}\left(Z_{a_n}\right)\right|\right) \\
& \leq 2 \mathbb{E}_{(Z', \sigma)}\left(\sup _{g \in \mathcal{F}_{M}}\left|\frac{1}{n} \sum_{j=1}^{\mu_n} W_{j, g}\left(\Xi_{a_n}\right)\right|\right)+2 M \mu_n \beta_{a_n} \\
& =\mathbb{E}_{(Z', \sigma)}\left(\sup _{g \in \mathcal{F}_{M}}\left|\frac{1}{\mu_n} \sum_{j=1}^{\mu_n} \frac{W_{j, g}\left(\Xi_{a_n}\right)}{a_n}\right|\right)+2 M \mu_n \beta_{a_n}, 
\end{align*}
where (i) holds since $\tilde{Y}_{j, g}\left(Z_{a_n}\right)$ and $\tilde{Y}_{j, g}\left(Z_{1, a_n}\right)$ have the same distribution, and the last inequality follows from Lemma \ref{lem: yu1994}.% and the fact that $\sigma_i$ is independent of both $Z_i$ and $Z_i'$. 

\end{proof}

We begin by decomposing the excess risk. As in the proof of Theorem \ref{thm: bd_ slow_appendix},  
\begin{align*}
\mathbb{E}\Big(\mathcal{L}(\widehat{f})-\mathcal{L}(f^{*})\Big)\leq 2~\mathbb{E}\Big(\sup_{f\in\mathcal{F}} ~ \big|\mathcal{L}(f)-\widehat{\mathcal{L}}(f) \big|\Big) + \inf_{f\in\mathcal{F}}~ \big(\mathcal{L}(f)-\mathcal{L}(f^{*})\big).
\end{align*}

It therefore suffices to control the empirical process $\mathbb{E}\big(\sup_{f\in\mathcal{F}} ~ \big|\mathcal{L}(f)-\widehat{\mathcal{L}}(f) \big|\big)$. By the same symmetrization argument as in Theorem \ref{thm: bd_ slow_appendix}, it is enough to bound
\begin{align*}
   & \mathbb{E} \Big[\sup_{f\in\mathcal{F}}\Big|\frac{1}{n}\sum_{i=1}^n \sigma_i f(x_i,\tau_i) \Big| \Big]\\
    &\stackrel{(i)}{\leq} \mathbb{E}\left(\sup_{f \in \mathcal{F}}\left|\frac{1}{\mu_n} \sum_{j=1}^{\mu_n} \frac{W_{j, f}\left(\Xi_{a_n}\right)}{a_n}\right|\right)+2 \mu_n \beta_{a_n}~~~~~~~~~~\text{by Lemma \ref{lem: mixing_lemma}}\\
    & \stackrel{(ii)}{\lesssim}\frac{F}{\sqrt{\mu_n}} \int_{0}^{1}\sqrt{\log N_{\mu_n}(\delta,\mathcal{F}_{\mu_n}/F,\|\cdot\|_{\infty})} d\delta +2 \mu_n \beta_{a_n} ~~~~~~\text{by Dudley’s theorem}\\
    & \stackrel{(iii)}{\lesssim} \frac{F}{\sqrt{\mu_n}} \int_{0}^{1}\sqrt{\log\left(\frac{ e \cdot \mu_n}{\delta \cdot \operatorname{Pdim}(\mathcal{F})}\right)^{\operatorname{Pdim}(\mathcal{F})}} d\delta +2 \mu_n \beta_{a_n} \\
    & =  \frac{F\sqrt{\operatorname{Pdim}(\mathcal{F})}}{\sqrt{\mu_n}} \int_{0}^{1}\sqrt{\log\left(\frac{ e \cdot \mu_n}{\delta \cdot \operatorname{Pdim}(\mathcal{F})}\right)} d\delta +2 \mu_n \beta_{a_n}\\
    &\leq C \sqrt{\frac{\operatorname{Pdim}(\mathcal{F})\log \mu_n}{\mu_n}}+2 \mu_n \beta_{a_n} \\
    &\leq C \sqrt{\frac{W^2 L^2\log(W^2L) \log \mu_n}{\mu_n}}+2 \mu_n \beta_{a_n},
\end{align*}
Hence, by \eqref{equ: symmetrization},
\begin{align*}
\mathbb{E}\sup_{f\in\mathcal{F}}\big|\mathcal{L}(f)-\widehat{\mathcal{L}}(f)\big| \leq  C \sqrt{\frac{W^2 L^2\log(W^2L) \log \mu_n}{\mu_n}}+8 \mu_n \beta_{a_n}.
\end{align*}

By Lemma \ref{lem: jiao_approximation}, for sufficiently large $U, V\in \mathbb{N}^+ $,  choose 
\begin{align*}
 W=\mathcal{O}\left((s+1)^2 d^{s+1} U \log U\right),\ L=\mathcal{O}\left((s+1)^2 V \log V\right),   
\end{align*}
By \eqref{equ: approximation error bound}, the approximation error can be bounded by
\begin{align*}
\inf_{f\in\mathcal{F}}~ (\mathcal{L}(f)-\mathcal{L}(f^{*}))\leq C H(s+1)^2 d^{s+\frac{\beta \vee 1}{2}}(U V)^{-\frac{2\beta}{d}}.
\end{align*}

Taking $UV=\lfloor \mu_n^{\frac{d}{2d+4\beta}} \rfloor$,  and combining the stochastic and approximation bounds, we obtain
\begin{align*}
\mathbb{E}\Big(\mathcal{L}(\widehat{f})-\mathcal{L}(f^{*})\Big) \leq &  C \sqrt{\frac{W^2 L^2\log(W^2L) \log \mu_n}{\mu_n}}+  C H(s+1)^2 d^{s+\frac{\beta \vee 1}{2}}(U V)^{-\frac{2\beta}{d}}+8 \mu_n \beta_{a_n} \\
\lesssim &  \frac{(s+1)^4d^{s+1}(UV)(\log U\log V)\sqrt{\log(W^2L)}\sqrt{ \log \mu_n}}{\sqrt{\mu_n}} \\
& + (s+1)^2 d^{s+\frac{\beta \vee 1}{2}}(U V)^{-\frac{2\beta}{d}} +8 \mu_n \beta_{a_n}\\
\leq &  \frac{(s+1)^4d^{s+1}(UV) (\log \mu_n)^3}{\sqrt{\mu_n}}+ (s+1)^2 d^{s+\frac{\beta \vee 1}{2}}(U V)^{-\frac{2\beta}{d}} +8 \mu_n \beta_{a_n}\\
\leq & 2 (s+1)^4 d^{s+\frac{\beta \vee 1}{2}} (\log \mu_n)^3 \mu_n^{-\frac{\beta}{2\beta+d}}+8 \mu_n \beta_{a_n}
\end{align*}
This completes the proof.

\section{Supporting Lemmas}\label{sec: Supporting lemmas}

\begin{definition}[Shattering]
Let $\cH$ be a family of functions from a set $\mathcal{X}$ to $\mathbb{R}$. A set $\{x_1, \ldots, x_n \}$ is said to be shattered by $\cH$, if there is a set $\{t_1,...,t_n\}$ such that 
$$
\left|\left\{(\text{sign}(h(x_1)-t_1),\cdots,\text{sign}(h(x_n) -t_n)): h \in \cH \right\}\right|=2^n, 
$$
% The term inside set $\{\cdot\}$ is counting how many different configurations of the vector $(h\left(z_1\right), \ldots, h\left(z_n\right))^\T$ are possible, and
where $|\cdot|$ denotes the cardinality of a set.
\end{definition}

\begin{definition}[Pseudo-dimension]
Let $\mathcal{F}$ be a family of functions mapping from $\mathcal{X}$ to $\mathbb{R}$. Then, the pseudo dimension of $\mathcal{F}$, denoted by $\operatorname{Pdim}(\mathcal{F})$, is the size of the largest set shattered by $\mathcal{F}$.
\end{definition}
%\begin{definition}[Pseudo-dimension]
%Consider a class  $\mathcal{H}$ of measurable functions $f: \mathcal{X}\rightarrow \mathbb{R}$, the pseudo-dimension $\operatorname{Pdim}(\mathcal{H})$ of $\mathcal{H}$ is defined as the largest integer $n$ for which there is some collection of $n$ pairs $\{x_i\times y_i\}_{i=1}^n \subset\mathcal{X}\times \mathbb{R}$ such that for any $\varepsilon=(\varepsilon_1,...,\varepsilon_n)^T\in \{0,1\}^n$, there exists a function $f\in \mathcal{H}$ such that $\forall i\in [n]$, we have $(f(x_i)-y_i)\varepsilon_i>0$.     
%\end{definition}
The pseudo-dimension serves as a crucial measure of the richness of the function space $\mathcal{H}$, namely a large pseudo-dimension implies a richer $\mathcal{H}$ and vice versa. It is noteworthy that the pseudo-dimension shares a close relationship with another well-known measure VC-dimension. Specifically, for $\mathcal{F}= \mathcal{F}(W,L)$ with a fixed architecture and fixed activation functions, we have $\operatorname{Pdim}(\mathcal{F}) = \operatorname{VCdim} (\mathcal{F})$ \citep{Bartlett2019}. %Moreover, \cite{anthony1999neural} provide the following two important results.  

Recall the definition  of $\cF_n$ from \eqref{def_F_n}. 
The following two lemmas are important, showing that the covering number of $\cF_n$ under $\|\cdot\|_\i$ norm can be bounded in terms of the pseudo-dimension and the pseudo-dimension can be further bounded in terms of the depth and width of the ReLU network.

\begin{lemma}[Theorem 12.2 in \citet{anthony1999neural}]\label{pseudo}
Assume that for all $f \in \mathcal{F},\|f\|_{\infty} \leq F$.  Then the following inequality holds for any  $n \geq \operatorname{Pdim}(\mathcal{F})$ and any $\delta>0$. 
\begin{align*}
N  \left(\delta,\mathcal{F}_{n}, \|\cdot\|_{\infty} \right) \leq\left(\frac{ e F  n}{\delta  \operatorname{Pdim}(\mathcal{F})}\right)^{\operatorname{Pdim}(\mathcal{F})}.
\end{align*}
\end{lemma}

\begin{lemma}[Theorem 7 in \citet{Bartlett2019}] \label{pseudo_bound}
Consider a ReLU network architecture $\mathcal{F}=\mathcal{F}(W, L)$, then for the pseudo-dimension, 
we have
$$
\operatorname{Pdim}(\mathcal{F}) \le CW^2 L^2 \log (W^2 L),
$$
where $C$ is an absolute constant.
\end{lemma}

\begin{lemma}[Approximation ability of deep ReLU networks for H\"{o}lder spaces;  Theorem 3.3 in \citet{jiao2023deep}]\label{lem: jiao_approximation}
 Let $\beta=s+r$ with $s \in \mathbb{N}$ and $r \in(0,1]$.  For sufficiently large $U, V \in \mathbb{N}^+  $, there exists a ReLU neural network architecture $\mathcal{F}(W, L)$ with width $W=\mathcal{O}\left((s+1)^2 d^{s+1} U \log U\right)$ and length $L=\mathcal{O}\left((s+1)^2 V \log V \right)$ such that
$$
\sup _{f^{*} \in \mathcal{G}([0,1]^d, \beta, H)} \inf_{f \in \mathcal{F}(W, L)}\|f^{*}-f\|^2_{2} 
\leq C H^2(s+1)^4 d^{2s+({\beta \vee 1})}(U V)^{-4 \beta / d},
$$
where $C$ is an absolute constant.
\end{lemma}
%%%%%%%%%%%%%%%%%%%%%%%%%%%%%%%%%%%%%%%%%%%%%%%%%%%%%%%%%%%%%%%%%%%%%%%%%%%%%%%
%%%%%%%%%%%%%%%%%%%%%%%%%%%%%%%%%%%%%%%%%%%%%%%%%%%%%%%%%%%%%%%%%%%%%%%%%%%%%%%

%\begin{lemma}[Bernstein's Inequality; Proposition 2.14 in \citet{wainwright2019high}]\label{bern_lem}
%Let $\{X_i\}_{i=1}^n$ be $n$ independent random variables satisfying $|X_i|\le b$ for all $i\in[n]$. Then for any $\delta>0$, we have
%\begin{align*}
%    P\Big\{\sum_{i=1}^n (X_i-\mathbb{E}[X_i])\ge n\delta\Big\}\le \exp\Big\{ -\frac{n\delta^2}{2(\frac{1}{n}\sum_{i=1}^nE[X_i^2]+\frac{b\delta}{3}) }\Big\}. 
%\end{align*}
%\end{lemma}

%\begin{lemma}[Symmetrization]\label{equ: symmetrization}
%Let $X_1,...,X_n$ be a sequence of random variables and $\sigma_1,...,\sigma_n$ denote the Rademacher  variables independent of $X_1,...,X_n$.
%For any measurable function class $\mathcal{F}$, we have
%\begin{align*}
%\mathbb{E}\Big[\frac{1}{n}\sup_{f\in\mathcal{F}} \big|\sum_{i=1}^n(f(X_i)-\mathbb{E}[f(X_i)])\big| \Big]\le2 \mathbb{E}\Big[\frac{1}{n}\sup_{f\in\mathcal{F}} \big|\sigma_i\sum_{i=1}^nf(X_i)\big| \Big]
%\end{align*}
%\end{lemma}

The following lemma can be found in  \citet{wainwright2019high}, which allows us to utilize the symmetrization technique for the Lipschitz function family. 
\begin{lemma}[Ledoux–Talagrand contraction inequality] \label{lem: Ledoux_Talagrand_contraction}
For any set $\mathcal{T}\in \mathcal{R}^d$,  let $\{\phi_j: \mathcal{R}\rightarrow\mathcal{R},j=1,...,d\}$ be any family of $C_\phi$-Lipschitz functions such that $\phi_j(0)=0$  for $j\in [d]$. Then, we have
\begin{align*}
\mathbb{E}\left( \sup_{\theta\in \mathcal{T}}\Big|\sum_{j=1}^d\sigma_j\phi_j(\theta_j)\Big| \right)\le 2C_\phi~  \mathbb{E}\left( \sup_{\theta\in \mathcal{T}}\Big| \sum_{j=1}^d \sigma_j\theta_j\Big|\right).
\end{align*}
\end{lemma}

The following two lemmas are useful in our proof. The bounded differences inequality in Lemma \ref{lem: bounded differences}  is applied to analyze the slow rate in Section \ref{sec: C}.
The Lemma \ref{Tala_concen} provides
a functional version of Bernstein's inequality,
used to analyze the fast rate in Section \ref{sec: D}. Proofs for both inequalities are provided in \cite{wainwright2019high}.

\begin{lemma}[Bounded differences inequality]\label{lem: bounded differences}
    We say that a function $f: \mathbb{R}^n \rightarrow \mathbb{R}$ satisfies the bounded differences property if $\exists  ~L_1, \ldots, L_n>0$ s.t. 
    % for any $\left(X_1, \ldots, X_n\right)$ in the domain of $f$, 
    for any coordinate $k$,
$$
\sup _{x^{\prime}, y^{\prime}}\left|f\left(x_1, \ldots, x_{k-1}, y^{\prime}, x_{k+1}, \ldots, x_n\right)-f\left(x_1, \ldots, x_{k-1}, x^{\prime}, x_{k+1}, \ldots, x_n\right)\right| \leq L_k.
$$
Suppose that $f$ satisfies the bounded difference property with $(L_1,\dots,L_n)$ and that the random vector $X =(X_1,X_2,\dots,X_n)$ has independent components. Then, for all $t> 0$ we have
\begin{align*}
\mathbb{P}\left( \big| f(X_1, \cdots, X_n)-\mathbb{E}f(X_1, \cdots, X_n) \big|\geq t\right) \leq 2 \exp \left(\frac{-2 t^2}{\sum_{i=1}^n L_k^2}\right).
\end{align*}
\end{lemma}

\begin{lemma} [Talagrand concentration for empirical processes] \label{Tala_concen}
Let $X_1,..., X_n$ be a collection of i.i.d. random variables supported on some space $\mathcal{X}$. Consider 
a countable class of $b$-uniformly bounded functions $\mathcal{F}$ from $\mathcal{X}$ to $\mathcal{R}$. Define the  random variable as $\mathcal{Z}=\sup_{f\in\mathcal{F}}\big\{\frac{1}{n}\sum_{i=1}^n f(X_i)\big\}$. Then, there are an universal constants $C$ such that
for all $t>0$, the random variable $\mathcal{Z}$ satisfies the upper tail bound
$$
\mathbb{P}\left(\mathcal{Z} \geq 2\mathbb{E}[\mathcal{Z}]+C(V\sqrt{t}+ b t)\right) \leq 2 \exp (-nt),
$$
where $V^2= \sup_{f \in \mathcal{F}} \mathbb{E}[f(X_1)^2]$ is the variance term.

% where $\Sigma^2=\sup _{f \in \mathcal{F}} \frac{1}{n} \sum_{i=1}^n f^2\left(X_i\right)$.    Moreover,
% \begin{align*}
%     \mathbb{E}[\Sigma^2]\le \sup_{f \in \mathcal{F}}     \mathbb{E}[f(X_1)^2]+2b\mathbb{E}[ \mathcal{Z}]. 
% \end{align*}
\end{lemma}

\section{Additional Experimental Results}\label{sec: expriments}

This section provides additional results and detailed experimental settings. In the experiments,   we compare DQPOPE with value-based OPE implemented using deep ReLU networks (referred to as DOPE), which estimates the value function $Q^{\pi}$ by minimizing squared loss. Following the notation in Section 3 of the main text, the iterative process of DOPE is summarized as follows.  

\begin{algorithm}[H]
		\caption{\small{Deep value-based OPE (DOPE)}}
		\label{alg: standardd OPE}
		\begin{algorithmic}[1]\small
			\STATE {\bfseries Initialize:} DNN class $\mathcal{F}$, $\widehat{Q}_0\in\mathcal{F}$, datasets $\{\mathcal{D}_t\}_{t=1}^{T}$ and target policy $\pi$.
			\FOR{$t =1$ to $T$}
			\STATE Collect sample $(s_i, a_i, r_i, s'_i)\in \mathcal{D}_t$
			\STATE Update: ~~
			$\widehat{Q}_{t} \leftarrow \arg \underset{Q\in \mathcal{F}}{\min}\frac{1}{|\mathcal{D}_t|}\sum_{\mathcal{D}_t}\Big(Q(s_i,a_i)-r_i-\widehat{Q}_{t-1}(s'_i, a'_i)\Big)^2$, ~~ where $a'_i\sim \pi(\cdot|s'_i)$.
			\ENDFOR
			\STATE {\bfseries Output:}~ $\widehat{Q}_{T}(s,a)$.
		\end{algorithmic}
	\end{algorithm}

%\subsection{\redtext{Additional Results about Sample Complexity Analysis} }

%To empirically validate the sample complexity results in Proposition 4.13 with a more complex environment, we conducted experiments on Hopper-v5, a MuJoCo task with an 11-dimensional state space and a 3-dimensional action space. 

%We used a target policy trained with SAC \citep{sac} from the Stable-Baselines3 library \citep{stable-baselines3}. The trained agent is available on Hugging Face \footnote{https://huggingface.co/farama-minari/Hopper-v5-SAC-expert}. The target policy value $V^{\pi}$ was estimated using 1,000 Monte Carlo rollouts, recording the discounted cumulative rewards $(\gamma = 0.99)$ for each rollout.  The offline data was collected by the trained SAC agent with varying sample sizes. As shown in Figure \ref{fig: ope_mujoco}, we observe an exponentially decaying error, which aligns with the theoretical results outlined in Proposition 4.13. Furthermore, DQPOPE demonstrates more accurate policy value estimation across different sample sizes. This highlights the robustness of DQPOPE, achieved by averaging multiple quantiles, as opposed to directly estimating a single mean.

%\begin{figure}
%    \centering
%    \includegraphics[width=0.5\linewidth]{Full_paper/policy_value_error.png}
%    \vskip -0.35in
%    \caption{The policy value error versus different sample sizes, where the error bar is calculated by 10 repeated experiments.}
%    \label{fig: ope_mujoco}
%\end{figure}

\subsection{Additional Results about the Simulation of Toy Example}

To illustrate the advantage of DQPOPE in mean value estimation more clearly, we visualize the estimation results of 400 repeated trials for both DQPOPE and DOPE under the reward following a $t(1.5)$ distribution.  It is obvious from Figure  \ref{fig: comparison of mean} that under this heavy-tailed setting, DQPOPE outperforms DOPE in estimating the mean: all the red dots  (DQPOPE results)  lie within the interval of $[-0.2,0.2]$, while a notable proportion of green dots (DOPE results) fall outside this range, leading to incorrect estimates.

\begin{figure}[!ht]
    \centering
    \includegraphics[width=0.55\linewidth]{ 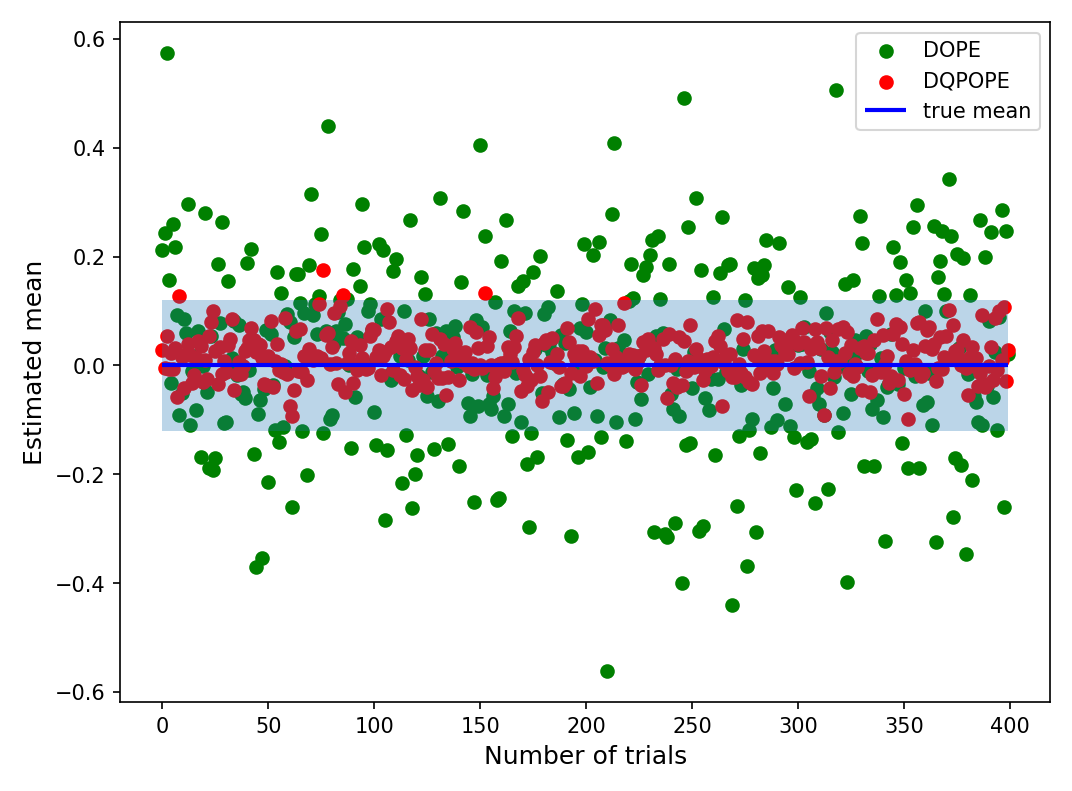}
    \vskip -0.3in
    \caption{Comparison of estimated mean between DOPE and DQPOPE under $t(1.5)$ distribution. The shaded area is plotted using 3 times the standard deviation of the DQPOPE mean estimation across 400 trials.}
    \label{fig: comparison of mean}
\end{figure}

\subsection{Additional Results on MuJoCo} 

To better evaluate the DQPOPE, we conduct experiments in MuJoCo, where the environments are significantly more complex than the simple toy example due to their high-dimensional observation spaces and continuous action spaces. We aim to empirically validate our theoretical findings in Theorem 4.12 and Proposition 4.13, demonstrating that DQPOPE achieves more robust and accurate policy value estimation by quantile averaging, rather than directly estimating a single expectation.

%\textbf{Experimental setup:} 

The experimental setup for the OPE setting is detailed as follows.

(i) For the target policy $\pi$, we use a target policy trained with TD3 \citep{TD3} in the Hopper environment, utilizing the Stable-Baselines3 implementation \citep{stable-baselines3}. The model that achieved the best test performance during $5\times10^5$ training steps is selected as the final target policy, and we save the model weights. 

(ii) The ground truth value of target policy, $V^{\pi}=261$, is estimated via 1,000 Monte Carlo rollouts by averaging the discounted cumulative rewards with $\gamma = 0.99$. 

(iii) The offline dataset $\mathcal{D}$ is collected by the trained TD3 agent with varying sample sizes and noise levels across different scenarios.

\subsubsection{Classic OPE Baselines}
 We first compare our method against several classic OPE baselines, including step-Wise Importance Sampling (WIS) \citep{precup2000eligibility}, and Doubly Robust (DR) \citep{thomas2016data}.

\begin{itemize}
    \item \textbf{WIS}:  Naive Importance Sampling (IS) re-weights the entire episode return by the full cumulative importance ratio, $\rho_{1:t'} = \prod_{t=1}^{t'} \frac{\pi(a_t|s_t)}{\pi^b(a_t|s_t)}$. This approach suffers from the curse of horizon, particularly in MuJoCo environments where the horizon length extends to 1000 steps. Consequently, we implement Step-Wise Importance Sampling (WIS). Unlike standard IS, WIS reduces variance by weighting each reward only by the cumulative importance ratio up to that specific time step, which is given as,
    $$
    V_{\text{WIS}} = \sum_{t=1}^{H} \gamma^{t-1} \frac{\rho_{1:t}}{\omega_t} r_t,
    $$
    where $H$ denotes the horizon length. $\omega_t = \sum_{i=1}^{n} \rho_{1:t}^{(i)}/n$ denotes the average cumulative important ratio at time step $t$, where $n$ is the number of trajectory. Since these methods require knowledge of the behavior policy $\pi_b$, we estimate it via maximum likelihood estimation (behavior cloning) on the dataset. Additionally, to compute log-probabilities for the deterministic target policy, we inject artificial Gaussian noise with a standard deviation of 0.01.

\item \textbf{DR}: The Doubly Robust method further combines IS with a value function estimator to further reduce variance. It utilizes the value function as a baseline to predict returns and employs importance sampling only to correct the error (advantage) between the observed rewards and the predictions. DR recursively updates the value function estimator $V_{DR}$  as follows, 
$$
V_{\text{DR}} \leftarrow \widehat{V}(s_t) + \frac{\pi(a_t|s_t)}{\pi_b(a_t|s_t)}\big( r_t + V_{\text{DR}}- \widehat{Q}(s_t,a_t)\big),
$$
where the value function estimators $\widehat{V}$ and $\widehat{Q}$ are estimated via DOPE on separate datasets.  Analogous to WIS, DR also requires the estimated behavior policy, and we adopt the same strategy as WIS.
\end{itemize}

%\subsubsection{ Comparison with classic OPE methods}

Figure \ref{fig: classic_ope_error} visualizes the comparison of policy value errors, normalized to the range $(0,1)$.  The offline dataset ($N= 25,000$) is collected using a mixture policy, with 90$\%$  of actions derived from the target TD3 policy and 10$\%$ from random sampling. Both DOPE and DQPOPE demonstrate significantly more accurate estimations of the policy value. In contrast, WIS and DR exhibit large deviations and unstable performance. These results underscore that \textbf{classic OPE methods, particularly Importance Sampling and its variants, struggle in high-dimensional, long-horizon environments like MuJoCo}, where the importance ratios or the multiplicative accumulation of importance ratios lead to vanishing estimates or explosive variance, resulting in unreliable policy value estimation.

\begin{figure}[!ht]
    \centering
    \includegraphics[width=0.55\linewidth]{ 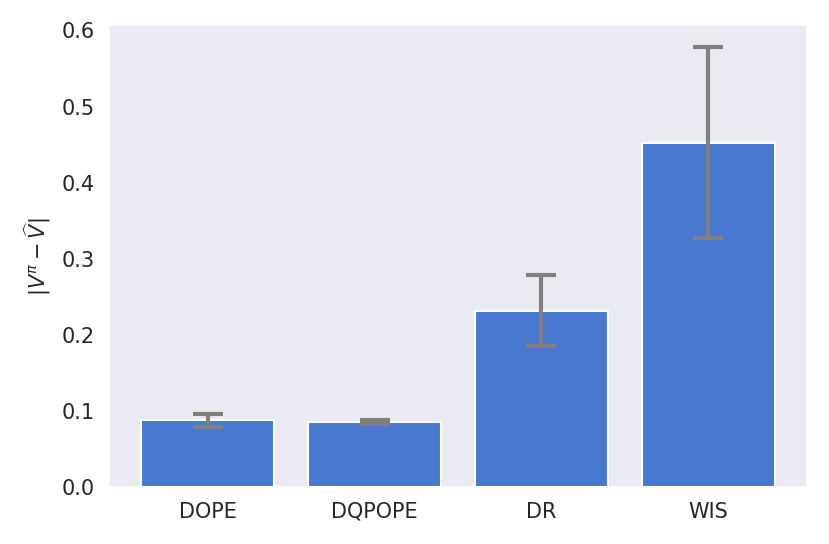}
    \vskip -0.3in
    \caption{Comparison of policy value errors between classic OPE methods and deep neural network-based methods. Each error bar represents the average of 10 repeated experiments.}
    \label{fig: classic_ope_error}
\end{figure}

\subsubsection{Deep Neural Network-based  OPE Baselines}

Given the limitations of classic baselines, we expanded our evaluation to include deep neural network-based OPE methods capable of capturing the full return distribution. These methods are better suited for high-dimensional, long-horizon environments. Specifically, we compare against a category-based method \citep{C51} and an MLE-based method \citep{mle-drl}.

\begin{itemize}
    \item  \textbf{CateOPE} (category-based OPE), adapts the C51 \citep{C51} algorithm to the off-policy evaluation setting. It discretizes the range of possible returns into a fixed set of atoms and iteratively updates a categorical distribution, allowing for explicit modeling of the full return distribution. The hyperparameters of CateOPE are summarized in Table \ref{table: cateope}, and the implementation procedure is summarized in Algorithm \ref{alg: CateOPE}.  
    
    \item \textbf{MLEOPE} (MLE-based OPE), proposed by \cite{mle-drl}, adopts a generative approach where the return distribution is modeled as a continuous diffusion process \citep{ho2020denoising}. By learning a conditional denoising network that iteratively refines Gaussian noise into valid return samples, MLEOPE implicitly models the full return distribution. The hyperparameters of MLEOPE are summarized in Table \ref{table: mleope},  and the implementation procedure is summarized in Algorithm \ref{alg: MLEOPE}.  
\end{itemize}

\begin{algorithm}[H]
    \caption{\small{Category-based OPE (CateOPE)}}
    \label{alg: CateOPE}
    \begin{algorithmic}[1]\footnotesize
        \STATE {\bfseries Initialize:} DNN class $\mathcal{F}$ (outputs probabilities), discrete support atoms $\mathbf{z} = \{z_1, \dots, z_K\}$, $\widehat{Z}_0\in\mathcal{F}$ uniformly spaced by $\Delta z = \frac{V_{\max} - V_{\min}}{K-1}$, datasets $\{\mathcal{D}_t\}_{t=1}^{T}$ and target policy $\pi$.
        \FOR{$t =1$ to $T$}
        \STATE Collect sample $(s_i, a_i, r_i, s'_i)\in \mathcal{D}_t$
        \STATE \textbf{Construct target distribution:}~~  \\Initialize $\{m_k\}_{k=1}^K\leftarrow 0$ for each $(s'_i,a'_i)$, where $a'_i \sim \pi(\cdot|s'_i)$.  Denote  $p'_j =[\widehat{Z}_{t-1}(s'_i, a'_i)]_j$ as the probability of the $j$-th atom of $\widehat{Z}_{t-1}(s'_i, a'_i)$.
            \FOR{$j = 1$ to $K$}
                \STATE Compute shifted atom location: $\hat{z}_j \leftarrow \text{clip}(r_i + \gamma z_j, V_{\min}, V_{\max})$.
                \STATE Compute relative index: $b_j \leftarrow (\hat{z}_j - V_{\min}) / \Delta z$.
                \STATE Find lower/upper neighbors: $l \leftarrow \lfloor b_j \rfloor, ~~ u \leftarrow \lceil b_j \rceil$.
                \STATE Distribute probability mass (linear interpolation):
                \STATE ~~~~$m_l \leftarrow m_l + p'_j \cdot (u - b_j)$
                \STATE ~~~~$m_u \leftarrow m_u + p'_j \cdot (b_j - l)$
            \ENDFOR
        \STATE Update (minimize cross-entropy loss): ~~
        $\widehat{Z}_{t} \leftarrow \arg \underset{Z\in \mathcal{F}}{\min} - \frac{1}{|\mathcal{D}_t|}\sum_{\mathcal{D}_t} \sum_{k=1}^{K} m_k \log  [Z(s_i, a_i)]_k$
        \ENDFOR
        \STATE {\bfseries Output:}~ $\widehat{Z}_{T}(s,a)$ ~~\color{gray}{(expected value $\mathbb{E}[\widehat{Z}_{T}(s,a)] = \mathbf{z}^\top \widehat{Z}_{T}(s,a)$)}.
    \end{algorithmic}
\end{algorithm}

\begin{algorithm}[H]
    \caption{\small{MLE-based OPE (MLEOPE)}}
    \label{alg: MLEOPE}
    \begin{algorithmic}[1]\footnotesize
        \STATE {\bfseries Initialize:} DNN class $\mathcal{F}$ (e.g., Diffusion Model), $\widehat{f}_0\in\mathcal{F}$, datasets $\{\mathcal{D}_t\}_{t=1}^{T}$ and target policy $\pi$.
        \FOR{$t =1$ to $T$}
        \STATE Collect sample $(s_i, a_i, r_i, s'_i)\in \mathcal{D}_t$
        \STATE Sample target return:~~
        $y_i = r_i + \gamma \cdot z'$, ~~ where $z' \sim \widehat{f}_{t-1}(\cdot|s'_i, a'_i)$ and $a'_i\sim \pi(\cdot|s'_i)$.
        \STATE Update (maximize log-likelihood): ~~
        $\widehat{f}_{t} \leftarrow \arg \underset{f\in \mathcal{F}}{\max}\frac{1}{|\mathcal{D}_t|}\sum_{\mathcal{D}_t} \log f(y_i | s_i, a_i)$.
        \ENDFOR
        \STATE {\bfseries Output:}~ $\widehat{f}_{T}(s,a)$ ~~\color{gray}{(expected value $\mathbb{E}_{Z \sim \widehat{f}_T(s,a)}[Z]$)}.
    \end{algorithmic}
\end{algorithm}

\textbf{ (1) Case 1: varying sample size}

We evaluate the policy value error, $|V^{\pi} - \widehat{V}|$, across four OPE algorithms under varying dataset sizes ($N \in \{10000, 25000, 50000\}$). All datasets are collected using the target TD3 policy. To ensure a fair comparison, CateOPE utilizes 64 atoms with learned probabilities, and the policy values are calculated by the mean of this categorical distribution. DQPOPE estimates the policy value by averaging 64 quantiles, while MLEOPE computes the mean of 64 samples generated via the denoising process.

The results are reported in Figure \ref{fig: mujoco_error}. Figure \ref{fig: mujoco_error} (a) visualizes the error after 500 training epochs. DQPOPE exhibits the best overall performance across all sample sizes.  When $N=10{,}000$, all distributional OPE methods attain lower error than the competing baselines, whereas for $N=50{,}000$, both CateOPE and MLEOPE remain less accurate than DQPOPE.
This performance gap can be attributed to inherent limitations. CateOPE is constrained by the requirement of a fixed, predefined support range ($V_{min}, V_{max}$), which may not perfectly align with the true return distribution. MLEOPE, by contrast, exhibits relatively high variance,  as the expected return is estimated via Monte Carlo averaging over samples generated from an iterative denoising process. This introduces sampling variability and makes the method highly sensitive to diffusion hyperparameters, such as the noise schedule and number of diffusion steps. Furthermore, MLEOPE is computationally expensive and significantly slower than other deep OPE baselines.

Figure \ref{fig: mujoco_error} (b) compares the estimated return distributions of the distributional OPE methods against the ground truth, which was generated via 1,000 rollouts of the target policy. For direct visual comparison, the distributions are rescaled to a common range while preserving the relative order of their means. Notably, DQPOPE provides the closest approximation of the ground truth distribution. These results further demonstrate the advantages of quantile process regression, which offers a flexible and straightforward framework for modeling return distributions. As discussed in Section A, this approach avoids reliance on specific parametric assumptions and mitigates the hyperparameter sensitivity that often arises in alternative distributional approaches.

\begin{figure}[!ht]
    \centering
    \includegraphics[width=0.9\linewidth]{ 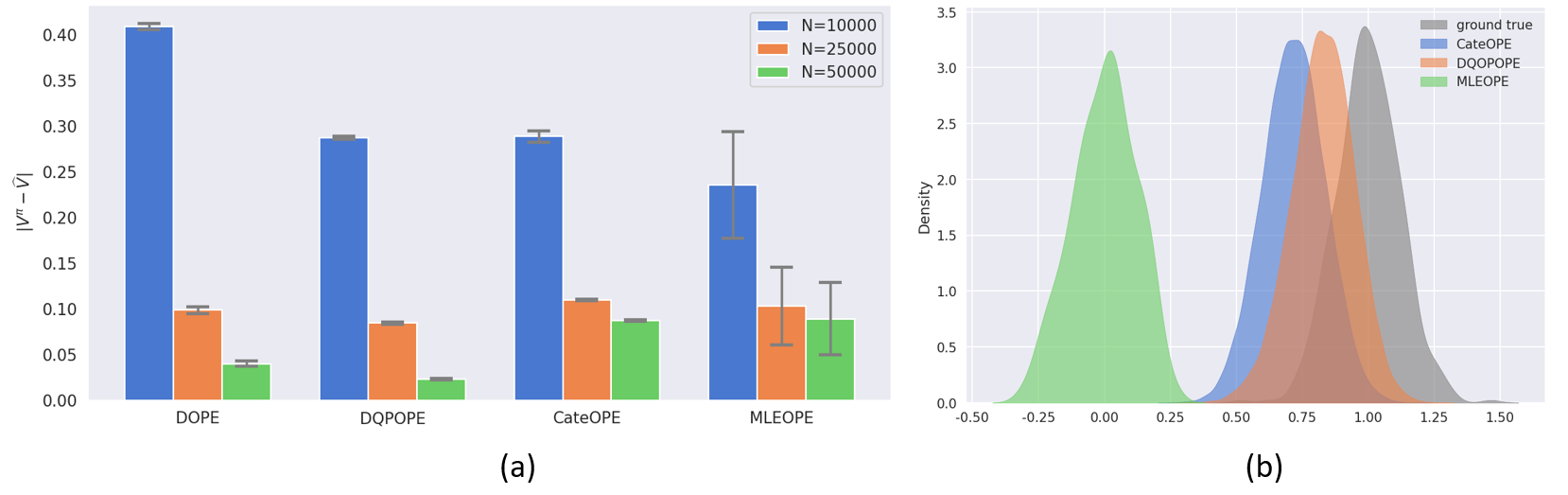}
    \vskip -0.3in
    \caption{(a) Comparison of policy value errors across four deep neural
network-based methods, where each error bar is averaged by 10 repeated experiments. (b) The estimated return distributions. }
\label{fig: mujoco_error}
\end{figure}

\textbf{(2) Case 2: injecting noise on reward} 

To assess the robustness of the proposed method, we conduct stress tests to evaluate the robustness of DQPOPE under noise disturbance. Specifically, offline datasets ($N=25,000$) are collected using a TD3 policy with Gaussian noise injected into the rewards at varying levels ($\sigma \in \{0.05, 0.1, 0.25\}$). All four OPE methods are then trained following the previously described procedure.

Figure \ref{fig: mujoco_reward_noise_1} presents boxplots of policy value errors. DQPOPE achieves the best overall accuracy and stability, with estimates that remain comparatively insensitive to reward noise. In contrast, both DOPE and other distributional methods exhibit significantly larger variance. This is particularly notable with MLEOPE. Although MLEOPE outperforms DOPE at the low noise level ($\sigma = 0.05$), it fails to scale with the noise. Its accuracy degrades, and its variance increases substantially as noise levels increase.

\begin{figure}[!ht]
    \centering
    \includegraphics[width=0.7\linewidth]{ 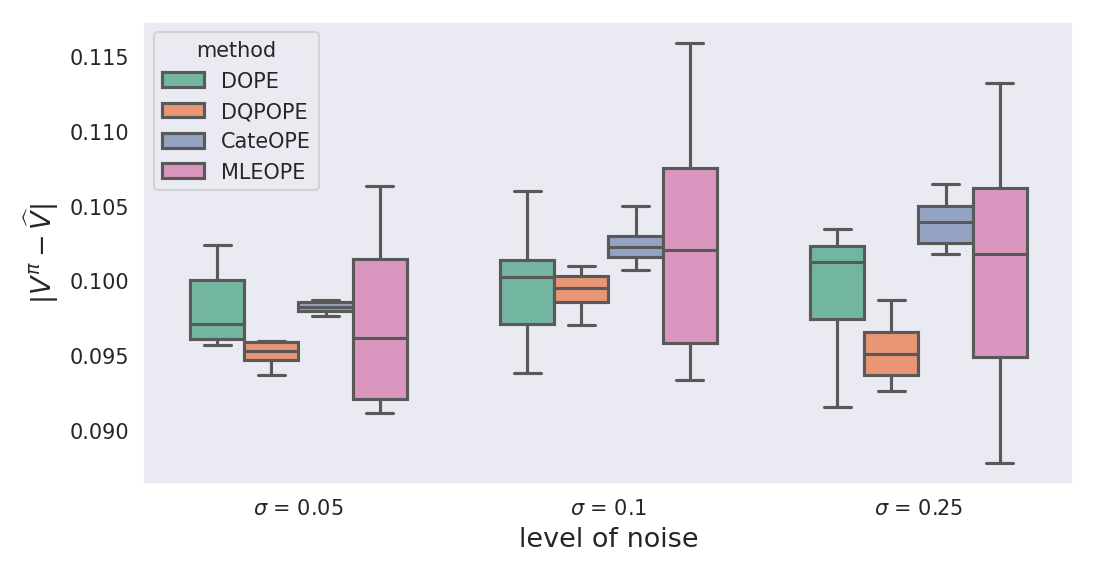}
    \vskip -0.3in
    \caption{Comparison of policy value error, where each boxplot is shown by 10 repeated experiments.}
    \label{fig: mujoco_reward_noise_1}
\end{figure}

To further investigate the influence of the number of quantile levels $K$ on policy value estimation, we evaluate DQPOPE with $K \in \{8, 32, 128\}$. The results are presented in Figure \ref{fig: mujoco_reward_noise_2}.  Notably,  utilizing a higher number of quantiles for the policy value calculation leads to a more accurate estimation.

\begin{figure}[!ht]
    \centering
    \includegraphics[width=0.7\linewidth]{ 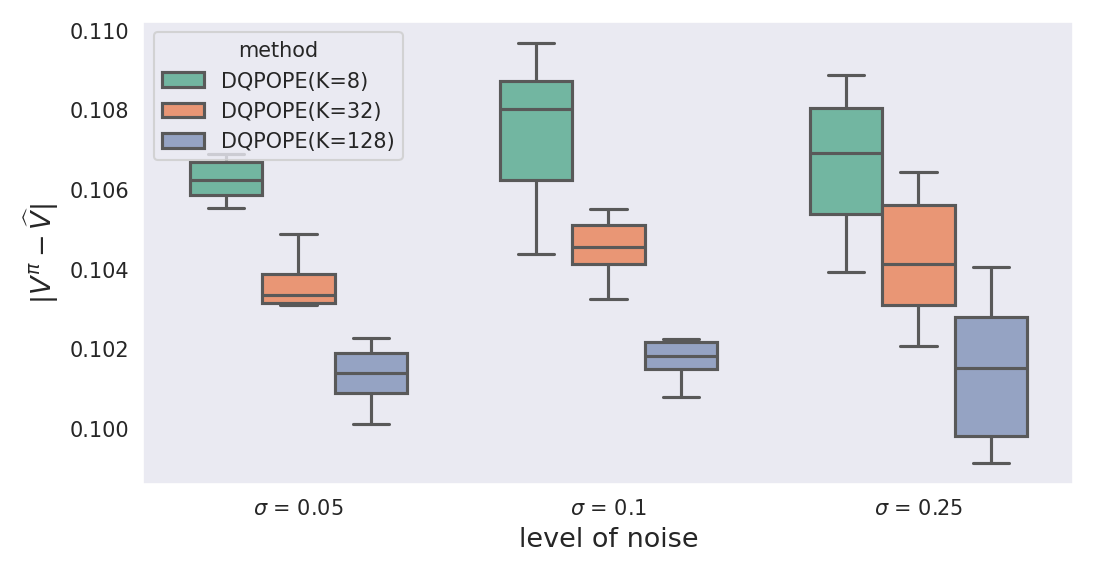}
    \vskip -0.3in
    \caption{Comparison of policy value error varying the number of quantile levels, where each boxplot is shown by 10 repeated experiments.}
    \label{fig: mujoco_reward_noise_2}
\end{figure}

\textbf{(3) Case 3: mixing random policy} 

We subsequently conduct stress tests to evaluate the robustness of DQPOPE under mixture policies.  Specifically, offline datasets ($N = 25,000$) are collected using a policy that interpolates between the trained TD3 agent and a uniform random policy, varying the mixing rate (probability of selecting a random action at each step) across $\{0.1, 0.2, 0.3\}$.

Figure \ref{fig: mujoco_mixing_policy} presents the distribution of policy value errors as boxplots. As expected, the variance of estimates for all methods increases with the mixing rate. However, DQPOPE achieves superior accuracy across all cases. This advantage is particularly pronounced at higher mixing rates ($0.2, 0.3$), where DQPOPE yields estimations that are significantly more accurate and stable than the baselines.

\begin{figure}[!ht]
    \centering
    \includegraphics[width=0.7\linewidth]{ 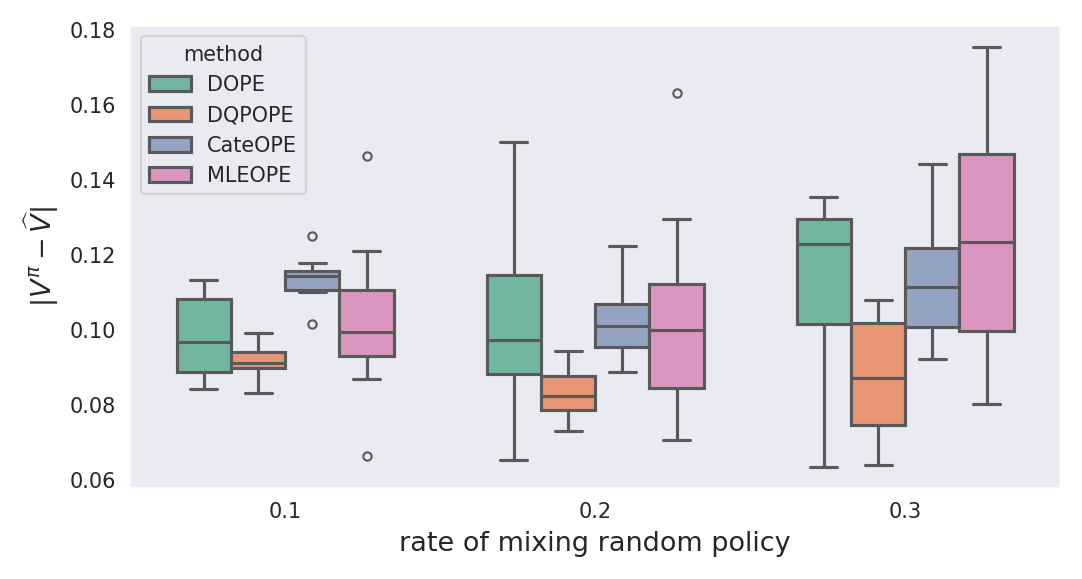}
    \vskip -0.3in
    \caption{Comparison of policy value error, where each boxplot is shown by 10 repeated experiments.}
    \label{fig: mujoco_mixing_policy}
\end{figure}

\subsubsection{Extension to Control Setting}

To demonstrate the potential benefits of Deep Quantile Process Regression (QPR) for policy optimization, we extend our evaluation to the continuous-control tasks in MuJoCo environments.

We select three widely used actor-critic algorithms as baselines: DDPG \citep{ddpg}, TD3 \citep{TD3}, and SAC \citep{sac}. For each method, we incorporate the DQPR architecture into the critic network, yielding the corresponding variants DDPG-QPR, TD3-QPR, and SAC-QPR. For a fair comparison, each variant uses 64 quantiles to estimate the expected return. 

\begin{figure}[!ht]
    \centering
    \includegraphics[width=0.9\linewidth]{ 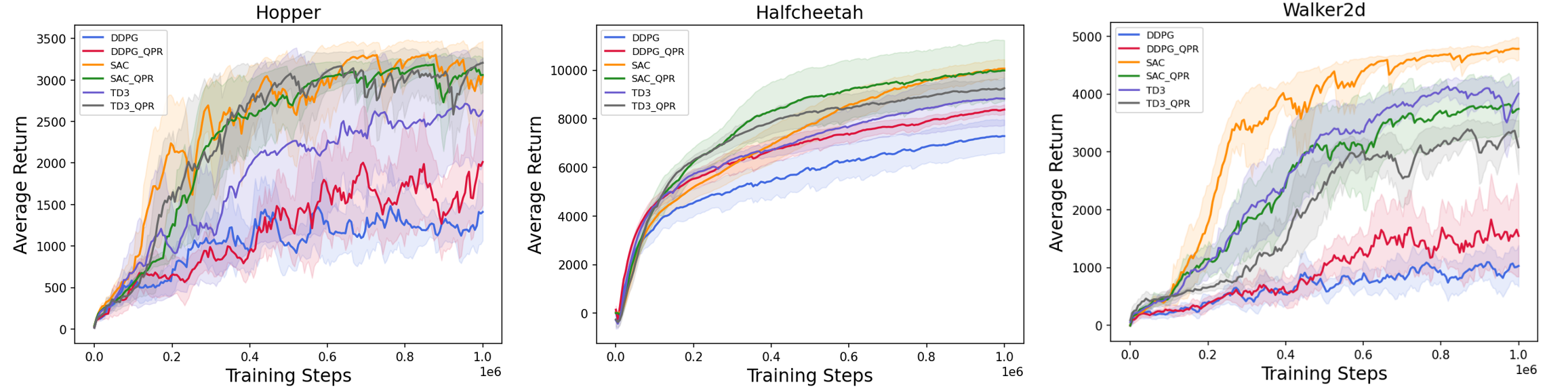}
    \vskip -0.3in
    \caption{ Curves recording the evaluation score on Hopper-v5, Halfcheetah-v5 and Walker2d-v5, and each curve is averaged over 3 seeds and shaded by their confidence intervals.}
    \label{fig: mujoco_control}
\end{figure}

\begin{figure}[!ht]
    \centering
    \includegraphics[width=0.9\linewidth]{ 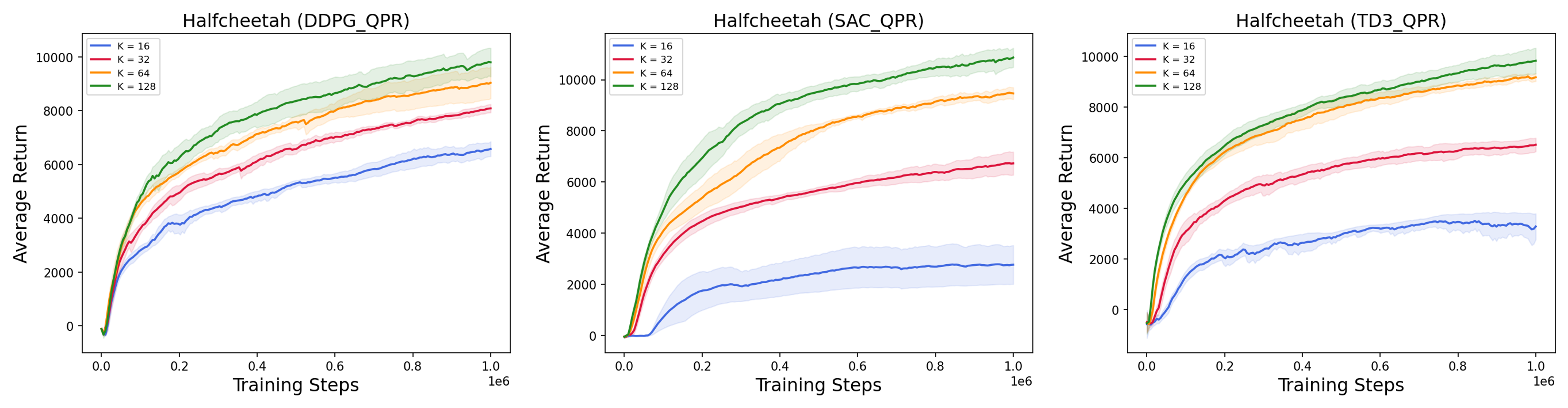}
    \vskip -0.3in
    \caption{ Curves recording the evaluation score by varying the number of quantile, where each curve is averaged over 3 seeds and shaded by their confidence intervals.}
    \label{fig: mujoco_control_quantile_number}
\end{figure}

Figure \ref{fig: mujoco_control} presents the test scores over $1 \times 10^6$ training steps across three tasks: HalfCheetah-v5, Hopper-v5, and Walker2d-v5. By incorporating Deep Quantile Process Regression, the improved algorithms consistently outperform or match the baselines. These results suggest that our technique enables more precise return distribution modeling, facilitating more accurate mean value estimation and, consequently, superior decision-making in control settings. Figure \ref{fig: mujoco_control_quantile_number} further investigates the effect of the number of quantiles, $K \in \{16, 32, 64, 128\}$, on performance. For all three methods (DDPG-QPR, SAC-QPR, and TD3-QPR), performance significantly improves as $K$ increases.

\begin{table}[!htb]
\begin{center}
\caption{The hyperparameters of DOPE.}
\vskip 0.1in
\renewcommand{\arraystretch}{0.7}
\begin{tabular}{lcccr}
\toprule
Hyperparameter & Value\\
\midrule
critic learning rate    & 1e-3\\
discount factor $\gamma$        & 0.99 \\
optimizer & Adam\\
batch size   &  256 \\
soft update (target critic) & 0.005 \\
epochs  & 500 \\

\bottomrule
\label{table: dope}
\end{tabular}
\end{center}
%\vskip -0.1in
\end{table}

\begin{table}[!htb]
\begin{center}
\caption{The hyperparameters of DQPOPE.}
\vskip 0.1in
\renewcommand{\arraystretch}{0.7}
\begin{tabular}{lcccr}
\toprule
Hyperparameter & Value\\
\midrule
critic learning rate    & 3e-3\\
discount factor $\gamma$    & 0.99 \\
optimizer & Adam\\
batch size   &  256 \\
soft update (target critic) & 0.005 \\
number of quantiles     & 64 \\
epochs  & 500 \\

\bottomrule
\label{table: dqpope}
\end{tabular}
\end{center}
%\vskip -0.1in
\end{table}

\begin{table}[!htb]
\begin{center}
\caption{The hyperparameters of CateOPE.}
\vskip 0.1in
\renewcommand{\arraystretch}{0.7}
\begin{tabular}{lcccr}
\toprule
Hyperparameter & Value\\
\midrule
critic learning rate    & 1e-4\\
discount factor $\gamma$        & 0.99 \\
optimizer & Adam\\
batch size   &  256 \\
number of atoms & 64 \\
soft update (target critic) & 0.005 \\
$V_{min}$     & 0 \\
$V_{max}$   & 300 \\
epochs  & 500 \\

\bottomrule
\label{table: cateope}
\end{tabular}
\end{center}
%\vskip -0.1in
\end{table}

\begin{table}[!htb]
\begin{center}
\caption{The hyperparameters of MLEOPE.}
\vskip 0.1in
\renewcommand{\arraystretch}{0.7}
\begin{tabular}{lcccr}
\toprule
Hyperparameter & Value\\
\midrule
critic learning rate    & 5e-4\\
discount factor $\gamma$        & 0.99 \\
optimizer & Adam\\
batch size   &  256 \\
soft update (target critic) & 0.005 \\
number of timesteps (diffusion process) & 100 \\
beta start & 1e-4\\
beta end & 0.02\\
return min (MinMax Normalization)    & 0 \\
return max (MinMax Normalization)    & 350 \\
epochs  & 500 \\

\bottomrule
\label{table: mleope}
\end{tabular}
\end{center}
%\vskip -0.1in
\end{table}

\subsection{Details about Real Data Application}\label{sec: mimic_appendix}

\subsubsection{An Improved Version for Implementing DQPOPE}

In Algorithm 1 of DQPOPE, the quantile level $\tau_i$ is sampled once for a data point $(s_i,a_i)$. However, considering the large amount of data required for training deep neural networks,  we adopt a practical modification inspired by QR-DQN \citep{QRDQN}, which estimates multiple quantiles simultaneously. This modification led to an improved version of Algorithm 1, which we implemented for the real-data experiments. 

Specifically, the key changes are as follows:
\begin{itemize}
    \item Collect sample $(s_i, a_i, r_i, s'_i)\in \mathcal{D}_t$, and sample quantile level $\{\tau^h_i\}_{h=1}^{m'}\sim\mathrm{Unif}(0,1)$ for each $(s_i,a_i)$.
    \item   Generate target sample from $\widehat{\eta}_{t-1}(s',a')$: ~~
     Sample $\{u^j_i\}_{j=1}^{m}\sim\mathrm{Unif}(0,1)$ for each $(s'_i,a'_i)$ with $a'_i\sim\pi(\cdot|s'_i)$, and plug $(s'_i,a'_i,u^j_i)$ into $\widehat{f}_{t-1}(s',a',U)$
   \item Compute target sample:  ~~
     $y^j_i \leftarrow r_i + \gamma \widehat{f}_{t-1}(s'_i,a'_i,u^j_i)$.
    \item Update: ~~
     $\widehat{f}_{t} \leftarrow \arg \underset{f \in \mathcal{F}}{\min}~
   \frac{1}{|\mathcal{D}_t|} \frac{1}{m}  \frac{1}{m'}\sum_{j=1}^m\sum_{h=1}^{m'}\sum_{(s_i,a_i,r_i,s'_i)\in\mathcal{D}_t}\rho_{\tau^h_i}\big(y^j_i-f(s_i,a_i, \tau^h_i)\big)$.
\end{itemize}
In this modification, we sample $m'$ quantile levels $\{\tau^h_i\}_{h=1}^{m'}$  for each data point $(s_i,a_i)$, and simultaneously fit the quantile at these multiple quantile levels. Additionally, the target samples $\{y_i^j\}_{j=1}^m$ generated from the target distribution $\widehat{\eta}_{t-1}(s',a')$ are used to construct a sample quantile loss, which acts as a form of "data augmentation" to facilitate stable training.  These modifications improve the training efficiency of neural networks.

\subsubsection{MIMIC-III Dataset Details}\label{sec: mimic_appendix_1}

We begin by providing an overview of the MIMIC-III dataset and outlining the data preprocessing steps. We then proceed to describe how we model the problem as a Markov Decision Process (MDP) and the design of the states, actions, and rewards. 

\paragraph{MIMIC-III dataset:} The MIMIC database includes 46,520 patients, from which we extracted a cohort of 19,611 sepsis patients. This cohort consists of 17,730 survivors and 1,881 non-survivors, with 44 observation variables and 25 treatment choices (5 discrete levels each for IV fluid and vasopressor). The mortality rate is above $9\%$ where mortality is determined by patient expiration within 48h of the final observation. Following prior work by \cite{komorowski2018artificial}, each observation variable is aggregated over each 4 hours with the mean of per-patient. All data is normalized to zero-mean with unit variance and missing values are imputed using k-nearest neighbor imputation.  Then, the normalized dataset was split into the training set, validation set, and test set with a ratio of 8:0.5:1.5. The data of 15,689 patients were used for model training, the data of 981 patients were used for choosing the best model, and the data of 2941 patients were used for policy evaluation.

\paragraph{Reward:} The reward design should reflect the change in the patient's health, assigning a positive reward for health improvements and a negative reward for deterioration.  The commonly used reward design incorporates key health indicators: the patient's SOFA score (which summarizes the extent of a patient's organ failure and thus acts as a proxy for patient health); the patient's lactate level (a measure of cell-hypoxia that is higher in septic patients because sepsis-induced low blood pressure reduces oxygen perfusion into tissue). Therefore, high SOFA scores and lactate levels should be penalized, while decreases in these two should be indicators rewarded. The reward function is designed  as follows,
\begin{align}\label{equ: mimic_reward}
r_t = C_0 \mathds{1}\left(s_{t+1}^{\mathrm{SOFA}}=s_t^{\mathrm{SOFA}} \& s_{t+1}^{\mathrm{SOFA}}>0\right)+C_1\left(s_{t+1}^{\mathrm{SOFA}}-s_t^{\mathrm{SOFA}}\right)+C_2 \tanh \left(s_{t+1}^{\mathrm{Lactate}}-s_t^{\mathrm{Lactate}}\right).
\end{align}
Following the settings by \cite{raghu2017deep}, we set $C_0=-0.025, C_1=-0.125, C_2=-2$. At terminal timesteps,  a reward of 15 is given if the patient survived, and -15 if the patient died.

\paragraph{Action:} Following  \cite{komorowski2018artificial}, the treatment decisions are determined by the total volume of intravenous fluids and the maximum dose of  vasopressors administered within 4 hours. Intravenous fluids include boluses and background infusions of crystalloids, colloids, and blood products, whereas vasopressors include norepinephrine, epinephrine, vasopressin, dopamine, and phenylephrine. We consider a $5\times 5$ action space, where each treatment dimension consists of one no-treatment category and four quartile-based bins.  At each time step, the administered dosages are mapped to the corresponding bins to define the treatment action. For example, the combination choice (0,0) corresponds to “no drug given".

\begin{table}[!ht]
\centering
\caption{ 44 physiological features used for representing state space. }
\label{table:mimic_state}
\renewcommand{\arraystretch}{0.7}
\begin{tabular}{c|c}
\toprule
\text { Age } & \text { Gender } \\
\text { Weight (kg) } & \text { Re-admission } \\
\text { Glasgow Coma Scale } & \text { HR (Heart Rate) } \\
  \text { SysBP (Systolic Blood Pressure)} & \text { Diastolic Blood Pressure }\\
  \text { Mean Blood Pressure } &  \text { Respiratory Rate } \\
  \text { Body Temp (C) } & \text { FiO2 } \\
\text { Potassium } & \text { Sodium } \\
\text { Chloride } & \text { Glucose } \\
\text { INR (International  Normalized Ratio) } & \text { Magnesium }\\
 \text { Calcium } & \text { Hemoglobin } \\
 \text { White Blood Cells } & \text { Platelets } \\
 \text { PTT (Partial Thromboplastin Time) } & \text { PT (Prothrombin Time)} \\
\text { Arterial pH } & \text { Lactate } \\
\text { PaO2 } & \text { PaCO2 } \\
\text { PaO2 } / \text { FiO2 } & \text { Bicarbonate (HCO3) } \\
\text { SpO2 } & \text { BUN (Blood Urea Nitrogen) } \\
\text { Creatinine } & \text { SGOT } \\
\text { SGPT} & \text { Total Bilirubin } \\
\text { Output (4h) } & \text { Output (total) } \\
\text { Cumulated Balance } & \text { SOFA } \\
\text { SIRS } & \text { Shock Index (= HR/SysBP)} \\
\text { Base Excess } & \text { Mechanical Ventilation } \\
\bottomrule
\end{tabular}
\end{table}

\paragraph{State:} The physiological features reflect the patient's health. We use 44 features in our experiments, summarized in Table \ref{table:mimic_state}. To better extract state information for learning the decision-making process, we train a separate state representation network for the patient’s health condition, inspired by the Approximate Information State (AIS) approach \citep{AIS2019}, which finally transforms the 44-dimensional features into a 64-dimensional embedding.

%The detailed construction of the state representation network follows the methodology proposed by \cite{killian2020empirical}, which finally transforms the 44-dimensional features into a 64-dimensional embedding.

\subsubsection{Training and Evaluation Details}\label{sec: mimic_appendix_2}

\paragraph{Target policy:}
(i) a DDQN policy trained using the DDQN procedure described below;
(ii) a random-dose policy, which selects uniformly from the 25 possible actions;
(iii) a high-dose policy, which selects uniformly from the four high-dose combinations $\{(3,3),(3,4),(4,3),(4,4)\}$; and
(iv) a low-dose policy, which selects uniformly from the four low-dose combinations $\{(0,0),(0,1),(1,0),(1,1)\}$.

\paragraph{Model architecture:} The target policy is trained by using the Dueling Double-Deep Q Network (DDQN) network architecture, which has a two-layer 128-unit fully connected network with Leaky-ReLU activation functions. The output layer adopts the dueling architecture, separating into advantage and value streams before combining them. The DQPOPE network architecture consists of three part, a state-action feature layer $\psi(s,a)$, a quantile-level embedding layer  $\phi(\tau)$ and a merge layer $\xi(\psi(s,a),\phi(\tau))$. The state-action layer $\psi(s,a)$ utilizes a two-layer 128-unit fully connected network with ReLU activation functions.  The quantile-level embedding layer $\phi(\tau)$ follows the implicit quantile representations of \citet{IQN},
$$
\phi_j(\tau):= g\left( \sum_{i=1}^{n} \cos(i\tau \pi)h_{ij} + b_j \right),
$$
where $h_{ij}, b_j$ are network parameters, and $g$ is one-layer 128-unit fully connected ReLU network. The merge layer $\xi(\psi(s,a),\phi(\tau))$ is one-layer 128-unit fully connected ReLU network with the element-wise product $\psi(s,a) \odot\phi(\tau)$ as input. The DOPE network architecture consists of a two-layer 256-unit fully connected network with ReLU activation functions as state-action feature extractor. Similarly, the DQOPE architecture also utilizes a two-layer 256-unit fully connected network with ReLU activation functions as state-action feature extractor.

\paragraph{Training procedures:} The training procedures are conducted exclusively on the training set, with the validation set used for model selection. First, we obtain the Dueling DDQN network $Q_{\text{DDQN}}$, and the target policy is given by $\pi(a|s)=\arg\max_{a}Q_{\text{DDQN}}(s,a)$, as visualized in Figure 6 of the main text. We then evaluate the target policy using both the DOPE and DQPOPE, with the target actions determined by $\arg\max_{a}Q_{\text{DDQN}}(s,a)$ during Bellman targets computation. All algorithms are trained with a learning rate of 0.0001, the discount factor $\gamma = 0.99$, a batch size of 64, and a target network update interval of 2.

\paragraph{Model evaluation:} The evaluation steps are conducted on the test set to assess the policy value of the estimated policy from the previous step. Unlike game environments where simulators are available,  real-world clinical settings rely on the test set to perform off-policy evaluation. In this study, we implemente Weighted Importance Sampling (WIS) \citep{jiang16_ope} to evaluate policies. WIS allows estimating the target policy value from patient trajectories $\{ (s^i_t,a^i_t,r^i_t,s^i_{t+1})_{t=1}^{T}\}_{i=1}^{N}$ according to how closely the behavior policy matches the target policy, where $T$ is the episode length and $N$ is the number of trajectories.

Define $\rho_t:=\frac{\pi_1(a_t| s_t)}{\pi_0(a_t|s_t)}$ as the per-step importance ratio, where $\pi_1$ represents the target policy and  $\pi_0$ represents the behaviour policy. Define $w_{t}:=\prod_{t^{\prime}=1}^t \rho_{t^{\prime}}$ as the cumulative importance ratio up to step $t$, $w^i_t := \prod_{t^{\prime}=1}^t \frac{\pi_1(a^i_t| s^i_t)}{ \pi_0(a^i_t|s^i_t)}$ as the cumulative importance ratio at horizon $t$ in $i$-th trajectories. Thus, the WIS estimator over all trajectories is given by,
\begin{align*}
V_{\text {wis}}=\frac{1}{N}\sum_{i=1}^N \sum_{t=1}^T w^i_t\gamma^{t-1} r^i_t,
\end{align*}
where the gamma $\gamma = 0.99$. The behavior policy $\pi_0$ is trained by a two-layer 128-unit fully connected network with ReLU activation functions via supervised learning on the training set, also known as behavior cloning. The target policy $\pi_1$ is derived by applying a softmax transformation to the Q-values obtained from DOPE and DQPOPE.

\end{document}